\makeatletter

\PassOptionsToPackage{svgnames}{xcolor}

\PassOptionsToPackage{noend}{algorithmic}
\PassOptionsToPackage{noend}{algpseudocode}

\AddToHook{package/lineno/before}{\RequirePackage{amsmath}}

\AtBeginDocument{
  \@ifpackageloaded{theapa}{
    \NewDocumentCommand{\citet}{o m}{
      \IfNoValueTF{#1}
        {\citeauthor{#2} (\citeyear{#2})}
        {\citeauthor{#2} (\citeyear[#1]{#2})}
    }
    \NewDocumentCommand{\citep}{o m}{
      \IfNoValueTF{#1}
        {\cite{#2}}
        {\cite[#1]{#2}}
    }
  }{
  }
}

\AddToHook{package/algorithmic/after}{
  \renewcommand{\algorithmiccomment}[1]{\hfill \# #1}
  \def\State\STATE
  \def\If\IF
  \def\Then\THEN
  \def\Elsif\ELSIF
  \def\Else\ELSE
  \def\Endif\ENDIF
  \def\For\FOR
  \def\Forall\FORALL
  \def\Do\DO
  \def\Endfor\ENDFOR
  \def\While\WHILE
  \def\Endwhile\ENDWHILE
  \def\Repeat\REPEAT
  \def\Until\UNTIL
  \def\Return\RETURN
  \def\Require\REQUIRE
  \def\Ensure\ENSURE
  \def\Comment\COMMENT
}

\AddToHook{package/algpseudocode/after}{
  \algrenewcommand\algorithmicindent{0.7em}
  \algrenewcommand{\algorithmiccomment}[1]{\hfill \# #1}
  \algblockdefx[method]{Method}{EndMethod}[2]{\textbf{method} \textsc{#1} (#2)}{\algorithmicend}
  \ifthenelse{\equal{\ALG@noend}{t}}
   {
   \algtext*{EndMethod}
   }{}
  \def\STATE\State
  \def\IF\If
  \def\THEN\Then
  \def\ELSIF\ElsIf
  \def\ELSE\Else
  \def\ENDIF\EndIf
  \def\FOR\For
  \def\FORALL\ForAll
  \def\DO\Do
  \def\ENDFOR\EndFor
  \def\WHILE\While
  \def\ENDWHILE\EndWhile
  \def\REPEAT\Repeat
  \def\UNTIL\Until
  \def\RETURN\Return
  \def\REQUIRE\Require
  \def\ENSURE\Ensure
  \def\COMMENT\Comment
}

\makeatother

\documentclass[letterpaper]{article}
\usepackage{aaai26}
\usepackage{times}
\usepackage{helvet}
\usepackage{courier}
\usepackage[hyphens]{url}
\usepackage{graphicx} 
\usepackage{caption} 
\usepackage{natbib}  
\usepackage{hyperref} 

\makeatletter
\@ifpackageloaded{theapa}{
}{
 \usepackage{natbib}
}
\makeatother

\usepackage{xparse}

\usepackage{amsmath}
\usepackage{amsthm}
\usepackage{amssymb}
\usepackage{amsfonts}

\usepackage{nicefrac}

\usepackage{tabularx}
\usepackage{booktabs}   
\usepackage{multirow}
\usepackage{subfigure}

\usepackage{manfnt}             

\usepackage{xspace}             
\usepackage{relsize}            
\usepackage{adjustbox}          
\usepackage{diagbox}            
\usepackage{pdflscape}          

\usepackage{microtype}

\usepackage{graphicx}
\usepackage{url}
\usepackage{comment}            
\usepackage{xcolor}             
\usepackage{colortbl}

\usepackage{algorithm}
\usepackage{algorithmic} 

\usepackage{xr}

\usepackage{bm}
\usepackage{alphabeta}          
\usepackage{stmaryrd}

\def\eqref#1{equation~\ref{#1}}

\def\1{\bm{1}}
\def\0{\bm{0}}

\def\rh{{\textnormal{h}}}

\def\vf{{\bm{f}}}

\DeclareMathAlphabet{\mathsfit}{\encodingdefault}{\sfdefault}{m}{sl}
\SetMathAlphabet{\mathsfit}{bold}{\encodingdefault}{\sfdefault}{bx}{n}

\def\R{{\mathbb{R}}}

\def\Z{{\mathbb{Z}}}

\DeclareMathOperator*{\argmin}{arg\,min}

\newcommand{\brackets}[1]{{\left<#1\right>}}
\newcommand{\braces}[1]{{\left\{#1\right\}}}

\newcommand{\dbrackets}[1]{{\left\llbracket#1\right\rrbracket}}

\newcommand{\then}{\therefore \qquad}

\NewDocumentCommand{\diffby}{s m O{}}{
 \IfBooleanTF{#1}
  {\frac{\partial#3}{\partial#2}}
  {\frac{d#3}{d#2}}
}

\newcommand{\satisfies}{\vDash}

\RenewDocumentCommand{\to}{o o}{
 \IfNoValueTF{#1}
  {\rightarrow}
  {\IfNoValueTF{#2}
   {\xrightarrow{#1}}
   {\xrightarrow[#2]{#1}}}
}

\NewDocumentCommand{\affect}{o o}{
 \IfNoValueTF{#1}
  {\rightsquigarrow}
  {
   \IfNoValueTF{#2}
   {\rightsquigarrow^{#1}}
   {\rightsquigarrow^{#1}_{#2}}
  }
}

\renewcommand{\then}{\Rightarrow}

\makeatletter

\usepackage{pdftexcmds}
\usepackage{etoolbox}
\newtoggle{dev}
\togglefalse{dev}

\ifcase\pdf@shellescape
  \or
  \toggletrue{dev}
  \or
\fi

\newcommand{\todo}[1]{\iftoggle{dev}{\red{\textbf{#1}}}{}}

\usepackage{stackengine}
\stackMath
\newcommand\tsup[2][2]{
 \def\useanchorwidth{T}
  \ifnum#1>1
    \stackon[-.5pt]{\tsup[\numexpr#1-1\relax]{#2}}{\scriptscriptstyle\sim}
  \else
    \stackon[.5pt]{#2}{\scriptscriptstyle\sim}
  \fi
}

\newcommand{\function}[1]{\textsc{#1}}

\newcommand{\mycolor}[2]{\textcolor{#1}{#2}}

\newcommand{\red}[1]{\mycolor{red}{#1}}
\newcommand{\green}[1]{\mycolor{Green}{#1}} 

\newcommand{\darkblue}[1]{\mycolor{RoyalBlue}{#1}}
\newcommand{\cyan}[1]{\mycolor{cyan}{#1}}
\newcommand{\magenta}[1]{\mycolor{magenta}{#1}}

\newcommand{\brown}[1]{\mycolor{Brown}{#1}}

\newcommand{\orange}[1]{\mycolor{orange}{#1}}

\def\_{\\[-0.3em]}

\newlength{\maxwidth}

\newtheorem{theo}{Theorem}

\let\@myref\ref

\renewcommand{\ref}[1]{\@myref{#1}\iftoggle{dev}{\todo{(Do not use ``ref'' directly!)}}{}}

\newcommand{\refsec}[1]{Sec.\,\@myref{#1}}
\newcommand{\refseq}[1]{Sec.\,\@myref{#1}}
\newcommand{\refig}[1]{Fig.\,\@myref{#1}}
\newcommand{\reftbl}[1]{Table \@myref{#1}}
\newcommand{\refstep}[1]{Step \@myref{#1}}
\newcommand{\refalgo}[1]{Alg.\,\@myref{#1}}
\newcommand{\refchap}[1]{Chap.\,\@myref{#1}}
\newcommand{\reflst}[1]{List \@myref{#1}}
\newcommand{\refeq}[1]{Eq.\,\@myref{#1}} 

\newcommand{\reftheo}[1]{Thm.\,\@myref{#1}}
\newcommand{\refline}[1]{line\,\@myref{#1}}
\newcommand{\refdef}[1]{Def.\, \@myref{#1}}
\newcommand{\refex}[1]{Example\,\@myref{#1}}
\newcommand{\refconv}[1]{Conv.\,\@myref{#1}}
\newcommand{\reffact}[1]{Fact.\,\@myref{#1}}
\newcommand{\reflemma}[1]{Lemma.\,\@myref{#1}}
\newcommand{\refcorol}[1]{Col.\,\@myref{#1}}
\newcommand{\refalg}[1]{Alg.\,\@myref{#1}}

\newcommand{\refsecs}[2]{Sec.\,\@myref{#1}-\@myref{#2}}
\newcommand{\refseqs}[2]{Sec.\,\@myref{#1}-\@myref{#2}}
\newcommand{\refigs}[2]{Fig.\,\@myref{#1}-\@myref{#2}}
\newcommand{\reftbls}[2]{Tables \@myref{#1}-\@myref{#2}}
\newcommand{\refsteps}[2]{Steps \@myref{#1}-\@myref{#2}}
\newcommand{\refalgos}[2]{Alg.\,\@myref{#1}-\@myref{#2}}
\newcommand{\refchaps}[2]{Chap.\,\@myref{#1}-\@myref{#2}}
\newcommand{\reflsts}[2]{Lists \@myref{#1}-\@myref{#2}}
\newcommand{\refeqs}[2]{Eq.\,\@myref{#1}-\@myref{#2}}
\newcommand{\refpages}[2]{p.\pageref{#1}-\@myref{#2}}
\newcommand{\reftheos}[2]{Thm.\,\@myref{#1}-\@myref{#2}}
\newcommand{\reflines}[2]{line\,\@myref{#1}-\@myref{#2}}
\newcommand{\refdefs}[2]{Def.\,\@myref{#1}-\@myref{#2}}
\newcommand{\refexs}[2]{Example\,\@myref{#1}-\@myref{#2}}
\newcommand{\refconvs}[2]{Conv.\,\@myref{#1}-\@myref{#2}}
\newcommand{\reffacts}[2]{Facts.\,\@myref{#1}-\@myref{#2}}
\newcommand{\reflemmas}[2]{Lemma.\,\@myref{#1}-\@myref{#2}}
\newcommand{\refcorols}[2]{Col.\,\@myref{#1}-\@myref{#2}}
\newcommand{\refalgs}[2]{Alg.\,\@myref{#1}-\@myref{#2}}

\newcounter{list}[section]

\makeatother

\makeatletter

\usepackage{textgreek}
\newcommand{\coolname}{N\textepsilon{}bula\xspace}

\newcommand{\pre}{\function{pre}}

\newcommand{\adde}{\function{add}}
\newcommand{\dele}{\function{del}}
\newcommand{\cost}{\function{cost}}

\def\hash{\text{\relsize{-1}\#}}
\newcommand{\ar}[1]{\hash{}#1}

\newcommand{\sota}{State-of-the-Art\xspace}
\newcommand{\lsota}{state-of-the-art\xspace}  

\newcommand{\astar}{\ifmmode{A^*}\else{A$^*$}\fi\xspace}
\newcommand{\gbfs}{\ifmmode{\mathrm{GBFS}}\else{GBFS}\fi\xspace}
\NewDocumentCommand{\uct}{s}{\ifmmode{\mathrm{UCT}{\IfBooleanT{#1}{^*}}}\else{UCT{\IfBooleanT{#1}{*}}}\fi\xspace}
\NewDocumentCommand{\guct}{s}{\ifmmode{\mathrm{GUCT}{\IfBooleanT{#1}{^*}}}\else{GUCT{\IfBooleanT{#1}{*}}}\fi\xspace}

\newcommand{\topen}{tree-based open list\xspace}

\newcommand{\newheuristic}[2]{
 \def#1{
  \relax\ifmmode
  h^\mathrm{#2}\xspace
  \else
  \text{#2}\xspace
  \fi
 }
}

\newheuristic{\lmcut}{LMcut}
\newheuristic{\mands}{M\&S}
\newheuristic{\pdb}{PDB}
\newheuristic{\ff}{FF}
\newheuristic{\ce}{CEA}
\newheuristic{\cg}{CG}
\newheuristic{\gc}{GC}
\newheuristic{\ad}{add}
\newheuristic{\hmax}{max}
\newheuristic{\lc}{LC}
\newheuristic{\blind}{blind}

\newcommand{\newlearnedheuristic}[2]{
 \def#1{
  \relax\ifmmode
  H^\mathrm{#2}\xspace
  \else
  \text{#2}\xspace
  \fi
 }
}

\newlearnedheuristic{\Hlmcut}{LMcut}
\newlearnedheuristic{\Hmands}{M\&S}
\newlearnedheuristic{\Hpdb}{PDB}
\newlearnedheuristic{\Hff}{FF}
\newlearnedheuristic{\Hce}{CEA}
\newlearnedheuristic{\Hcg}{CG}
\newlearnedheuristic{\Had}{add}
\newlearnedheuristic{\Hmax}{max}
\newlearnedheuristic{\Hlc}{LC}
\newlearnedheuristic{\Hblind}{blind}

\newcommand{\newUnitCostHeuristic}[2]{
 \def#1{
  \relax\ifmmode
  \hat{h}^\mathrm{#2}\xspace
  \else
  \text{#2}\xspace
  \fi
 }
}

\newUnitCostHeuristic{\lmcuto}{LMcut}
\newUnitCostHeuristic{\mandso}{M\&S}
\newUnitCostHeuristic{\ffo}{FF}
\newUnitCostHeuristic{\ceo}{CEA}
\newUnitCostHeuristic{\cgo}{CG}
\newUnitCostHeuristic{\ado}{add}
\newUnitCostHeuristic{\gco}{GoalCount}
\newUnitCostHeuristic{\lco}{LC}

\newcommand{\newrandomheuristic}[2]{
 \def#1{
  \ifmmode
  \rh^\mathrm{#2}\xspace
  \else
  \text{#2}\xspace
  \fi
 }
}

\newrandomheuristic{\rlmcut}{LMcut}
\newrandomheuristic{\rmands}{M\&S}
\newrandomheuristic{\rpdb}{PDB}
\newrandomheuristic{\rff}{FF}
\newrandomheuristic{\rce}{CEA}
\newrandomheuristic{\rcg}{CG}
\newrandomheuristic{\rad}{add}
\newrandomheuristic{\rhmax}{max}
\newrandomheuristic{\rlc}{LC}

\makeatother

\usepackage{etoolbox}
\usepackage{xkeyval}
\makeatletter

\def\strips@initialize{
\def\@transitiononly{0}
\def\@conditiontype{0}
\def\@usecondeffect{0}
\def\@cost{0}
\def\@useaxiom{0}
\def\@lifted{0}
\def\@track{0}
}

\define@key{strips}{transition-only}[]{\def\@transitiononly{1}}
\define@key{strips}{negative-preconditions}[]{\def\@conditiontype{1}}
\define@key{strips}{disjunctive-preconditions}[]{\def\@conditiontype{2}}
\define@key{strips}{conditional-effects}[]{\def\@usecondeffect{1}}
\define@key{strips}{action-costs}[]{\ifnumcomp{\@cost}{<}{1}{\def\@cost{1}}{}}
\define@key{strips}{unitcost}[]{\ifnumcomp{\@cost}{<}{2}{\def\@cost{2}}{}}
\define@key{strips}{derived-predicates}[]{\def\@useaxiom{1}}
\define@key{strips}{lifted}[]{\def\@lifted{1}}
\define@key{strips}{optimal}[]{\ifnumcomp{\@track}{<}{1}{\def\@track{1}}{}}
\define@key{strips}{satisficing}[]{\ifnumcomp{\@track}{<}{2}{\def\@track{2}}{}}
\define@key{strips}{agile}[]{\ifnumcomp{\@track}{<}{3}{\def\@track{3}}{}}

\def\conditionset{
\if\@useaxiom0
P
\else
P\cup P_X
\fi
}

\let\satisfies@orig\satisfies

\def\satisfies{
\if\@conditiontype0
\supseteq
\else
\satisfies@orig
\fi
}

\def\condition{
\if\@conditiontype0
\conditionset
\else
\mathcal{F}(\conditionset)
\fi
}

\def\ga{
\if\@lifted0
a
\else
a^{\dagger}
\fi
}

\def\applyformula{
\if\@usecondeffect0
(s \setminus \dele(a)) \cup \adde(a)
\else
(s
 \setminus \braces{e \mid (c \triangleright e) \in \dele(\ga), c\satisfies s})
 \cup      \braces{e \mid (c \triangleright e) \in \adde(\ga), c\satisfies s}
\fi
}

\NewDocumentCommand{\strips}{O{}}{
\strips@initialize
\setkeys{strips}{#1}
\if\@lifted1
  \strips@propositional\par
  \strips@lifted
\else
  \strips@propositional
\fi
}

\newcommand{\strips@propositional}{
\if\@conditiontype1
Given a set of propositional variables $V$,
let $\mathcal{F}(V)$ be a propositional formula consisting of $V$ and
logical operations $\braces{\land,\lnot}$.
\fi
\if\@conditiontype2
Given a set of propositional variables $V$,
let $\mathcal{F}(V)$ be a propositional formula consisting of $V$ and
logical operations $\braces{\land,\lor,\lnot}$.
\fi
\if\@useaxiom0
We define a propositional STRIPS Planning problem
as a 4-tuple $\brackets{P,A,I,G}$
where
 $P$ is a set of propositional variables,
 $A$ is a set of actions,
 $I\subseteq P$ is the initial state, and
 $G\subseteq \conditionset$ is a goal condition.
\else
We define a propositional STRIPS Planning problem
as a 6-tuple $\brackets{P,A,X,P_X,I,G}$
where
 $P$ is a set of propositions,
 $A$ is a set of actions,
 $X$ is a set of axioms,
 $P_X$ is a set of derived propositions ($P\cap P_X=\emptyset$),
 $I\subseteq P$ is the initial state, and
 $G\subseteq \conditionset$ is a goal condition.
\fi
\ifnumcomp{\@transitiononly}{>}{0}{
We omit the details of action applications as they are not important in this paper.
It suffices to say an action $a\in A$ transitions from a state $s\subseteq P$ to a successor $s'=a(s)\subseteq P$.
}{
\ifnumcomp{\@cost}{<}{1}{
Each action $a\in A$ is a 3-tuple $\brackets{\pre(a),\adde(a),\dele(a)}$ where
}{
Each action $a\in A$ is a 4-tuple $\brackets{\pre(a),\adde(a),\dele(a),\cost(a)}$ where
$\cost(a) \in \Z^{0+}$ is a cost\ifnumcomp{\@cost}{=}{2}{ (We assume unit-cost: $\forall a\in A; \cost(a)=1$)}{},
}
$\pre(a) \subseteq \condition$ is a precondition and
\if\@usecondeffect0
$\adde(a), \dele(a)\subseteq P$ are the add-effects and delete-effects.
\else
$\adde(a), \dele(a)$ are the add-effects and delete-effects.
Each effect is denoted as $c \triangleright e$ where
$c \in \condition$ is an \emph{effect condition} and
$e \in P$.
\fi
\if\@useaxiom1
The set of axioms $X$ consists of clauses $f \Rightarrow p$ where
$f \in \condition$ is a body and $p \in P_X$ is a head.
\fi
A state $s\subseteq \conditionset$ is a set of true propositions
(all of $P\setminus s$ is false),
an action $a$ is \emph{applicable} when $s \satisfies \pre(a)$ (read: $s$ \emph{satisfies} $\pre(a)$),
and applying action $a$ to $s$ yields a new successor state
\if\@useaxiom0
$a(s) = \applyformula$.
\else
$a(s)$.
To compute $a(s)$, we first obtain a non-derived state
$s' \gets \applyformula \setminus P_X $.
Then we perform a fix-point calculation
$s' \gets s' \cup \braces{p \in P_X \mid (f\Rightarrow p)\in X \land s \satisfies f}$.
\fi
\par
}                               
The task of classical planning is to find a sequence of actions called a \emph{plan} $(\ga_1,\cdots,\ga_n)$
where, for $1\leq t\leq n$,
 $s_0=I$,
 \ifnumcomp{\@transitiononly}{>}{0}{}{$s_t\satisfies \pre(a_{t+1})$,}
 $s_{t+1}=a_{t+1}(s_t)$,
 and $s_n\satisfies G$.
\ifnumcomp{\@track}{>}{0}{
 A plan is \emph{optimal} if
 \ifnumcomp{\@cost}{<}{1}{
   there is no shorter plan.
 }{
   there is no plan with a lower cost $\sum_t \cost(a_t)$.
 }
 \ifnumcomp{\@track}{>}{1}{
   A plan is otherwise called \emph{satisficing}.
   \ifnumcomp{\@track}{>}{2}{
     A problem setting that completely ignores the solution quality is called \emph{agile} setting,
     while \emph{satisficing} setting implies that the solver still attempts to find a
     \ifnumcomp{\@cost}{<}{1}{shorter}{cheaper}
     plan.
     This paper focuses on the \emph{agile} setting.
   }{
     This paper focuses on the \emph{satisficing} setting.
   }
 }{}
}{}
}

\newcommand{\strips@lifted}{
In \emph{Lifted STRIPS}, each propositional variable is an \emph{instantiation}/\emph{grounding} of
a first-order logic predicate.
Each predicate $p(x_1,\ldots,x_{\ar{p}})$ is parameterized by a list of parameters/variables/arguments $X=(x_1,\ldots,x_{\ar{p}})$,
where $\ar{p}$ is an \emph{arity} of $p$.
A proposition is obtained by substituting each $x_i$ with an \emph{object} in a set $O$.
Each $p$ therefore has $O^{\ar{p}}$ instantiations.
Similarly, each action $a\in A$ is now called a \emph{ground action},
which is an instantiation of a \emph{lifted action} $a(x_1,\ldots,x_{\ar{p}})$ parameterized by $\ar{a}$ parameters.
A ground action is obtained by substituting the arguments as well as
the parameters used in the preconditions and the effects.
}

\makeatother

\usepackage{etoolbox}
\usepackage{xkeyval}
\makeatletter

\long\def\addto#1#2#3{
  \ifinlist{#3}{#1}{
  }{
    \listadd{#1}{#3}
    \ifdefempty#2{
     \expandafter\def\expandafter#2\expandafter{#2#3}
    }{
     \expandafter\def\expandafter#2\expandafter{#2,#3}
    }
  }
}

\define@key{heuristics}{ff}[1]{
 \def\heuristics@ff{#1}
 \ifnumcomp{\heuristics@ff}{>}{0}{
  \addto{\heuristiclist}{\heuristicstr}{\ff}
  \addto{\heuristiccitelist}{\heuristiccitestr}{hoffmann01}
 }{}
}
\define@key{heuristics}{ad}[1]{
 \def\heuristics@ad{#1}
 \ifnumcomp{\heuristics@ad}{>}{0}{
  \addto{\heuristiclist}{\heuristicstr}{\ad}
  \addto{\heuristiccitelist}{\heuristiccitestr}{bonet2001planning}
 }{}
}
\define@key{heuristics}{hmax}[1]{
 \def\heuristics@hmax{#1}
 \ifnumcomp{\heuristics@hmax}{>}{0}{
  \addto{\heuristiclist}{\heuristicstr}{\hmax}
  \addto{\heuristiccitelist}{\heuristiccitestr}{bonet2001planning}
 }{}
}
\define@key{heuristics}{gc}[1]{
 \def\heuristics@gc{#1}
 \ifnumcomp{\heuristics@gc}{>}{0}{
  \addto{\heuristiclist}{\heuristicstr}{\gc}
  \addto{\heuristiccitelist}{\heuristiccitestr}{FikesHN72}
 }{}
}
\define@key{heuristics}{cea}[1]{
 \def\heuristics@cea{#1}
 \ifnumcomp{\heuristics@cea}{>}{0}{
  \addto{\heuristiclist}{\heuristicstr}{\ce}
  \addto{\heuristiccitelist}{\heuristiccitestr}{helmert2008unifying}
 }{}
}
\define@key{heuristics}{cg}[1]{
 \def\heuristics@cg{#1}
 \ifnumcomp{\heuristics@cg}{>}{0}{
  \addto{\heuristiclist}{\heuristicstr}{\cg}
  \addto{\heuristiccitelist}{\heuristiccitestr}{Helmert04}
 }{}
}

\define@key{heuristics}{simplified}[1]{\def\heuristics@simplified{#1}}

\define@key{heuristics}{relaxation}[1]{\def\heuristics@relaxation{#1}}

\define@key{heuristics}{perfect}[1]{\ifnumcomp{\heuristics@properties}{<}{#1}{\def\heuristics@properties{#1}}{}}
\define@key{heuristics}{admissible}[2]{\ifnumcomp{\heuristics@properties}{<}{#1}{\def\heuristics@properties{#1}}{}}
\define@key{heuristics}{inadmissible}[3]{\ifnumcomp{\heuristics@properties}{<}{#1}{\def\heuristics@properties{#1}}{}}
\define@key{heuristics}{perfectsat}[4]{\ifnumcomp{\heuristics@properties}{<}{#1}{\def\heuristics@properties{#1}}{}}
\define@key{heuristics}{tdom}[5]{\ifnumcomp{\heuristics@properties}{<}{#1}{\def\heuristics@properties{#1}}{}}

\define@key{heuristics}{expansion}[1]{\def\heuristics@expansion{#1}}

\NewDocumentCommand{\heuristics}{O{}}{
\def\heuristics@ff{0}
\def\heuristics@ad{0}
\def\heuristics@hmax{0}
\def\heuristics@gc{0}
\def\heuristics@cea{0}
\def\heuristics@cg{0}
\def\heuristics@relaxation{0}
\def\heuristics@simplified{0}
\def\heuristics@properties{0}
\def\heuristics@expansion{0}
\def\heuristiclist{}
\def\heuristiccitelist{}
\def\heuristicstr{}
\def\heuristiccitestr{}
\setkeys{heuristics}{#1}
\ifnumcomp{\heuristics@simplified}{>}{0}{
A domain-independent heuristic function $h(s)$
returns an estimate of the cumulative cost from a state $s$ to one of the goal states (states that satisfy $G$).
}{
Given a problem $\brackets{P,A,I,G}$ and a state $s$,
a domain-independent heuristic function $h(s, \brackets{P,A,I,G})$
returns an estimate of the cumulative cost from $s$ to one of the goal states (states that satisfy $G$),
typically through a symbolic, non-statistical means including problem relaxation and abstraction.
It is often abbreviated as $h(s)$ or $h(s,G)$.
}
\ifdefempty\heuristiclist{}{
Notable \lsota functions that appear in this paper includes
$\heuristicstr$ \citep{\heuristiccitestr}.
}
\ifnumcomp{\heuristics@properties}{<}{1}{}{
  Often, the true optimal cost from a state $s$ is called
  the \emph{perfect heuristics} $h^*(s)$ \citep{helmert2008good}.
  \ifnumcomp{\heuristics@properties}{<}{2}{}{
    \emph{Admissible} heuristics are those which never overestimate $h^*$,
    i.e., $\forall s; h(s)\leq h^*(s)$.
    Optimizing algorithms like \astar \citep{hart1968formal} are
    guaranteed to find the optimal solutions with such heuristics.
    \ifnumcomp{\heuristics@expansion}{<}{1}{}{
      Moreover, \astar is the optimal expansion algorithm, i,e.,
      expands the fewest nodes among all algorithms under the same admissible $h$.
    }
    \ifnumcomp{\heuristics@properties}{<}{3}{}{
      Otherwise they are called \emph{inadmissible} heuristics,
      and are typically combined with satisficing/agile algorithms like GBFS \citep{doran1966experiments,bonet2001planning}.
      \ifnumcomp{\heuristics@properties}{<}{4}{}{
        Furthermore, heuristics that preserve the same ordering as $h^*$ are called
        \emph{perfect satisficing heuristics} $h^\leq$ \citep{kuroiwa2022biased},
        i.e., $\forall s,t; h^\leq(s)\leq h^\leq(t)\then h^*(s)\leq h^*(t)$.
        \ifnumcomp{\heuristics@expansion}{<}{1}{}{
          GBFS is the optimal expansion algorithm under $h^\leq$.
        }
        \ifnumcomp{\heuristics@properties}{<}{5}{}{
          Given a monotonic \emph{inflation} function $t:\R^{0+}\to\R^{0+}$
          s.t. $\forall x; t(x)\geq x$ and $\forall x,y; x\geq y \then t(x)\geq t(y)$,
          heuristics that preserve the same ordering as $h^*$ when inflated are called
          \emph{$t$-dominating heuristics},
          i.e., $\forall s,t; t(h(s))\leq h(t)\then h^*(s)\leq h^*(t)$.
        }
      }
    }
  }
}

\if\heuristics@relaxation1
A significant class of heuristics is called delete relaxation heuristics,
which solve a relaxed problem which does not contain delete effects,
and then returns the cost of the solution of the relaxed problem as an output.
The cost of the optimal solution of a delete relaxed planning problem from a state $s$ is
denoted by $h^+(s)$, but this is too expensive to compute in practice (NP-complete) \citep{bylander1996}.
Therefore, practical heuristics typically try to obtain its further relaxations
that can be computed in polynomial time.
\fi

\ifnumcomp{\heuristics@hmax}{>}{1}{
Max heuristics $\hmax$ \citep{bonet2001planning} is recursively defined as follows:
\begin{align}
 \hmax(s,G) = \max_{p\in G}
 \left\{
  \begin{array}{l}
   0\ \text{if}\ p\in s.\ \text{Otherwise,}\\
   \min_{\braces{a\in A\mid p\in\adde(a)}} \\
    \quad \left[\cost(a)+\ad(s, \pre(a))\right].
  \end{array}
 \right.
\end{align}
}{}

\ifnumcomp{\heuristics@ad}{>}{1}{
Additive heuristics $\ad$ \citep{bonet2001planning} is recursively defined as follows:
\begin{align}
 \ad(s,G) = \sum_{p\in G}
 \left\{
  \begin{array}{l}
   0\ \text{if}\ p\in s.\ \text{Otherwise,}\\
   \min_{\braces{a\in A\mid p\in\adde(a)}} \\
    \quad \left[\cost(a)+\ad(s, \pre(a))\right].
  \end{array}
 \right.
\end{align}
}{}

\ifnumcomp{\heuristics@ff}{>}{1}{
FF heuristics $\ff$ \citep{hoffmann01} is defined based on another heuristics $h$, such as $h=\ad$, as a subprocedure.
For each proposition $p$,
the action $a$ that adds $p$ with the minimal $\cost(a)+h(s, \pre(a))$
is conceptually ``the cheapest action that achieves a subgoal $p$'',
called the \emph{cheapest achiever} / \emph{best supporter} $\text{bs}(p,s,h)$ of $p$.
Using this, $\ff$ is defined as the sum of actions in a relaxed plan $\Pi^+$ constructed as follows:
\begin{align}
 \ff(s,G,h) &= \sum_{a\in \Pi^+(s,G,h)} \cost(a)\\
 \Pi^+(s,G,h) &= \bigcup_{p\in G}
 \left\{
  \begin{array}{l}
   \emptyset\ \text{if}\ p\in s.\ \text{Otherwise,}\\
   \braces{a} \cup \Pi^+(s,\pre(a)) \\
   \qquad \text{where}\ a=\text{bs}(p,s,h).
  \end{array}
 \right.\\
 \text{bs}(p,s,h)&=\argmin_{\braces{a\in A\mid p\in \adde(a)}} \left[\cost(a)+h(s, \pre(a))\right].
\end{align}
\ifnumcomp{\heuristics@ff}{>}{2}{
  In practice, $\ff$ can be implemented in several ways, each producing different values
  due to the tie-breaking difference in the $\argmin$ in $\text{bs}(p,s,h)$.
  \citet{kuroiwa2019case} showed that Graphplan-based implementation yields the best planner performance
  due to the combination of low-level efficiency and heuristic accuracy.
}{}
}{}

\ifnumcomp{\heuristics@gc}{>}{1}{
Goal Count heuristics $\gc$ is a simple heuristic proposed in \citep{FikesHN72}
that counts the number of propositions that are not satisfied yet.
$\brackets{\text{condition}}$ is a cronecker's delta / indicator function that returns 1 when the condition is satisfied.
\begin{align}
 \gc(s,G) &= \sum_{p\in G} \dbrackets{p\not \in s}.
\end{align}
}{}
}

\makeatother

\iftoggle{dev}{
  \IfFileExists{\jobname.needpyg}{
    \usepackage[finalizecache,cachedir=.]{minted}
  }{
    \usepackage[cachedir=.]{minted}
  }
}{
  \usepackage[frozencache,cachedir=.]{minted}
}

\allowdisplaybreaks

\author{Masataro Asai}
\title{Bilevel MCTS for Amortized O(1) Node Selection in Classical Planning}

\affiliations{
MIT-IBM Watson AI Lab \\
IBM Research Cambridge \\
masataro.asai@ibm.com
}

\begin{document}

\maketitle

\begin{abstract}
We study an efficient implementation of
Multi-Armed Bandit (MAB)-based Monte-Carlo Tree Search (MCTS) for classical planning.
One weakness of MCTS is that
it spends a significant time deciding which node to expand next.
While selecting a node from an OPEN list with $N$ nodes has
$O(1)$ runtime complexity with traditional array-based priority-queues for dense integer keys,
the tree-based OPEN list used by MCTS requires $O(\log N)$,
which roughly corresponds to the search depth $d$.
In classical planning, $d$ is arbitrarily large (e.g., $2^k-1$ in $k$-disk Tower-of-Hanoi)
and the runtime for node selection is significant,
unlike in game tree search, where the cost is negligible compared to the node evaluation (rollouts)
because $d$ is inherently limited by the game (e.g., $d\leq 361$ in Go).
To improve this bottleneck,
we propose a bilevel modification to MCTS that
runs a best-first search from each selected leaf node with an expansion budget proportional to $d$,
which achieves amortized $O(1)$ runtime for node selection, equivalent to the traditional queue-based OPEN list.
In addition, we introduce Tree Collapsing, an enhancement that reduces action selection steps
and further improves the performance.
\end{abstract}

\section{Introduction}

The exploration-exploitation trade-off has been a common research topic across AI literature,
including Reinforcement Learning, Game Tree Search, and Classical Planning.
Bandit-based Monte-Carlo Tree Search (MCTS) has been widely successful in addressing the trade-off in the first two areas,
while it is not popular in the recent Classical Planning literature.
As a result, successful methods (e.g., \citep{nakhost2009monte, xie14type, kuroiwa2022biased})
that try to address the trade-off in classical planning have commonly been
approaches that are not backed by or benefiting from the statistical tools developed in the bandit community,
even if they have their theoretical justifications.

Inspired by the initial work on applying MCTS to classical planning \citep{schulte2014balancing}
which suggested that MCTS can implement traditional best-first search by
viewing the tree as an OPEN list and modifying the backpropagation and selection rules,
recent work \citep{wissow2023scale}
empirically showed that it can outperform queue-based diversified search
as long as it uses the correct multi-armed bandit (MAB) algorithm that suits the classical planning context
(UCB1-Normal2).
However,
the paper spends much of its space on the theoretical analysis
and comparisons against simple baselines (e.g., GBFS) and other MCTS,
making it unclear whether the approach can empirically compete against \lsota planners
such as LAMA \citep{richter2011lama}.

One weakness of MCTS is that
it spends a significant amount of effort on node selection.
While the traditional priority-queue-based OPEN list has $O(1)$ runtime for \function{popmin} operation,
the tree-based OPEN list used by MCTS requires $O(D)$ runtime for recursively descending the tree,
where $D$ is the depth of the path from the root to the selected leaf.
In a perfectly balanced tree with a constant branching factor,
$D$ grows at an $O(\log N)$ rate for $N$ leaf nodes and causes a significant bottleneck.
This cost is significant in single-agent planning,
unlike in game tree search, where the depth tends to be inherently limited by the game.
For example, the maximum search depth of the Game of Go is $19\times 19=361$,
while the solution length of Tower of Hanoi with $k$ disks is exponential ($2^k-1$).

\begin{figure}[tb]
 \includegraphics[width=\linewidth]{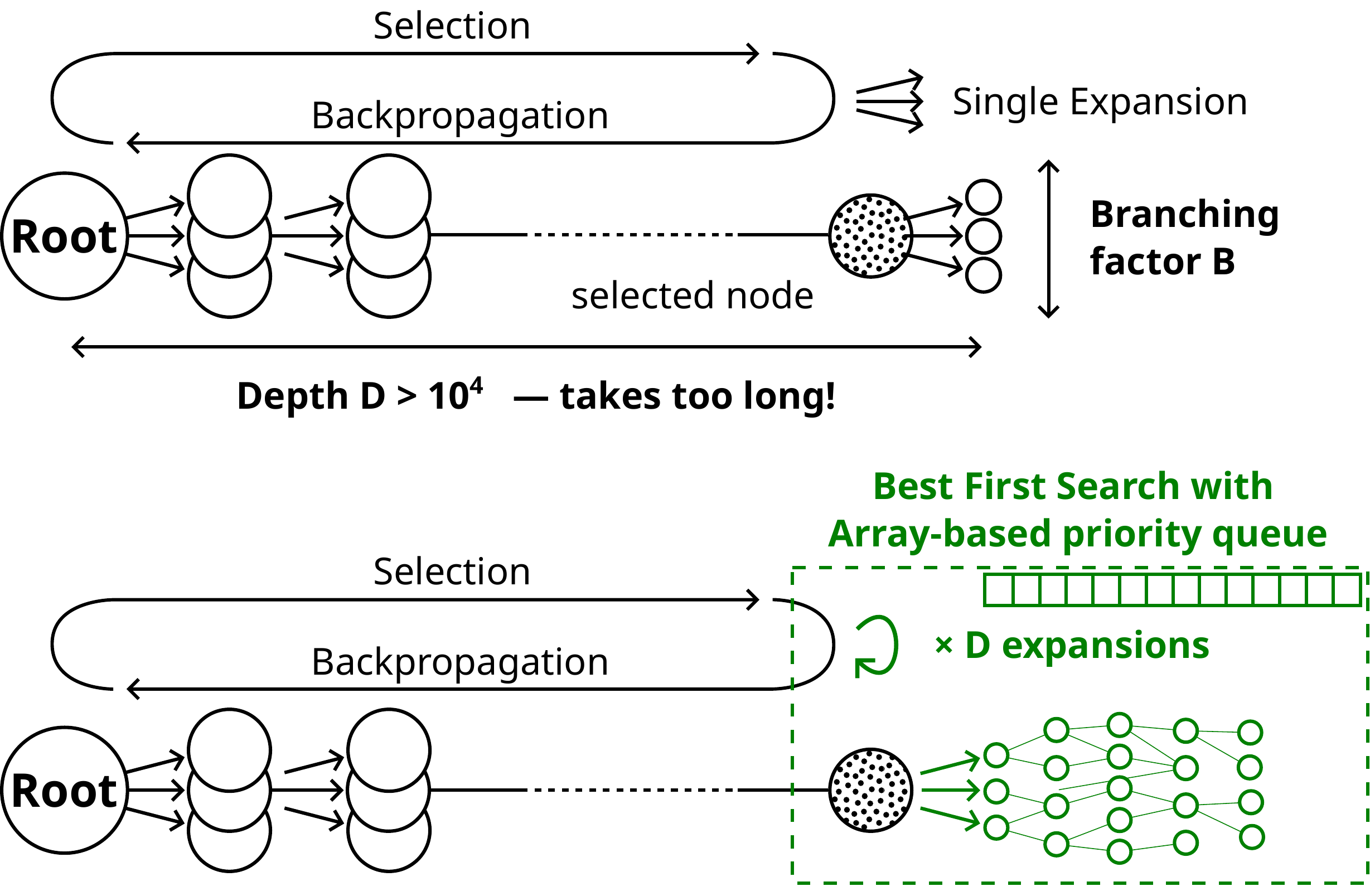}
 \caption{MCTS vs. \green{Bilevel MCTS}.}
 \label{fig:bilevel}
\end{figure}

To improve this behavior,
we propose \green{Bilevel MCTS} (\refig{fig:bilevel}) that
runs a best-first search from each selected leaf with an expansion budget
proportional to the depth of the node,
achieving an amortized $O(1)$ runtime for selection per node.
To further improve the performance,
we propose an enhancement to MCTS called
\emph{\cyan{Tree Collapsing}},
which reduces the depth of the tree by letting the grandparents
adopt their grandchildren as long as the number of children is below a threshold $\theta$,
thereby bypassing action selection steps.

Finally, to match the state of the art performance,
we combine UCB1-Normal2 MAB,
\emph{novelty metric} \citep{lipovetzky2017bfws},
\emph{alternation queue} \citep{RogerH10},
and \emph{boosted preferred operators} \citep{RichterH2009}
based on LAMA's configuration \citep{richter2011lama}.
The resulting configuration named \textbf{\coolname}
outperformed LAMA,
SM-Type-LAMA \cite{kuroiwa2022biased},
and more recent advanced planners such as
NOLAN \cite{correa2025nolan}, DecStar \cite{gnad2018decstar}, and ApxNovelty \cite{singh2021approximate}.

This paper is organized in a way that alternates between
progressively introducing improvements to the base algorithm and then performing a small experiment.
All experiments are conducted on a compute cluster with Intel(R) Xeon(R) 6258R CPUs @ 2.70GHz
in the Agile setting, i.e.,
under a 5-minute time limit and the 8GB memory limit
(including the PDDL-to-SAS+ translation/grounding time), and with unit-cost conversion.
Each algorithm is implemented in Fast Downward \cite[FD]{Helmert2006} unless otherwise noted,
and is evaluated on benchmark instances in IPC2018 and IPC2023.
Randomized algorithms are evaluated with 5 seeds.
Our code is going to be published in \url{github.com/guicho271828/downward-mcts}.

\section{Preliminary}

\strips[transition-only,agile]

\heuristics[simplified,ff,cea,cg]

\subsection{Forward Heuristic Best-First Search}
\label{sec:mcts}

Classical planning problems are typically solved as a path finding problem
defined over a state space graph induced by the transition rules,
and the current dominant approach is based on \emph{forward search}.
Forward search maintains a set of search nodes called an \emph{OPEN list}.
They repeatedly
(1) (\emph{selection}) select a node from OPEN,
(2) (\emph{expansion}) generate its successor nodes,
(3) (\emph{evaluation}) evaluate the successor nodes, and
(4) (\emph{queueing}) reinsert them into OPEN.
Termination typically occurs when a node is expanded that satisfies a goal condition,
but a satisficing/agile algorithm can perform \emph{early goal detection},
which immediately checks whether any successor node generated in step (2) satisfies the goal condition.
Since this paper focuses on agile search, we use early goal detection for all algorithms.

Within forward search,
forward \emph{best-first} search defines a particular ordering in OPEN
by defining \emph{node evaluation criteria} (NEC) $f$ for selecting the best node in each iteration.
Let us denote a node by $n$ and the state represented by $n$ as $s_n$.
As NEC,
Dijkstra's search \citep{dijkstra1959note} uses $f_{\mathrm{Dijkstra}}(n)=g(n)$ ($g$-value), the minimum cost from the initial state $I$ to the state $s_n$ found so far.
\astar \citep{hart1968formal} uses $f_{\astar}(n)=g(n)+h(s_n)$, the sum of $g$-value and the value returned by a heuristic function $h$ ($h$-value).
Greedy Best First Search \citep{doran1966experiments,bonet2001planning} uses $f_{\gbfs}(n)=h(s_n)$.
Forward best-first search that uses $h$ is called forward \emph{heuristic} best-first search.

MCTS is a class of forward heuristic best-first search
that represents OPEN as the leaves of a tree.
We call the tree a \emph{\topen}.
Our MCTS is based on the description in \citep{keller2013trial,schulte2014balancing}.
Overall, MCTS works in the same manner as other best-first search with a few key differences.
(1) (\emph{selection}) To select a node from the \topen,
it recursively selects an action on each branch of the tree, start from the root, using the NEC
to select a successor node,
descending until reaching a leaf node.
(Sometimes the action selection rule is also called a \emph{tree policy}.)
At the leaf, it
(2) (\emph{expansion}) generates successor nodes,
(3) (\emph{evaluation}) evaluates the new successor nodes,
(4) (\emph{queueing}) attaches them to the leaf, and
\emph{backpropagates} (or \emph{backs-up}) the information to the leaf's ancestors, all the way up to the root.

The evaluation obtains a heuristic value $h(s_n)$ of a leaf node $n$.
In adversarial games like Backgammon or Go, it is obtained either by
(1) hand-crafted heuristics,
(2) \emph{playouts} (or \emph{rollouts})
where the behaviors of both players are simulated
by uniformly random actions (\emph{default policy}) until the game terminates,
or (3) a hybrid \emph{truncated simulation},
which returns a hand-crafted heuristic after performing a short simulation \citep{gelly2011monte}.
In recent work, the default policy is replaced by a learned policy \citep{alphago}.

Trial-based Heuristic Tree Search \citep[THTS]{keller2013trial,schulte2014balancing},
a MCTS for classical planning,
is based on two key observations:
(1) the rollout is unlikely to terminate in classical planning due to sparse goals,
unlike adversarial games, like Go, which are guaranteed to finish in a limited number of steps with a clear outcome (win/loss); and
(2) a \topen can efficiently handle a large number of node reordering due to the updates to NEC, and
thus is more flexible than a priority queue-based OPEN list, and
can readily implement standard search algorithms such as \astar and GBFS without significant performance penalty.
In this paper, we use THTS and MCTS interchangeably.

Finally, Upper Confidence Bound applied to trees \citep[UCT]{kocsis2006bandit}
is a MCTS that uses UCB1 \citep{auer2002finite} Multi-Armed Bandit \cite[MAB]{thompson1933likelihood,robbins1952some,bush1953stochastic} algorithm for action selection and became widely popular in adversarial games.
\citet{schulte2014balancing} proposed several variants of UCT including GreedyUCT (\guct),
which differs from UCT in that
the NEC assigned to the node is simply its heuristic value $h(s_n)$ just like in GBFS,
rather than $g(n)+h(s_n)$ like in \astar.
This paper
only discusses the greedy variants
due to our focus on the agile planning.

\citet{wissow2023scale} demonstrated that, for a MAB-based MCTS to work on a particular task,
the MAB algorithm must be chosen with careful theoretical considerations.
They highlighted that the rewards based on heuristic values have no upper bound,
unlike adversarial games which typically provide binary win (1) / loss (0) rewards,
thus they do not meet the theoretical requirement to be used with the commonly used UCB1 MAB algorithm:
To use UCB1, the reward distributions of all arms must have the common known finite support $[0,c]$.
As a result, combining heuristics-based rewards with UCB1 means
the algorithm does not guarantee any of UCB1's known theoretical properties,
such as a logarithmic regret bound, leading to poor performance.

They then demonstrated that MAB algorithms that assume Gaussian reward distributions
can handle heuristics-based rewards properly due to the assumption of infinite support $\R$.
They used two Gaussian bandits, UCB1-Normal \citep{auer2002finite} and UCB1-Normal2 \citep{wissow2023scale},
where the latter outperformed the former as well as other finite-support bandits.
UCB1-Normal2 is a frequentist version of Bayes-UCB2 \citep{tesauro2010bayesian}.
We refer to the MCTS that uses UCB1-Normal2 as GreedyUCT-Normal2 (GUCTN2).
\begin{align}
 \text{UCB1}_i &\textstyle= {\hat{\mu}_i + c\sqrt{{(2\log T)}/{t_i}}}\\
 \text{UCB1-Normal2}_i &\textstyle= \hat{\mu}_i + \hat{\sigma}_i \sqrt{ 2\log T }
 \label{eq:ucb-normal2}
\end{align}

\refalgo{alg:mcts} shows the pseudocode of MCTS adjusted for graph search, taken from \citep{schulte2014balancing}.
Since efficient graph search algorithms must avoid visiting the same state multiple times,
\refalgo{alg:mcts} marks certain nodes as \emph{locked}, and excludes them from the selection candidates.
In an offline search,
a node is locked
either (1) when a node is a dead-end that will never reach a goal
(detected by having no applicable actions, by a heuristic function, or other facilities),
(2) when there is a node with the same state in the search tree with a smaller g-value, 
or
(3) when all of its children are locked.

\refalgo{alg:mcts} differs slightly from existing work by omitting \emph{tree grafting},
which involves moving a subtree to a new parent
when it is \emph{reopened} (reached via a path with a smaller $g$-value).
While it could help satisficing search,
it is generally detrimental in agile search.
Appendix \refsec{sec:tree-grafting} contains additional evaluations with this feature.

\begin{algorithm}[tbh]
 \begin{algorithmic}[1]
  \WHILE{True}
  \STATE Parent node $p\gets r$ \label{line:mcts-selection-begin}
  \WHILE[Selection]{not leaf $p$}
  \STATE $p\gets \argmin_{n\in S(p)} f(n)$
  \ENDWHILE \label{line:mcts-selection-end}
  \FOR[Expansion]{a child node $n\in S(p)$} \label{line:mcts-expansion-for}
  \STATE \textbf{return} $n$ \textbf{if} $n$ is a goal. \COMMENT{Early goal detection} \label{line:mcts-expansion-begin}
  \STATE \textbf{continue} to \refline{line:mcts-expansion-for} \textbf{if} state $s_n$ is not new
  \STATE compute $h(s_{n})$ \COMMENT{Evaluation}
  \label{line:mcts-expansion-end}
  \ENDFOR
  \STATE $B\gets B\cup\{p\}$
  \WHILE[Backpropagation]{\textbf{not} $B.\function{empty}()$} \label{line:mcts-backprop-begin}
  \STATE $n\gets B.\function{popmax}()$ \COMMENT{largest $g$ = depth first}
  \STATE Update $n$'s statistics and its lock status from $S(n)$
  \IF{$n$'s statistics or its lock status has changed}
  \STATE $B\gets B\cup\{n.\function{parent}\}$
  \ENDIF
  \label{line:mcts-backprop-end}
  \ENDWHILE
  \ENDWHILE
 \end{algorithmic}
 \caption{
 High-level general MCTS.
 \textbf{Input}:
 Root node $r$,
 successor function $S$,
 NEC $f$,
 heuristic function $h$,
 ordered set $B$ sorted by $g$.
 Initialize $\forall n; g(n)\gets \infty$, $B\gets\emptyset$.
}
 \label{alg:mcts}
\end{algorithm}

\section{\green{Bilevel MCTS}}

\begin{algorithm}[tb]
 \caption{Bilevel MCTS. (\green{The difference from \refalg{alg:mcts}})}
 \label{alg:bilevel-mcts}
 \begin{algorithmic}[1]
  \WHILE{True}
  \STATE Parent node $p\gets r$, \green{Budget $b\gets 0$}
  \WHILE[Selection]{not leaf $p$}
  \STATE $p\gets \argmin_{n\in S(p)} f(n)$, \green{$b\gets b+1$ \label{line:budget}}
  \ENDWHILE
  \STATE \green{Priority queue sorted by $f$: $Q \gets \braces{p}$} \label{line:Q}
  \WHILE[BFS]{\green{$b>0$ \textbf{and not} $Q.\function{empty}()$}}
  \STATE \green{Parent node $p\gets Q.\function{popmin}()$}, \green{$b\gets b-1$}
  \FOR[Expansion]{a child node $n\in S(p)$}
  \STATE [identical to \reflines{line:mcts-expansion-begin}{line:mcts-expansion-end} in \refalg{alg:mcts}]
  \STATE \green{$Q \gets Q\cup \braces{n}$}
  \ENDFOR
  \STATE \cyan{$p\gets \function{collapse}(p, \theta)$} \textbf{if} $\theta$ is optionally given \label{line:collapse}
  \STATE $B\gets B\cup\{p\}$
  \ENDWHILE
  \STATE [identical to \reflines{line:mcts-backprop-begin}{line:mcts-backprop-end} in \refalg{alg:mcts}]
  \ENDWHILE
 \end{algorithmic}
\end{algorithm}

Theoretical regret bounds of MAB algorithms are often functions of the number of total pulls $T$,
e.g.,
the regret of an \emph{asymptotically optimal} MAB is logarithmically upper bounded by $O(\log T)$.
In the application to the heuristic search,
$T$ corresponds to the total number of node evaluations in the current subtree.
These theoretical guarantees \emph{could} make MCTS efficient if the performance is measured by the number of node evaluations,
but in reality it is a moot point because node evaluations are not the sole bottleneck in a search algorithm.
Namely,
under a simplified assumption,
the runtime of \emph{(1) selection} step in a tree-based OPEN list
grows logarithmically to the size of OPEN.
\begin{theo}
 Assume that the search space forms a tree with a constant branching factor $B$, and
 we have a tree-based OPEN list of depth $D$, containing $N = B^D$ leaves.
 The runtime of \emph{selection} step is $BD = O(\log N)$.
 \label{theo:action-selection-complexity}
\end{theo}

To demonstrate this,
in the scatter plot in \refig{fig:evalsec-bilevel} (left),
we compare
the number of node evaluations per second using $\ff$ heuristics
between GBFS and GUCTN2.
The latter sometimes has more than three orders of magnitude slower evaluations compared to the former,
indicating that
action selections in MCTS can slow down the search significantly.
(See appendix \refig{fig:evalsec-bilevel-more} for larger plots with more heuristics.)

Further, to empirically validate that
our cost model in \reftheo{theo:action-selection-complexity} under the simplified assumption
applies to the actual instances,
we visualized the correlation between the node processing speed and
the average depth of the evaluated nodes on instances successfully solved by GUCTN2 with $\ff$.
\refig{fig:evalsec-correlation} shows an inverse correlation between the search depth
and the number of node evaluations per second.
(See appendix \refig{fig:evalsec-correlation-more} for larger plots with more heuristics.)

\begin{figure*}[tb]
 \centering
 \includegraphics[height=0.22\linewidth]{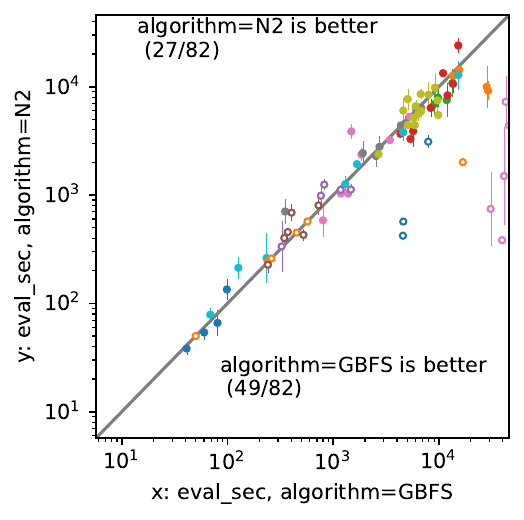}
 \includegraphics[height=0.2\linewidth]{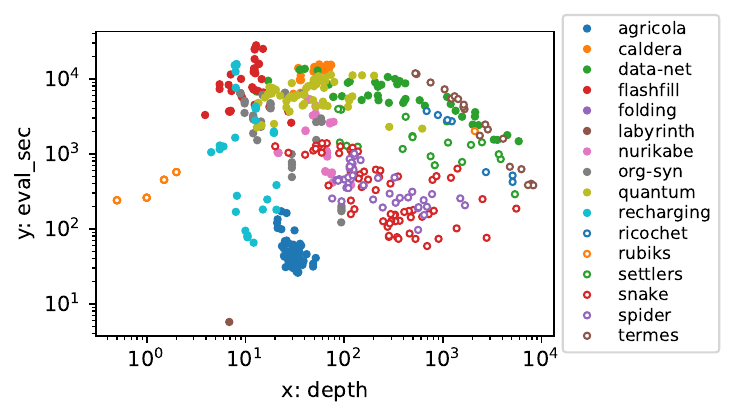}
 \includegraphics[height=0.2\linewidth]{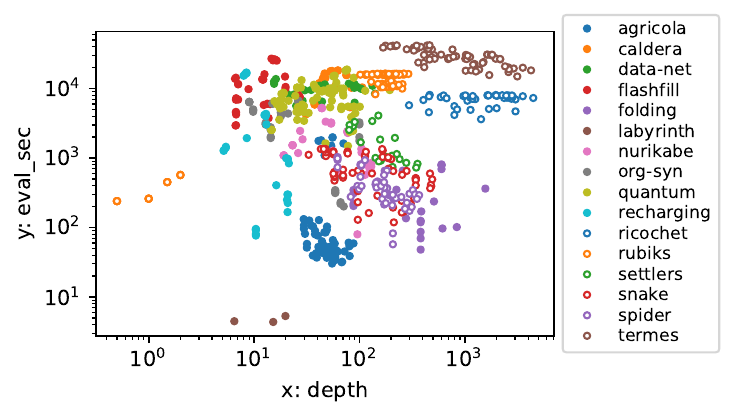}
 \caption{
 \textbf{Left:}
 Comparing the number of node evaluations per second
 on IPC instances solved by both GBFS ($x$-axis) vs. GUCTN2 ($y$-axis) within the limit, using $\ff$ heuristics.
 The points below the diagonal indicate that GUCTN2 has significantly slower node evaluations.
 \textbf{Middle, Right:}
 Log-log plots comparing the number of node evaluations per second ($y$-axis)
 versus the average depth of the nodes evaluated during the search ($x$-axis)
 for $\ff$.
 GUCTN2 (middle) shows that the search becomes slower as it goes deeper.
 \green{Bilevel GUCTN2} (right) explores deeper yet shows less degradation in the node/sec.
 The effect is pronounced in \brown{\textbf{termes}}, \darkblue{\textbf{ricochet}}, \orange{\textbf{rubiks}}, \green{\textbf{data-network}}.
 }
 \label{fig:evalsec-bilevel}
 \label{fig:evalsec-correlation}
\end{figure*}

Based on these theoretical and empirical observations, we propose \emph{Bilevel MCTS} (\refalgo{alg:bilevel-mcts}),
a modification to MCTS that performs a budgeted best-first search (BFS) from the leaf node reached by the action selection.
It uses a priority queue $Q$ (\refline{line:Q}) for the BFS,
and the runtime complexity of its \function{popmin} operation affects the runtime of \emph{selection} in MCTS.
It has the following theoretical characteristics (Proofs in the appendix \refsec{theo:action-selection-complexity-bilevel-proofs}):
\begin{theo}
 \label{theo:action-selection-complexity-bilevel-heap}
 In the Bilevel MCTS,
 the amortized runtime of each \emph{selection} is $O(\log \log N)$ if $Q$ is a heap.
\end{theo}
Although this is already an improvement over the original $O(\log N)$,
we are inspired by \citet{burns2012implementing}
to improve it further
by leveraging the unit cost structure of agile classical planning
with array-based priority queues \citep{dial1969shortest}
whose \function{popmin} have an $O(1)$ runtime.
\begin{theo}
 \label{theo:action-selection-complexity-bilevel-bucket}
 In the Bilevel MCTS,
 the amortized runtime of each \emph{selection} is $O(1)$ if $Q$ is an array-based priority queue.
\end{theo}

Our bilevel modification \emph{sacrifices search efficiency to gain low-level efficiency},
as it deviates from the policy suggested by the bandit.
Therefore,
we must empirically verify that the trade-off pays off in finding the goal faster.

We compared \green{Bilevel GUCTN2} with GUCTN2, as well as baseline GBFS
and a \lsota diversified search algorithm Softmin-type(h) \citep{kuroiwa2022biased}
by measuring the number of instances solved within the time limit (\# solved)
and the Agile IPC score $\sum_i \min \braces{1,1-\nicefrac{\log t_i}{\log 300}}$ where $t_i>0$ is the runtime for solving instance $i$ in seconds.
We used $\ff$, $\cg$, and $\ce$ heuristics for this evaluation
in order to demonstrate the heuristic independence.
\reftbl{tbl:base-table} shows that
\green{Bilevel GUCTN2} (base) significantly outperforms the standard GUCTN2 (base) and GBFS.
We also observed that
GUCTN2 struggles in IPC23,
which \green{Bilevel GUCTN2} improves significantly.
As a result, \green{Bilevel GUCTN2} outperforms Softmin
in the IPC18+23 total \# solved and IPC score in $\ce$ and $\ff$.
\green{Bilevel GUCTN2} did not outperform Softmin
in a relatively cheap $\cg$ heuristic, which still emphasizes the remaining bottleneck in MCTS,
but it narrowed the gap significantly compared to GUCTN2.
The detailed domain-wise breakdown is available in the Appendix \refsec{sec:eager-domainwise}.
\refig{fig:evalsec-correlation} (right) also shows that
\green{Bilevel GUCTN2} significantly improves the node processing speed over GUCTN2,
especially in a deeper search.

This tradeoff is surprising
because the majority of its search is
performed by the BFS when the search is deep.
In other words,
with a strong exploration policy (UCB1-Normal2),
such a small amount of exploration is often sufficient
to gain significant performance improvements.

\begin{table*}[tb]
\centering
 \begin{adjustbox}{max width=\linewidth}
\setlength{\tabcolsep}{0.3em}
\begin{tabular}{r|r|rr|rrrrrrr|rrrrrrr|r|rrrr}
\toprule
 &      & GBFS & Soft & base & \multicolumn{6}{c|}{GUCTN2} & base & \multicolumn{7}{c|}{\green{Bilevel GUCTN2}} & \multicolumn{4}{c}{\magenta{Fixed Budget Bilevel}}\\
 & $h$, year &      & min
 & \cyan{$\theta=0$} & \cyan{10} & \cyan{20} & \cyan{40} & \cyan{80} & \cyan{160} & \cyan{320}
 & \cyan{$\theta=0$} & \cyan{10} & \cyan{20} & \cyan{40} & \cyan{80} & \cyan{160} & \cyan{320} & \orange{DTC}
 & $b=10$ & 100 & 1000 & 10000
 \\
\midrule\multirow{7}{*}{\rotatebox{90}{\# solved}}
 & $\ce$, \textit{ipc18} & 40 & 48.4 & 50.2 & 53.4 & 51.4 & 50 & 51.2 & 52.6 & 50 & 53.8 & 56.4 & \textbf{61.4} & 58.6 & 58.8 & 57.8 & 58.2 & \textbf{61} & 57.6 & 51.8 & 49 & 46.6 \\
 & $\ce$, \textit{ipc23} & 30 & 34.2 & 25.8 & 26.6 & 27 & 27.2 & 25.4 & 26.4 & 26.4 & 30.8 & 32 & 32.2 & 34.2 & 33 & 34.4 & \textbf{34.8} & 33.4 & 32.4 & \textbf{36.8} & 33.8 & 33.6 \\
 & $\cg$, \textit{ipc18} & 35 & 59.2 & 42.4 & 43.2 & 44.6 & 44.4 & 45.6 & 41.4 & 37.8 & 56.2 & 58.4 & 59.6 & \textbf{60.2} & \textbf{59.8} & 56.2 & 57.2 & 57.6 & 51.6 & 55.8 & 48.6 & 48.8 \\
 & $\cg$, \textit{ipc23} & 36 & 39.4 & 25 & 25.4 & 27.2 & 26.8 & 28 & 26.8 & 27.6 & 41.6 & 41.6 & 42 & \textbf{45.8} & 45.4 & 43.4 & 42 & \textbf{46.4} & 34.2 & 42.2 & 44 & 42.2 \\
 & $\ff$, \textit{ipc18} & 60 & 80.2 & 79.8 & 82 & 79.8 & 82.8 & 77.8 & 81.4 & 79.6 & 87.4 & 87.2 & 88.6 & 90.4 & \textbf{90.6} & \textbf{91.4} & 87 & 90.2 & 85.8 & 83.4 & 82.2 & 72.8 \\
 & $\ff$, \textit{ipc23} & 51 & 55.4 & 23.8 & 24.6 & 27.2 & 27.2 & 25.8 & 24.8 & 25.6 & 58.2 & 56.8 & 57.8 & 58 & 56.2 & 55 & 55 & 58.4 & 45.4 & 56.8 & \textbf{60} & \textbf{58.6} \\
\cmidrule{2-23} & \textit{total} & 252 & 316.8 & 247 & 255.2 & 257.2 & 258.4 & 253.8 & 253.4 & 247 & 328 & 332.4 & 341.6 & \textbf{347.2} & 343.8 & 338.2 & 334.2 & \textbf{347} & 307 & 326.8 & 317.6 & 302.6 \\
\midrule\multirow{7}{*}{\rotatebox{90}{Agile IPC Score}}
 & $\ce$, \textit{ipc18} & 17.7 & 20.6 & 23.8 & 23.4 & \textbf{24.5} & 23.8 & 23.7 & 24.1 & 23.4 & 22.4 & 23.4 & \textbf{25.2} & 23.6 & 23.9 & 23.3 & 23.4 & 24.4 & 23.1 & 22.1 & 20.8 & 18.8 \\
 & $\ce$, \textit{ipc23} & 23.1 & 24.1 & 21.5 & 21.5 & 22.3 & 22.3 & 20.9 & 21.3 & 20.8 & 23.8 & 24.3 & 24.3 & \textbf{25.3} & 23.9 & 25.2 & 25.0 & 24.2 & 25.0 & \textbf{26.3} & 24.8 & 23.7 \\
 & $\cg$, \textit{ipc18} & 18.4 & 27.2 & 19.9 & 19.7 & 21.1 & 21.6 & 21.5 & 20.4 & 19.6 & 26.1 & 27.6 & 28.6 & \textbf{30.0} & \textbf{28.7} & 27.7 & 26.8 & 28.0 & 23.9 & 25.1 & 23.3 & 21.8 \\
 & $\cg$, \textit{ipc23} & 23.6 & 24.8 & 16.4 & 16.9 & 18.6 & 18.4 & 18.0 & 17.8 & 19.0 & 25.5 & 23.7 & 24.9 & \textbf{25.9} & 25.7 & 25.3 & 22.9 & \textbf{26.5} & 22.0 & 24.4 & 24.0 & 24.6 \\
 & $\ff$, \textit{ipc18} & 30.4 & 37.7 & 40.4 & 41.4 & 41.7 & 41.1 & 40.1 & 39.8 & 38.0 & 43.0 & 44.2 & 44.6 & 44.9 & \textbf{45.7} & \textbf{46.4} & 43.9 & 45.5 & 44.3 & 43.8 & 39.5 & 35.6 \\
 & $\ff$, \textit{ipc23} & 31.0 & \textbf{35.5} & 19.0 & 19.1 & 21.5 & 20.6 & 19.9 & 20.0 & 20.6 & 32.9 & 32.7 & 33.8 & \textbf{34.4} & 34.4 & 32.2 & 32.0 & 33.7 & 27.3 & 34.0 & 32.7 & 33.4 \\
\cmidrule{2-23} & \textit{total} & 144.2 & 169.9 & 141.0 & 142.1 & 149.7 & 147.8 & 144.0 & 143.3 & 141.2 & 173.8 & 176.0 & 181.4 & \textbf{184.1} & 182.3 & 180.1 & 174.0 & \textbf{182.4} & 165.6 & 175.7 & 165.2 & 158.1 \\
 \bottomrule
\end{tabular}

 \end{adjustbox}
 \caption{
 \textbf{Synopses}:
 Evaluating the effect of MCTS modifications (\green{Bilevel}, \cyan{Tree Collapsing}, \orange{Dynamic Tree Collapsing}, \magenta{Fixed Budget Bilevel})
 and the baselines (\gbfs, Softmin-type(h)) over multiple heuristics $h\in\braces{\ce,\cg,\ff}$.
 Each cell shows the number of IPC18+23 instances solved within 5 min. and 8GB,
 or the Agile IPC score $\sum_i \min \braces{1, 1-\nicefrac{\log t_i}{\log 300}}$ where $t_i>0$ is the runtime for instance $i$ in seconds.
 \textbf{Bold scores} indicate top-2 (including ties).
 \textbf{Summary}:
 \green{Bilevel} generally improves GUCTN2.
 \cyan{Tree Collapsing} generally improves \green{Bilevel}.
 \cyan{Tree Collapsing} does not significantly improve GUCTN2 by itself except $\ff$ in IPC 2023 (thus overall better total).
 The best \cyan{$\theta$} tends to be between 40 and 80, but depends on the domain and $h$.
 \orange{DTC} avoids hyperparameter tuning and is overall in the second place.
 \magenta{The fixed budget bilevel} improved over GUCTN2,
 but performed significantly worse than \green{the dynamic depth-based budget strategy}.
  }
 \label{tbl:base-table}
\end{table*}

\subsection{\green{Depth-Based} vs. \magenta{Fixed} Budget Strategy}
\label{sec:budget}

The budget $b$ of BFS in \green{Bilevel MCTS} is
automatically adjusted by the depth of the selected node (\refline{line:budget}).
This dynamic budgeting is indeed crucial in the runtime complexity analysis
(\reftheos{theo:action-selection-complexity-bilevel-heap}{theo:action-selection-complexity-bilevel-bucket}),
but it also has a significant empirical effect.

To demonstrate this,
we evaluated a variant named \magenta{Fixed Budget Bilevel MCTS},
where the BFS budget $b$ is a hyperparameter $\in \braces{10,100,1000,10000}$.
This variant spends too much time on the BFS when the tree is shallow,
and too little when the tree is deep.
\reftbl{tbl:base-table} (rightmost columns)
shows that, while
\magenta{Fixed Budget Bilevel GUCTN2} improved the \#solved over the base GUCTN2
largely thanks to $\ff$+IPC2023 configurations (+34 instances),
it performed much worse overall compared to \green{Bilevel GUCTN2} whose budget is dynamic.
The Agile IPC scores also tend to show inferior performance.

\subsection{\cyan{Tree Collapsing}}

We propose a further modification to \green{Bilevel-MCTS}, called \cyan{\emph{Tree Collapsing}}.
The number of NEC computations grows proportionally to the depth of the tree.
However, many action selection steps are not informative enough
because some nodes contain only a small number of children.
Since each selection in MCTS narrows down the set of candidate leaf nodes for expansion,
the nodes with fewer children fails to reduce the subset \emph{fast enough}.

We address this issue by collapsing the tree (see \refig{fig:collapse} and \refalgo{alg:collapse})
when the width of the tree is below a certain threshold $\theta$ given as a hyperparameter.
Let $p$ be an expanded node,
$S(p)$ be its successors,
and $p'$ be the parent of $p$.
Before inserting $p$ into the backpropagation queue $B$ (\refalgo{alg:bilevel-mcts}, \refline{line:collapse}),
we compute $S(p')+S(p)-1$.
If it is less than $\theta$,
we deem that $p'$ should have more children.
Then we connect each $n\in S(p)$ directly to $p'$ and discard $p$.
Note that this happens separately from the table that tracks the predecessor of each node,
which is necessary for extracting the final plan.

In addition, we remove the need for hyperparameter tuning with \orange{\emph{{Dynamic Tree Collapsing (DTC)}}}.
DTC is inspired by the success of dynamic budgeting ---
The collapsing threshold $\theta$ is dynamically set to the depth of the node being collapsed.

We evaluated \cyan{Tree Collapsing} with various hyperparameter \cyan{$\theta$},
for both the base GUCTN2 and the \green{Bilevel GUCTN2}.
The result (\reftbl{tbl:base-table}) shows that
it generally improves \green{Bilevel}
in both \# solved and the IPC scores.
We also observed improvements in the node/sec,
supporting our hypothesis.
Due to space, we included the plots in the appendix \refsec{sec:evalsec-collapsing}.
However, it tends to have a negligible impact
when applied to the base GUCTN2, except in $\ff$+IPC 2023 configuration
where both the \#solved and the IPC scores have improved.

This indicates a synergetic effect between \green{Bilevel} and \cyan{Tree Collapsing},
which is explained by its interaction with the queue operations in the backpropagation as follows.
In \green{Bilevel MCTS}, backpropagation is paused during the BFS,
and the expanded nodes wait in the backpropagation queue $B$.
Meanwhile, \cyan{\function{collapse}} often returns the same grandparent node
which are detected by $B$ as duplicates.
This results in a significant reduction of the size of $B$,
as well as the runtime for the backpropagation.
This additional effect is not present in the base MCTS because
backpropagation happens after every expansion.

Next, we evaluated \orange{DTC}.
It managed to closely match the performance of the best hyperparameter \cyan{$\theta=40$},
demonstrating its ability to eliminate the need for hyperparameter tuning.
To avoid the tuning, we use the \orange{DTC} exclusively hereafter.

\begin{algorithm}[tb]
 \caption{
 $\function{\cyan{collapse}}(\text{Parent}\ p, \text{threshold}\ \theta)$
 }
 \label{alg:collapse}
 \begin{algorithmic}[1]
  \IF{$p$ is not root}
  \STATE Grandparent $p'\gets p.\function{parent}$
  \IF{$|S(p')|+|S(p)|-1<\theta$}
  \STATE $S(p')\gets (S(p') \setminus \braces{p}) \cup S(p); \quad p\gets p'$
  \ENDIF
  \ENDIF
  \RETURN $p$
 \end{algorithmic}
\end{algorithm}

\begin{figure}[tb]
 \centering
 \includegraphics[width=0.8\linewidth]{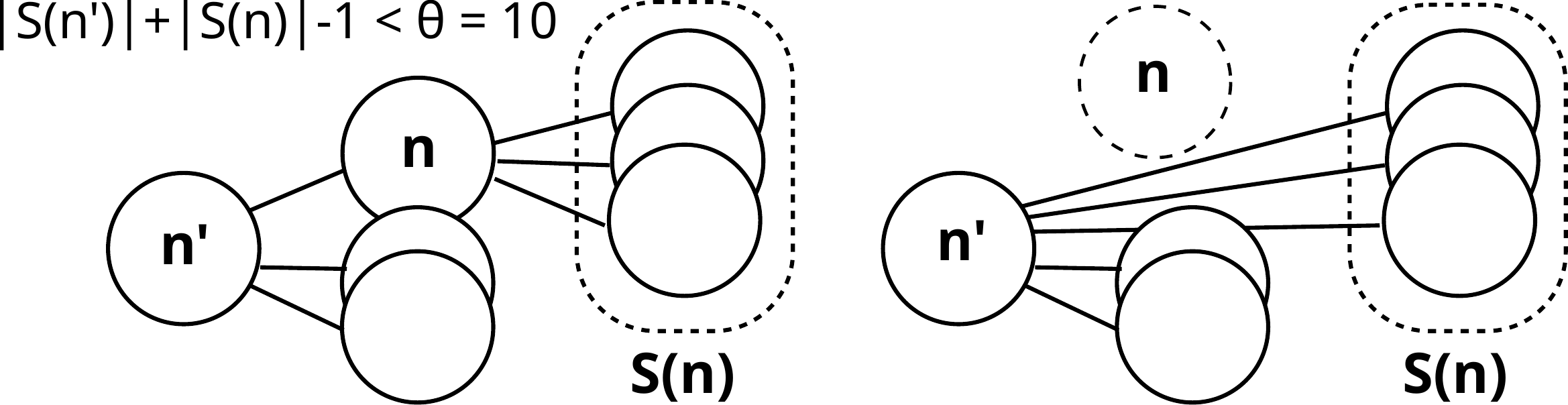}
 \caption{Example illustrating \cyan{Tree Collapsing} with $\theta=10$.}
 \label{fig:collapse}
\end{figure}

\subsection{Lazy Evaluation}
\label{sec:lazy}

Finally, we combined GUCTN2 with lazy evaluation,
a well-known enhancement that trades heuristic accuracy with speed by
assigning the parent's heuristic value to the children, instead of the children's own value.
The results are primarily negative, and we discuss this in the appendix (\reftbl{tbl:lazy-table}).

\section{Orthogonality from Existing Techniques} 

\begin{table}[tb]
 \begin{adjustbox}{max width=\linewidth}
 \begin{tabular}{rl}
  LAMA & $[\ff, \function{pr}(\ff), \ar{l}, \function{pr}(\ar{l})]$\\
  SM-Type-LAMA & $[\ff, \function{pr}(\ff), \function{sm}([\ff,g]), \ar{l}, \function{pr}(\ar{l}), \function{sm}([\ar{l},g]) ]$\\
  NOLAN & $[\ff, \function{pr}(\ff), \ar{l}, \function{pr}(\ar{l}), [w_{\ar{l}}, \ar{l}] ]$\\
  \midrule
  LAMAe+BFWS& $[[w_{\ar{l}, \ff}, \ff], \function{pr}(\ff), [w_{\ar{l}, \ff}, \ar{l}], \function{pr}(\ar{l}) ]$\\
  \green{N2}\orange{DTC}+BFWS & $\function{n2}([w_{\ff}, \ff])$\\
  \green{N2}\orange{DTC}+LAMAe & $[\function{n2}([\ff]), \function{pr}(\ff), \function{n2}([\ar{l}]), \function{pr}(\ar{l}) ]$\\
  \midrule
  \coolname & $[\function{n2}([w_{\ar{l}, \ff}, \ff]), \function{pr}(\ff), \function{n2}([w_{\ar{l}, \ff}, \ar{l}]), \function{pr}(\ar{l}) ]$\\
 \end{tabular}
 \end{adjustbox}
 \caption{
 Summary of the configurations.
 LAMAe denotes the same configuration as LAMA, but with eager evaluations.}
 \label{tbl:configurations}
\end{table}

\begin{table*}[tb]
 \centering
 \begin{adjustbox}{max height=0.15\linewidth}
\begin{tabular}{rrc|rr|rr||rrrr||rrrrr}\toprule
& & IPC   & \multicolumn{2}{c|}{Lapkt-BFWS}           & \multicolumn{2}{c||}{FD-BFWS} & \multicolumn{2}{c}{LAMA} & Dec  & NO   & \multirow{2}{*}{LAMAe} & LAMAe  & \green{N2}\orange{DTC} & \green{N2}\orange{DTC} & \multirow{2}{*}{\coolname}\\
& & year  & Apx$^{\text{fd}}$ & $\vf_5$$^{\text{fd}}$ & $\vf_4$ & $\vf_5$             &  & +SM                   & Star & LAN  &                        & +BFWS  & +BFWS                  & +LAMAe                 & \\
\midrule
\multirow{6}{*}{\rotatebox{90}{5 min. agile}} &
\multirow{3}{*}{\rotatebox{90}{\# solved}} &
     ipc18 & 99  & 83  & 109 & 92  & 111 & \textbf{120.4} & 99  & 105         & 98 & 100         & 107   & 94.6  & \textbf{122}   \\
 & & ipc23 & 67  & 67  & 55  & 71  & 66  & 66             & 63  & \textbf{72} & 61 & \textbf{72} & 60.4  & 61.6  & 70.2           \\
 & & total & 166 & 150 & 164 & 163 & 177 & \textbf{186.4} & 162 & 177         & 159& 172         & 167.4 & 156.2 & \textbf{192.2} \\
\cmidrule{2-16}  & \multirow{3}{*}{\rotatebox{90}{score}} &
     ipc18 & 51.3 & 41.4 & 53.9 & 45.3          & 57.1 & \textbf{60.3} & 51.1 & 47.5          & 47.2& 45.8 & 52.1 & 44.3 & \textbf{60.0} \\
 & & ipc23 & 39.0 & 39.9 & 36.4 & \textbf{43.9} & 38.3 & 39.0          & 42.6 & \textbf{45.2} & 38.2& 39.3 & 34.7 & 34.8 & 39.4          \\
 & & total & 90.3 & 81.3 & 90.2 & 89.2          & 95.4 & \textbf{99.3} & 93.7 & 92.8          & 85.4& 85.2 & 86.8 & 79.1 & \textbf{99.4} \\
\midrule
\multirow{6}{*}{\rotatebox{90}{30 min. extended}} &
\multirow{3}{*}{\rotatebox{90}{\# solved}} &
     ipc18 & 118 & 90  & 132 & 102 & 123 & 133.4 & 110 & 133         & 117& \textbf{140} & 126.8 & 111.6 & \textbf{145}   \\
 & & ipc23 & 83  & 69  & 76  & 72  & 81  & 80    & 70  & \textbf{86} & 69& \textbf{86}  & 80.4  & 73.4  & 85.6           \\
 & & total & 201 & 159 & 208 & 174 & 204 & 213.4 & 180 & 219         & 186& \textbf{226} & 207.2 & 185   & \textbf{230.6} \\
\cmidrule{2-16} & \multirow{3}{*}{\rotatebox{90}{score}} &
     ipc18 & 65.4  & 52.5 & 69.8  & 58.1          & 71.0  & \textbf{76.3}  & 63.6  & 65.0          & 62.1  & 63.5  & 67.8  & 58.4  & \textbf{77.9}  \\
 & & ipc23 & 48.5  & 46.8 & 43.8  & \textbf{50.4} & 46.6  & 47.2           & 48.4  & \textbf{53.7} & 45.0  & 49.2  & 44.0  & 42.7  & 48.9           \\
 & & total & 113.9 & 99.2 & 113.7 & 108.6         & 117.6 & \textbf{123.4} & 112.0 & 118.7         & 107.1 & 112.8 & 111.8 & 101.0 & \textbf{126.8} \\
\bottomrule
\end{tabular}
 \end{adjustbox}
 \caption{
Comparing \lsota planners.
Top-2 (including ties) are highlighted in \textbf{bold}.
Synopsis:
Apx$^{\text{fd}}$ = Approximate Novelty Search with FD-based grounder,
$\vf_5^{\text{fd}}$ = The original Lapkt-based BFWS with FD-based grounder,
$\vf_4, \vf_5$ = Our BFWS reimplementations,
+SM = Softmin-Type-LAMA.
}
 \label{tbl:bfws-table}
\end{table*}

Existing \sota planners have employed a number of enhancements to achieve the performance,
including
\emph{width-based search} \citep{lipovetzky2017bfws},
\emph{alternation queue} \citep{RogerH10},
and \emph{boosted preferred operators} \citep{RichterH2009}.
Interestingly,
these techniques all address different aspects of exploration
(i.e., deviates from the best-first behavior based on cost-to-go estimates).
Width-based search encourages exploring states that are dissimilar to states visited before,
the alternation queue tries to diversify the heuristics to use,
and bandits in MCTS perform statistical reasoning on a single heuristic
to decide whether to trust it more (exploit) or less (explore).

Since each technique targets a different aspect of exploration,
we naturally expect them to offer orthogonal and modular performance improvement,
which we demonstrate in this section through extensive evaluations.

\paragraph{Novelty Metric}

Before evaluating our novelty-enhanced MCTS later,
we must first compare our novelty metric implementation with existing work
in order to distinguish the effect of the implementation difference.

Novelty metric is a metric can be seen as a soft approximation of close list:
It tries to rank states
by checking how dissimilar a newly generated state is from the set of states that have been generated.
Best-first search algorithms that incorporate the novelty metric as part of their sorting criteria are
generally referred to as Best-First Width Search (BFWS).
The metric is defined as the minimum size of the conjunctions of propositions that have never been true
in a set of states visited before.
The set is divided into subsets by their heuristic values,
and the metric is measured for each subset.
Since computing the novelty with large conjunctions requires an exponentially large storage,
practical implementations often restrict the size of conjunctions (thus the maximum value of the metric) to $w_{\max}=2$,
or try to approximate the computation through a bloom filter \citep[ApxNovelty]{singh2021approximate}.
More recently, NOLAN \cite{correa2025nolan} proposed
selecting $w_{\max}=1$
when the number $|P|$ of propositions exceeds a threshold $V=100$, and $w_{\max}=2$ otherwise.
We reimplemented the metric in FD. 
For a controlled experiment, all configurations use the FD-based grounder,
although Lapkt-based planners could use FF- or Tarski-based grounders.

\citeauthor{lipovetzky2017bfws} claimed that
the open-list configuration which lexcographically sorts the nodes
with cheap heuristics $\vf_5=[w_{\ar{g},\ar{r}},\#g]$ achieved a good result,
where $\ar{g}$ denotes goal-count heuristics \citep{fikes1972strips} and
$\ar{r}$ counts how many propositions in the relaxed plan's actions' preconditions became true so far
since the last time $\ar{g}$ decreased in the path from the root.
The low computational cost makes $\vf_5$ an ideal baseline for comparing the low-level efficiency.
\reftbl{tbl:bfws-table} shows that our reimplementation of $\vf_5$
(appendix \refsec{sec:novelty-implementation} for implementation details) was faster than the original Lapkt-based BFWS.
Both implementations use the maximum novelty 2.
NOLAN's dynamic $w_{\max}$ degraded the performance (appendix \refsec{sec:nolan-dynamic-width}),
presumably because their $V=100$ was tuned from pre-IPC14 domains.
We also observed that
their $\vf_4=[w_{\ar{l}, \ff},\ar{l},\ff]$ performed better overall (\reftbl{tbl:bfws-table}),
where $\ar{l}$ is landmark-sum heuristic \citep{buchner2023landmark}.
Based on these results, we use $w_{\max}=2$ and $\vf_4$ as the base of our novelty-enhanced MCTS.
We also reproduced $\vf_6$, the second phase of Dual-BFWS \citep{lipovetzky2017bfws},
which underperformed (appendix \refsec{sec:f6}).

\paragraph{Alternation with Boosted Preferred Operators}

\emph{Alternation queue} \cite{RogerH10} 
is a method for combining multiple heuristics
by expanding nodes from different queues in a round-robin manner.
A \emph{preferred operator queue} is a priority queue that is populated only by
the nodes generated by the operators that are marked by a heuristic function as preferred,
and is often used as a member of alternated queues.
When a preferred operator queue is \emph{boosted},
the alternation queue watches the smallest value returned by each heuristics observed so far.
When the value is updated, it expands from the preferred operator queues exclusively
for a fixed number of iterations (the boost value).
Let $\function{pr}(h)$ denote a preferred operator queue for a heuristic $h$.
LAMA \citep{richter2011lama}, which won IPC11, alternates 4 queues using the lazy evaluation.
\citet{kuroiwa2022biased} added two Softmin-Type(h) queue (denoted as $\function{sm}$) to LAMA,
resulting in a 6-queue SM-Type-LAMA that outperformed LAMA.
Both configurations use boost=1000.
See \reftbl{tbl:configurations} for a summary of configurations.
Note that, due to the agile setting,
LAMA and variants evaluated here perform only the first iteration of the anytime search.

\paragraph{\coolname}

We combined
\green{Bilevel GUCT\textbf{N2}} + \orange{DTC},
\textbf{n}ovelty \textbf{B}FWS, and \textbf{LA}MA
into a configuration named \emph{\coolname}.
Let $\function{n2}([w_{\ar{l}, \ff}, h])$ denote a tree-based OPEN list
that backpropagates the minimum $w_{\ar{l}, \ff}$ (Full-Bellman backup)
and the average and the standard deviation of $h$ (Monte-Carlo Backup) from the children.
In each action selection, it selects the child with the best $w_{\ar{l}, \ff}$,
then breaks ties with UCB1-Normal2$_i$ (\refeq{eq:ucb-normal2}).

Due to the negative result of lazy evaluations (\refsec{sec:lazy}) with MCTS,
\coolname is based on \textbf{LAMAe} which uses eager evaluations and boost=10
(eager evaluation can rely less on preferred operators due to the accurate estimates.)
It also avoids the issue of evaluating the novelty metric lazily,
whose effect has not been studied in the literature (outside the scope of this paper).
Then
\coolname replaces the standard best-first queues in LAMAe with $\function{n2}([w_{\ar{l}, \ff}, h])$,
rather than adding new queues like in SM-Type-LAMA,
resulting in a simpler 4-queues configuration.

We compared \coolname with \textbf{ApxNovelty}, BFWS variants, \textbf{LAMA}, \textbf{SM-Type-LAMA}, \textbf{NOLAN}, as well as
\textbf{DecStar} \cite{gnad2018decstar} which won IPC2023 Agile.
\textbf{NOLAN} adds
a $\vf_2$ BFWS configuration by \citet{lipovetzky2017bfws} to LAMA's queue.
(We use $w_{\max}=2$; appendix \refsec{sec:nolan-dynamic-width}.)

The 5-minute agile result (\reftbl{tbl:bfws-table}, top) shows that
\coolname outperforms them in terms of \# solved,
and is on par with SM-Type-LAMA in the IPC score.
It outperformed ApxNovelty with the same FD-based grounder despite using a simpler novelty metric.
Reimplementing approximate novelty in Fast Downward and combine it with \coolname is an interesting avenue of future work.
Despite not an IPC participant, SM-Type-LAMA was also quite strong,
outperforming ApxNovelty without novelty metric.
Domain-wise results and solution cost comparisons are included in the appendix \refsec{sec:domain-wise}.

As \coolname combines 3 techniques, we perform an ablation study removing one of the 3 components.
In \reftbl{tbl:bfws-table},
\textbf{LAMAe+BFWS} replaces the best-first queues in LAMAe with tiebreaking queues $[w_{\ar{l}, \ff}, h]$ for $h\in \braces{\ar{l}, \ff}$.
\textbf{\green{N2}\orange{DTC}+LAMAe} replaces the best-first queues in LAMAe with $\function{n2}([h])$ for $h\in \braces{\ar{l}, \ff}$.
\textbf{\green{N2}\orange{DTC}+BFWS} is a single queue configuration using $\function{n2}([w_{\ff}, \ff])$.
All configurations use the eager search.
Appendix \reftbl{tbl:lamanormal2-collapsing} also contains other extensive evaluations,
including the result of varying the collapsing parameter $\theta$ in \coolname.

\paragraph{Extended 30 Minutes Run}

\begin{figure}[tb]
 \includegraphics[width=\linewidth]{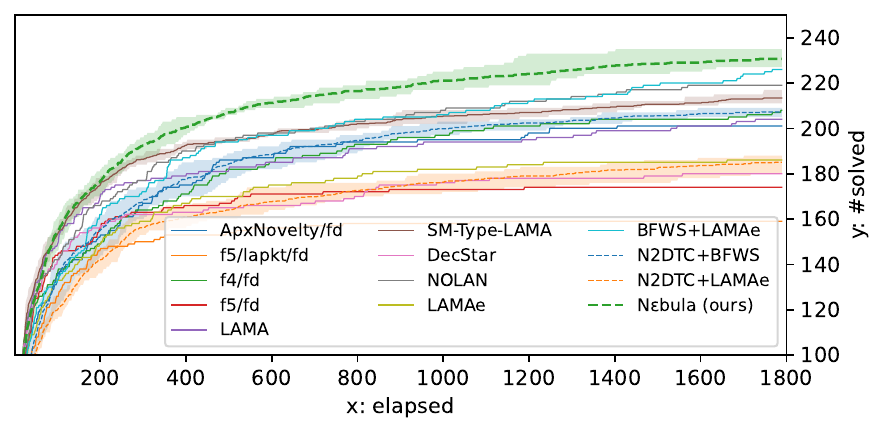}
 \caption{
 A histogram plot of the number of IPC18+IPC23 instances solved within 30 minutes.
 The lines indicate the average of 5 seeds, while the bands indicate the maximum and the minimum among the seeds.}
 \label{fig:histogram}
\end{figure}

Bilevel MCTS is still a computationally demanding algorithm in terms of the runtime and the memory
in comparison to simple queue-based search because of the tree management cost.
Its true potential is unlocked in the longer, extended 30 minutes runs.
\refig{fig:histogram} shows a coverage histogram plot which demonstrates this.
\reftbl{tbl:bfws-table} (bottom half) also shows the instances solved and the extended IPC score
$\sum_i \min \braces{1, 1-\nicefrac{\log t_i}{\log 1800}}$ which mimics the Agile IPC score.
\coolname significantly exceeds existing work (230.6 instances solved),
including once closely matched SM-Type-LAMA (213.4).
However, near the end of 30 minute time limit,
the space for trees caused memory exhaustion, which allowed queue-based LAMAe+BFWS to catch up (226).
Micro-optimization of the tree management could widen the gap.
See Appendix \refsec{sec:hists} for
the similar histogram plots using the memory usage and the number of node evaluations.

\section{Related Work}

MCTS variants that perform Bilevel search are common in the game literature \citep{kishimoto2012game},
such as PN$^2$ \citep{kishimoto2012game}.
However, the main motivation of existing work tends to be improving the memory usage,
much like linear-memory graph search algorithms (e.g. IDA* \citep{korf1985depth}, RBFS \citep{korf1993linear}),
and they discard the nodes explored by the second-level search to release the memory,
unlike our Bilevel MCTS.

Df-UCT \citep{yoshizoe2011scalable} similarly aims at improving the action selection performance,
but with the main focus on improving the communication overhead in distributed search.

\section{Conclusion}

Classical planning tasks have a significantly longer horizon compared to game tree search,
which makes the application of naive MCTS impractical.
We proposed Bilevel MCTS, a modification that reduces the action selection overhead in MCTS
by running a best-first search from the leaf nodes.
We also proposed tree-collapsing to reduce the overhead further.
Combining them with proven techniques in the literature (novelty, alternation with boosted preferred operator queue)
yielded a \lsota planner, \coolname, that outperforms existing work.

\fontsize{9.5pt}{10.5pt}
\selectfont

\clearpage
\appendix
\section*{Appendix}

\section{Proofs for \emph{selection} runtime complexity}
\label{theo:action-selection-complexity-bilevel-proofs}

\setcounter{theo}{1}
\begin{theo}
 \label{apxtheo:action-selection-complexity-bilevel-heap}
 Assume that the search space forms a tree with a constant branching factor $B$, and
 a tree-based OPEN list of depth $D$, containing $N = B^D$ leaves.
 In the Bilevel MCTS,
 the amortized runtime of each \emph{selection} is $O(\log \log N)$ if $Q$ is a heap.
\end{theo}

\begin{proof}
 In $D$ consecutive expansions,
 the first expansion performs the recursive action selection which takes $O(BD)$ runtime,
 then the following $D-1$ expansions use the priority queue $Q$.
 Due to the constant branching factor,
 at the time of $d$-th expansion from $Q$,
 it has $B+(B-1)(d-1)=O(Bd)$ nodes,
 having inserted $B$ nodes initially, then $B(d-1)$ nodes by $d-1$ expansions, then having $d-1$ nodes removed.
 If $Q$ is a heap,
 each $d$-th \function{popmin} from $Q$ therefore takes at most $\log Bd$ runtime.
 After $D-1$ expansions,
 this totals at most $\sum_{d=1}^{D-1} \log Bd = O(D\log BD)$.
 Thus, on average, \emph{selection} takes $\frac{BD+D\log BD}{D}= O(\log D) = O(\log \log N)$ per expansion.
 Note that $B$ is a constant.
\end{proof}

\begin{theo}
 \label{apxtheo:action-selection-complexity-bilevel-bucket}
 Assume that the search space forms a tree with a constant branching factor $B$, and
 a tree-based OPEN list of depth $D$, containing $N = B^D$ leaves.
 In the Bilevel MCTS,
 the amortized runtime of each \emph{selection} is $O(1)$ if $Q$ is an array-based priority queue.
\end{theo}

\begin{proof}
 The total runtime is $BD+D-1$, therefore each \emph{selection} takes $\frac{BD+D-1}{D}= O(1)$ per expansion.
\end{proof}

\section{Eval/Sec Improvements from\\ Bilevel MCTS}
\label{sec:evalsec-bilevel-more}

In the main paper, experiments using \guct-Normal2+$\ff$ showed that
\begin{enumerate}
 \item MCTS can have a significantly slower node evaluation speed than GBFS,
 \item This is due to the search depth, and
 \item Bilevel MCTS improves the node evaluation speed when the search is deep.
\end{enumerate}
In this section we demonstrate the results for other heuristics ($\cg, \ce$).

\refig{fig:evalsec-bilevel-more} shows the result comparing the evaluation speed (item 1).
In $\cg$, we observed that MCTS is even more severely affected by the bottleneck.
This is because the $\cg$ heuristic is relatively lightweight compared to $\ff$,
as seen by some points reaching $10^5$ evaluations per seconds, making the bottleneck more pronounced.
In contrast, the effect on a heavier and slower heuristic like $\ce$ is small,
as the runtime bottleneck shifts toward the cost of heuristic evaluation.

\refig{fig:evalsec-correlation-more} shows
the effect of bilevel MCTS on the correlation between search depth and node/sec (item 2, 3).
We observed a similar trend as seen in the previous figure.

\section{GUCTN2/Bilevel GUCTN2\\ Eager Evaluation Results, Domain-Wise Breakdown}
\label{sec:eager-domainwise}

The domain-wise breakdown of \reftbl{tbl:base-table} is available in \reftbl{tbl:base-table-domainwise}.

\section{GUCTN2/Bilevel GUCTN2\\ Lazy Evaluation Results}
\label{sec:lazy}

\reftbl{tbl:lazy-table} shows the results that correspond to \reftbl{tbl:base-table},
except that all heuristic evaluations are performed lazily, i.e., using the parent's heuristic value.
We added this evaluation because this configuration is a common request from the reviewers.
We however note that
the previous study on the effect of lazy evaluation \citep{RichterH2009}
already extensively showed that it often fails to improve the agile performance
due to the accuracy lost from this mechanism,
and they only work when combined with a boosted preferred operator (discussed later).

Our experiment confirms the previous study \citep{RichterH2009}.
Lazy evaluation tends to be ineffective, not only because
it makes the heuristics (and thus the action selection) less accurate,
but also because
the computational cost of action selections becomes more pronounced due to cheaper heuristic evaluations.
Note that although \citet{schulte2014balancing} reported a successful application of lazy evaluation to MCTS,
they used a different bandit (UCB1 with normalization) that has a significantly different search behavior.

\begin{table*}[tb]
\centering
 \begin{adjustbox}{max width=\linewidth}
\setlength{\tabcolsep}{0.3em}
\begin{tabular}{r|r|rr|rrrrrrr|rrrrrrr|r|rrrr}
\toprule
 &      & GBFS & Soft & base & \multicolumn{6}{c|}{GUCTN2} & base & \multicolumn{7}{c|}{\green{Bilevel GUCTN2}} & \multicolumn{4}{c}{\magenta{Fixed Budget Bilevel}}\\
 & $h=$ &      & min
 & \cyan{$\theta=0$} & \cyan{10} & \cyan{20} & \cyan{40} & \cyan{80} & \cyan{160} & \cyan{320}
 & \cyan{$\theta=0$} & \cyan{10} & \cyan{20} & \cyan{40} & \cyan{80} & \cyan{160} & \cyan{320} & \orange{DTC}
 & $b=10$ & 100 & 1000 & 10000
 \\
\midrule\multirow{4}{*}{\rotatebox{90}{\# solved}}
 & $\ce$ & 70.0  & \textbf{82.6}  & 73.4 & 77.6 & 78.2 & 75.6  & 74.6  & 77.0  & 78.6  & 72.8  & 75.0  & 73.4           & 73.6          & 77.0  & 79.6  & 75.2  & 77.4  & 68.8 & \textbf{82.4} & 81.4  & 76.8  \\
 & $\cg$ & 71.0  & \textbf{98.6}  & 63.2 & 69.2 & 70.6 & 71.0  & 68.0  & 65.0  & 61.6  & 87.2  & 87.4  & 97.2           & \textbf{97.4} & 92.4  & 89.6  & 88.0  & 84.4  & 73.2 & 84.0          & 88.2  & 81.4  \\
 & $\ff$ & 111.0 & \textbf{135.6} & 95.8 & 98.0 & 99.8 & 102.6 & 101.4 & 101.6 & 97.6  & 121.8 & 124.6 & \textbf{125.0} & 123.4         & 118.6 & 117.2 & 112.2 & 110.8 & 87.4 & 113.2         & 124.8 & 122.4 \\
\cmidrule{2-23}
 & total & 252.0 & \textbf{316.8} & 232.4 & 244.8 & 248.6 & 249.2 & 244.0 & 243.6 & 237.8 & 281.8 & 287.0 & \textbf{295.6} & 294.4 & 288.0 & 286.4 & 275.4 & 272.6 & 229.4 & 279.6 & 294.4 & 280.6 \\
\midrule\multirow{4}{*}{\rotatebox{90}{agl score}}
 & $\ce$ & 40.8 & \textbf{44.7} & 42.1 & 42.4 & 43.1 & 41.5 & 41.3 & 41.5 & 42.2  & 39.1 & 39.8 & 37.7 & 40.2          & 40.9 & 43.2 & 40.9 & 39.6 & 35.4 & \textbf{45.1} & 43.0          & 39.7 \\
 & $\cg$ & 42.0 & \textbf{52.1} & 33.6 & 36.0 & 36.7 & 37.0 & 34.9 & 32.6 & 32.1  & 47.0 & 46.8 & 50.3 & \textbf{51.0} & 47.0 & 46.3 & 43.5 & 44.3 & 39.6 & 46.3          & 44.4          & 41.5 \\
 & $\ff$ & 61.4 & \textbf{73.2} & 54.4 & 54.4 & 54.7 & 56.4 & 55.3 & 54.1 & 53.6  & 63.0 & 63.2 & 63.2 & 62.0          & 61.5 & 62.0 & 59.9 & 59.6 & 45.3 & 65.6          & \textbf{66.4} & 63.1 \\
\cmidrule{2-23}
 & total & 144.2 & \textbf{169.9} & 130.2 & 132.7 & 134.5 & 135.0 & 131.5 & 128.2 & 127.8 & 149.1 & 149.8 & 151.3 & 153.1 & 149.4 & 151.5 & 144.2 & 143.5 & 120.4 & \textbf{156.9} & 153.8 & 144.3 \\
 \bottomrule
\end{tabular}
 \end{adjustbox}
 \caption{
 Evaluating various MCTS enhancements with lazy evaluation.
 The performance tend to be inferior to the eager evaluation variants.
  }
 \label{tbl:lazy-table}
\end{table*}

\section{Eval/Sec Improvements from \cyan{Tree Collapsing}}
\label{sec:evalsec-collapsing}

Next, we show the scatter plots demonstrating the improvement due to \cyan{Tree Collapsing}.
\refig{fig:collapse-node-sec} shows the scatter plots of node evaluations per seconds
between GUCTN2 configurations with and without \cyan{Tree Collapsing}.
\cyan{Tree Collapsing} generally improve the performance, except in $\ce$ which is a heavier heuristic (i.e., it takes more time to compute).

\section{Brief Detail on Novelty Implementation}
\label{sec:novelty-implementation}

We briefly explain the implementation details of our novelty metric,
assuming the common background knowledge of SAS formalism and the internal of Fast Downward.
Our implementation is optimized to reduce the number of iterations, use compact bitvectors,
and is extended to handle axioms (derived variables).

Suppose a novelty metric defined for a list of heuristics $h_1, h_2,\ldots, h_k$,
such as $h_1=\ar{l}, h_2=\ff$ in $w_{\ar{l}, \ff}$.
For each tuple of heuristic values $[h_1(s),\ldots h_k(s)]$ of a state $s$ to evaluate,
it instantiate two bitvectors $v_1$ and $v_2$ of length $|P|$ and $|P|^2$, respectively,
where $|P|$ is the number of facts (propositions) in the domain.
Since Fast Downward is a SAS+-based planner, we associate $j$-th value of $i$-th SAS variable with $n$-th fact
as $n=I(i,j) = j+ \sum_{m=1}^{i-1} |d(m)|$, where $d(m)$ is the domain of $m$-th SAS variable.
Let us also denote its inverse as $i, j = I^{-1}(n)$.

The two vectors $v_1$ and $v_2$ represent the tables that store the set of 1-tuples and 2-tuples seen before.
These are represented as bitstrings, i.e., an array of 32bit unsigned ints where each bit represents a boolean.
This is significantly more efficient than simply using an array of bools, which causes each bool to consume one byte.

Rather than iterating over $P$ or $P\times P$ to mark all the 1-tuples and 2-tuples present in the state $s$,
we iterate over the conditional effects that were fired in the predecessor state $s'$,
and mark the tuples only for those that could be newly introduced as a result of the operator application.
Since the number of effects in each operator
tends to be significantly less than the size of the state $|P|$, this saves the number of iterations.
Moreover, since each tuple of facts represents a set,
we record 2-tuples of facts $\braces{P_i,P_j}$ only when $i < j$.

Finally, we handle derived variables.
Note that $P$ is extended to include derived variables,
but often a majority of them are not actually used by the operators.
As all derived variables are derived from existing normal propositions,
recording the new-ness of derived variables is fundamentally redundant.
To avoid recording such useless information,
we filter the derived variables to record for novelty computation
to only those that appear in any action precondition or any effect condition.
The set of such derived variables is precomputed during the initialization.

\section{Evaluating Dynamic Maximum Width in NOLAN, as well as the Second Phase of Dual-BFWS}
\label{sec:nolan-dynamic-width}
\label{sec:f6}

\reftbl{tbl:nolan-dynamic-width} shows the result of using NOLAN's dynamic $w_{\max}$ strategy
in $\vf_4$, $\vf_5$, and our reimplmentation of NOLAN \cite{correa2025nolan}.
We also present the result of $\vf_6 = [w_{\ar{l}, \ff}, \function{preferred}, \ar{l}, w_{\ar{l}}, \ff]$ here,
which is the second complete search phase of Dual-BFWS.
(We do not use the dual-phased search as employed in the original FF planner's EHC or Dual-BFWS.)
In fast downward,
$\function{preferred}$ is an evaluator that takes the value of $1$ if the node is generated by a preferred operator,
and $0$ otherwise. This is mentioned as $\function{help}$ by \citet{lipovetzky2017bfws}.

In NOLAN \cite{correa2025nolan}, it is unclear whether $V$ is measured by the number of SAS variables
or the number of propositional variables.
\reftbl{tbl:nolan-dynamic-width} is produced by a version that counts the SAS variables, but
the results were the same when we counted the propositional variables.

Note that this does not diminish the usefulness of the dynamic width approach presented in NOLAN;
Tuning the hyperparameter $V$ may help the performance in IPC18+23.
However, this topic is outside the main focus of our paper.

\section{\coolname with Various \cyan{Tree-Collapsing} Parameters}
\label{sec:lamanormal2-collapsing}

\reftbl{tbl:lamanormal2-collapsing} shows the domain-wise, IPC-wise and the total \# solved and Agile IPC scores.
It demonstrate that \cyan{Tree Collapsing} still has a significant effect in this complex configuration.
Of a particular note is the effect of \orange{Dynamic Tree Collapsing}, which significantly improved the Agile IPC score.

\section{\sota Planners \\ Domain-wise Results}
\label{sec:domain-wise}

\reftbls{tbl:domain-wise}{tbl:domain-wise-extended}
shows the full domain-wise result of \reftbl{tbl:bfws-table}
for both the 300 seconds agile runs and the 1800 seconds extended runs.

\section{Solution Cost Comparison through Satisficing IPC Scores}
\label{sec:cost-results}

For each configuration and seed,
we computed the satisficing IPC score $\sum_{i} \nicefrac{C^*_i}{C_{i}}$, where,
for instance $i$,
$C^*_i$ is the best solution length found among all algorithms and all 5 seeds,
and $C_{i}$ is the solution length of the configuration and the seed.
When the solution is not found, $C_{i}=\infty$ and $\nicefrac{C^*_i}{C_{i}}=0$.
When none of the configurations and the seeds found a plan,
all configurations get the score of 0 on that instance.

\reftbls{tbl:ipc-sat-300}{tbl:ipc-sat-1800} shows the results for 300-second and 1800-second runs.
\coolname's bandit approach does not optimize the solution length
but rather optimize the theoretical regret bound (i.e., the number of evaluations),
thus fairs worse than the other existing planners
that are simultaneously flirting with the idea of solution cost optimization.
However, because it solved more instances than the other approaches,
its satisficing score is not too bad.

\section{Histogram Coverage Plots based on Elapsed Time, Evaluations, Memory RSS}
\label{sec:hists}

\refig{fig:histograms} shows the enlarged histogram plots
for the elapsed time, the number of evaluations and the memory RSS.

\section{The Effect of Reopening and Tree-grafting in Agile Search}
\label{sec:tree-grafting}

\reftbl{tbl:tree-grafting} shows the 300 seconds agile runs for
GUCTN2 with node reopening and tree grafting features enabled.
All runs are based on eager heuristic evaluations.
Node reopening and grafting had a slight performance penalty.
Note that grafting implies reopening (grafting happens only when the node is reopened),
thus there is no need to run the configuration with reopening disabled and grafting enabled.

\begin{table}[tb]
\centering
\begin{tabular}{rr|rrr}
\toprule
 & Reopening & disabled & enabled & enabled \\
 & Tree Grafting & disabled & disabled & enabled \\
\midrule\multirow{4}{*}{\rotatebox{90}{\# solved}} & $\ce$ & \textbf{76.0} & 75.4 & 75.4 \\
 & $\cg$ & \textbf{67.4} & 64.8 & 65.2 \\
 & $\ff$ & 103.6 & 102.4 & \textbf{103.8} \\
\addlinespace & total & \textbf{247.0} & 242.6 & 244.4 \\
\midrule\multirow{4}{*}{\rotatebox{90}{agl score}} & $\ce$ & 45.3 & 45.0 & \textbf{45.5} \\
 & $\cg$ & \textbf{36.3} & 35.0 & 35.2 \\
 & $\ff$ & 59.4 & \textbf{59.9} & 59.8 \\
\addlinespace & total & \textbf{141.0} & 139.9 & 140.5 \\
\bottomrule
\end{tabular}
\caption{
The result comparing the number of solved instances and the agile scores across IPC18 and IPC23 instances.
The effect of enabling reopening and tree grafting is almost negligible but slightly negative.
}
\label{tbl:tree-grafting}
\end{table}

\clearpage

\begin{table*}[htbp]
 \centering
 \begin{adjustbox}{max height=0.5\textheight}
\setlength{\tabcolsep}{0.3em}
\begin{tabular}{r|r|rr|rrrrrrr|rrrrrrr|r|rrrr}
\toprule
 & domain & GBFS & Soft & base & \multicolumn{6}{c|}{GUCTN2} & base & \multicolumn{7}{c|}{\green{Bilevel GUCTN2}} & \multicolumn{4}{c}{\magenta{Fixed Budget Bilevel}}\\
 & $h=$ &      & min
 & \cyan{$\theta=0$} & \cyan{10} & \cyan{20} & \cyan{40} & \cyan{80} & \cyan{160} & \cyan{320}
 & \cyan{$\theta=0$} & \cyan{10} & \cyan{20} & \cyan{40} & \cyan{80} & \cyan{160} & \cyan{320} & \orange{DTC}
 & $b=10$ & 100 & 1000 & 10000 \\
\midrule\multirow{57}{*}{\rotatebox{90}{\# solved}} & agricola & 4 & 3.6 & 3.6 & 4.2 & 4 & 3.6 & 3.4 & 3.6 & 2.8 & 8.2 & \textbf{9.6} & \textbf{9.6} & 9.2 & 8.8 & 9.4 & 8.8 & 9 & \textbf{9.8} & 5.8 & 4.8 & 5.8 \\
 & caldera & 1 & 2 & \textbf{3.4} & 3.2 & \textbf{3.4} & \textbf{3.4} & \textbf{3.4} & \textbf{3.4} & 3 & 3 & 3.2 & \textbf{3.6} & 3 & 2.6 & 3 & 3 & 3 & 3.2 & 3.2 & \textbf{3.4} & 3 \\
 & data-net & 3 & \textbf{7.2} & 6.8 & \textbf{7.2} & 7 & 6.8 & 6.8 & 6.4 & 6.6 & 3.8 & 3.8 & 3.6 & 3.6 & 3.8 & 3.4 & 4.4 & 3.8 & 3 & 4.2 & 4.6 & 3 \\
 & flashfill & 8 & 8 & 9 & 10 & 9.8 & 9.8 & 9.2 & 10 & 8.4 & \textbf{10.4} & 10 & \textbf{10.4} & \textbf{10.6} & \textbf{10.4} & \textbf{10.4} & 9.4 & \textbf{10.4} & 9.2 & 9.8 & 8.2 & 8 \\
 & nurikabe & 10 & 9.6 & 9 & 9.2 & 7.8 & 9.2 & 9.4 & 10.4 & \textbf{10.6} & 7.2 & 7 & 7.6 & 7.2 & 7.4 & 9 & \textbf{10.8} & 7.2 & 8.6 & 7 & 8 & 9.8 \\
 & org-syn & 7 & 5.8 & 6 & 6 & 6 & 5.8 & 7.2 & 6.4 & 6.8 & 7 & 7 & 7.2 & 7 & 7 & \textbf{7.4} & 6.8 & 7.2 & \textbf{7.6} & 6 & 7.2 & 7 \\
 & settlers & 4 & \textbf{7} & 3.8 & 3.6 & 4.6 & 3.6 & 4.6 & 4.8 & 6.2 & 6.2 & 6.4 & 6.6 & 5.8 & 5 & 3.2 & 3.4 & \textbf{7} & 2 & 5.6 & \textbf{7} & 5 \\
 & snake & 0 & 0 & 5.8 & \textbf{7} & 5.8 & 4.8 & 4.6 & 4.6 & 3.4 & 4 & 5.8 & 6.4 & 6.6 & 6.8 & 6 & 5.6 & \textbf{7} & \textbf{8} & 4.2 & 1.8 & 0.8 \\
 & spider & 1 & 1.6 & 1.2 & 1.8 & 1.8 & \textbf{2} & 1.4 & \textbf{2} & 1.2 & 0 & 0 & 1.8 & \textbf{2} & 1.8 & 1.4 & 1.4 & 1.4 & 1 & 1 & 0 & 1 \\
 & termes & 2 & 3.6 & 1.6 & 1.2 & 1.2 & 1 & 1.2 & 1 & 1 & 4 & 3.6 & 4.6 & 3.6 & \textbf{5.2} & 4.6 & 4.6 & 5 & \textbf{5.2} & 5 & 4 & 3.2 \\
 & \textit{ipc18} & 40 & 48.4 & 50.2 & 53.4 & 51.4 & 50 & 51.2 & 52.6 & 50 & 53.8 & 56.4 & \textbf{61.4} & 58.6 & 58.8 & 57.8 & 58.2 & \textbf{61} & 57.6 & 51.8 & 49 & 46.6 \\
 & folding & 0 & 0 & 0 & 0 & 0 & 0 & 0 & 0 & 0 & 0 & 0 & 0.6 & \textbf{1} & 0 & 0.4 & 0 & \textbf{0.8} & 0 & 0 & 0 & 0 \\
 & quantum & 16 & 16.2 & 16 & 16 & 16 & 16 & 14.6 & 14.6 & 14.4 & 16.4 & 16.4 & 16.4 & 16.4 & 16.2 & 16 & 16 & 16.4 & \textbf{17} & \textbf{17} & \textbf{17} & \textbf{17} \\
 & recharging & \textbf{9} & 8.2 & 5.2 & 5.4 & 6.4 & 5.8 & 6.2 & 6.6 & 6.6 & 6.4 & 6.6 & 6.4 & 7 & 5.8 & 7.8 & 7.6 & 6.4 & 8 & \textbf{9} & \textbf{9} & \textbf{9} \\
 & ricochet & 1 & 4.8 & 0.8 & 1.4 & 0.8 & 1.6 & 0.8 & 1.4 & 1.6 & 3 & 4 & 4 & 5 & 5.2 & 5 & \textbf{6.6} & 4.8 & 3.2 & \textbf{5.8} & 3.6 & 3.4 \\
 & rubiks & 4 & 5 & 3.8 & 3.8 & 3.8 & 3.8 & 3.8 & 3.8 & 3.8 & 5 & 5 & 4.8 & 4.8 & \textbf{5.8} & \textbf{5.2} & 4.6 & 5 & 4.2 & 5 & 4.2 & 4.2 \\
 & \textit{ipc23} & 30 & 34.2 & 25.8 & 26.6 & 27 & 27.2 & 25.4 & 26.4 & 26.4 & 30.8 & 32 & 32.2 & 34.2 & 33 & 34.4 & \textbf{34.8} & 33.4 & 32.4 & \textbf{36.8} & 33.8 & 33.6 \\
 & \textit{$\ce$ total} & 70 & 82.6 & 76 & 80 & 78.4 & 77.2 & 76.6 & 79 & 76.4 & 84.6 & 88.4 & \textbf{93.6} & 92.8 & 91.8 & 92.2 & 93 & \textbf{94.4} & 90 & 88.6 & 82.8 & 80.2 \\
\midrule & agricola & 0 & \textbf{12} & 9.2 & 10 & 10 & 10 & 10 & 9 & 5.2 & 9.4 & 10 & 10 & 10 & \textbf{10.4} & 8.4 & 8.4 & 8.6 & 8.8 & 10 & 7.2 & 7.2 \\
 & caldera & 4 & \textbf{3.6} & 2 & 2 & 2 & 2 & 2 & 2 & 2 & 2 & 2 & 2 & 2 & 2 & 2 & 2 & 2.2 & 2 & 2 & 2 & 3.4 \\
 & data-net & 1 & \textbf{6} & 2 & 2 & 2.4 & \textbf{3} & 2.8 & 2.4 & 2.4 & 2.6 & 2.6 & 2.4 & 2.8 & 2 & 1.2 & 1.6 & 2.4 & 1.4 & 2.6 & \textbf{3} & 2.2 \\
 & flashfill & 0 & \textbf{3.8} & 0 & 0 & 0 & 0 & 0 & 0 & 0 & 1.2 & 1.2 & 1.2 & 1.2 & 1.2 & 1.2 & 1.2 & 1.2 & 0.6 & 0.2 & 1 & \textbf{2} \\
 & nurikabe & 7 & 8 & 9.8 & 8.6 & 9.2 & 9.6 & \textbf{10.6} & 9.8 & 9.8 & 7.8 & 9.6 & 9.8 & 9.4 & \textbf{10} & 9.8 & \textbf{10} & 8.8 & 8.2 & 7.4 & 7 & 7 \\
 & org-syn & \textbf{4} & \textbf{4} & \textbf{4} & \textbf{4} & \textbf{4} & \textbf{4} & \textbf{4} & \textbf{4} & \textbf{4} & \textbf{4} & \textbf{4} & \textbf{4} & \textbf{4} & \textbf{4} & \textbf{4} & \textbf{4} & \textbf{4.2} & \textbf{4} & \textbf{4} & \textbf{4} & \textbf{4} \\
 & settlers & 5 & \textbf{9} & 2 & 2 & 2.4 & 2 & 3 & 1.4 & 2.6 & 6.6 & 5.6 & 6.6 & 5.6 & 2.8 & 3.2 & 2.6 & 5 & 2.8 & \textbf{6.8} & 5.4 & 6.6 \\
 & snake & 3 & 2.8 & 8.8 & 9.8 & 9.4 & 8.6 & 7.8 & 7.2 & 6.4 & 9.8 & 11.2 & 11.6 & 13.6 & 14.6 & 14.2 & \textbf{15} & 13 & \textbf{16.4} & 10.8 & 6 & 3.8 \\
 & spider & \textbf{5} & 4.8 & 3.6 & 3.8 & 4.2 & 4.2 & 4.4 & 4.6 & 4.4 & 4.6 & 4.2 & 4.2 & 4 & 4.4 & 4.6 & 4.6 & 4 & 4 & \textbf{5} & 4.6 & \textbf{5} \\
 & termes & 6 & 5.2 & 1 & 1 & 1 & 1 & 1 & 1 & 1 & 8.2 & 8 & 7.8 & 7.6 & \textbf{8.4} & 7.6 & 7.8 & 8.2 & 3.4 & 7 & \textbf{8.4} & 7.6 \\
 & \textit{ipc18} & 35 & 59.2 & 42.4 & 43.2 & 44.6 & 44.4 & 45.6 & 41.4 & 37.8 & 56.2 & 58.4 & 59.6 & \textbf{60.2} & \textbf{59.8} & 56.2 & 57.2 & 57.6 & 51.6 & 55.8 & 48.6 & 48.8 \\
 & folding & 3 & 3.2 & 2.8 & 2.8 & 2.8 & 3.2 & 3.8 & 2.8 & 3.4 & 8.8 & 9 & 9 & \textbf{10.6} & 9.6 & 8.6 & 8.4 & \textbf{9.8} & 5 & 8.8 & 9 & 7.6 \\
 & labyrinth & 0 & 0 & \textbf{0.2} & 0 & \textbf{0.2} & 0 & 0 & 0 & 0 & \textbf{0.2} & \textbf{0.2} & \textbf{0.2} & 0 & 0 & 0 & \textbf{0.2} & \textbf{0.2} & \textbf{0.2} & 0 & 0 & 0 \\
 & quantum & \textbf{15} & 13.4 & 8.2 & 8 & 9.8 & 9 & 9.2 & 9.4 & 10.2 & 13.6 & 12.4 & 12.8 & 13.4 & 12.6 & 11.2 & 11 & 13.4 & 13 & 14 & \textbf{14.8} & 14.2 \\
 & recharging & \textbf{12} & \textbf{13} & 10 & 10.4 & 10.2 & 10.2 & 11.4 & 10.6 & 10 & 11.4 & 11.6 & 11.2 & 11.6 & 11.4 & 11.2 & 10.6 & 11.2 & 11.6 & 11.8 & 11.8 & 11.4 \\
 & ricochet & 2 & 5 & 0.8 & 1.2 & 1.2 & 1.4 & 0.6 & 1 & 1 & 3.4 & 4.2 & 5.8 & 6.2 & 6.4 & \textbf{7.2} & \textbf{6.8} & 6 & 1.2 & 4.4 & 5.2 & 4.2 \\
 & rubiks & 4 & 4.8 & 3 & 3 & 3 & 3 & 3 & 3 & 3 & 4.2 & 4.2 & 3 & 4 & \textbf{5.4} & 5.2 & 5 & \textbf{5.8} & 3.2 & 3.2 & 3.2 & 4.8 \\
 & \textit{ipc23} & 36 & 39.4 & 25 & 25.4 & 27.2 & 26.8 & 28 & 26.8 & 27.6 & 41.6 & 41.6 & 42 & \textbf{45.8} & 45.4 & 43.4 & 42 & \textbf{46.4} & 34.2 & 42.2 & 44 & 42.2 \\
 & \textit{$\cg$ total} & 71 & 98.6 & 67.4 & 68.6 & 71.8 & 71.2 & 73.6 & 68.2 & 65.4 & 97.8 & 100 & 101.6 & \textbf{106} & \textbf{105.2} & 99.6 & 99.2 & 104 & 85.8 & 98 & 92.6 & 91 \\
\midrule & agricola & 4 & 9.6 & 10 & 9.8 & 10 & 10 & 10 & 9.8 & 9.6 & 12 & 11.4 & \textbf{12.6} & \textbf{12.4} & 11.6 & 11.4 & 10 & 11.6 & 11.8 & 9 & 9.6 & 7 \\
 & caldera & 4 & 6.4 & 6.4 & 6.6 & 6.6 & 7 & 6.8 & 6.8 & 6.6 & 6.8 & 6.8 & 6.2 & 6 & 6.6 & 6.4 & 6.6 & 7 & 6.2 & \textbf{7.2} & \textbf{7.2} & 6.2 \\
 & data-net & 4 & 8 & \textbf{8.6} & 8.4 & 8.4 & 7.8 & 8.4 & 7.6 & \textbf{8.8} & 4.8 & 5.4 & 4.2 & 5.2 & 4.4 & 4.6 & 3.6 & 4.8 & 3 & 5.4 & 6.4 & 6.4 \\
 & flashfill & 9 & 9.2 & 7.2 & 7.6 & 7.4 & 7.6 & 7.4 & 7.6 & 7.8 & 7.4 & 7.6 & 7.6 & 8 & 7.6 & 7.6 & 8.2 & 7.8 & \textbf{11} & \textbf{10.8} & 8.4 & 9.8 \\
 & nurikabe & 7 & 7 & 7.8 & 6.2 & 7.4 & 6.8 & 6.6 & 7 & 6.8 & \textbf{8} & \textbf{8} & 7.8 & 7.6 & 7.4 & 7.6 & 7.4 & \textbf{8} & 6.2 & 7 & 7 & 7 \\
 & org-syn & 9 & 9.2 & \textbf{10} & \textbf{10} & 9 & 9 & 9 & 9.4 & 9 & 9.6 & 8 & 7 & 8 & 8.2 & 8 & 8 & 9.6 & 8 & 8 & 9 & 9 \\
 & settlers & 0 & 5.6 & 3.6 & 3.4 & 2.6 & 3.6 & 3 & 4.4 & 3.8 & 4.2 & 4.4 & 5.8 & 5 & \textbf{6} & \textbf{7} & 5.2 & 3 & 2 & 2.2 & 1.8 & 0.2 \\
 & snake & 5 & 5 & 14.2 & 16 & 16 & 16.6 & 14.6 & 15 & 14 & 12.6 & 14.4 & 15.2 & 16.2 & \textbf{17} & 16.6 & 15 & \textbf{16.8} & 16.2 & 12.8 & 11.6 & 8.6 \\
 & spider & 6 & 8.8 & 8.2 & 9.4 & 8 & \textbf{9.6} & 7.8 & 9 & 9 & 8.8 & 8.2 & 9.2 & 8.6 & 8.8 & 8.8 & 9.2 & 8 & \textbf{10.6} & 8.2 & 8 & 7 \\
 & termes & 12 & 11.4 & 3.8 & 4.6 & 4.4 & 4.8 & 4.2 & 4.8 & 4.2 & 13.2 & 13 & 13 & 13.4 & 13 & 13.4 & \textbf{13.8} & \textbf{13.6} & 10.8 & 12.8 & 13.2 & 11.6 \\
 & \textit{ipc18} & 60 & 80.2 & 79.8 & 82 & 79.8 & 82.8 & 77.8 & 81.4 & 79.6 & 87.4 & 87.2 & 88.6 & 90.4 & \textbf{90.6} & \textbf{91.4} & 87 & 90.2 & 85.8 & 83.4 & 82.2 & 72.8 \\
 & folding & 1 & 1 & 0 & 0.4 & 0.4 & 0.8 & 1 & 0.8 & 0.8 & 5 & 5.2 & \textbf{5.8} & \textbf{5.8} & 4.6 & 3.4 & 3.2 & 5.6 & 3.8 & 4.8 & 4.4 & 1.6 \\
 & labyrinth & 0 & 0 & \textbf{0.2} & 0 & 0 & \textbf{0.2} & \textbf{0.2} & 0 & \textbf{0.2} & \textbf{0.6} & \textbf{0.2} & \textbf{0.2} & 0 & \textbf{0.2} & \textbf{0.2} & \textbf{0.2} & \textbf{0.2} & 0 & 0 & 0 & 0 \\
 & quantum & 17 & \textbf{18} & 11.4 & 11.2 & 14.4 & 13 & 12 & 12.4 & 13 & 17.6 & 17.4 & 17.2 & 16.8 & 16.2 & 15.2 & 15.6 & 17.4 & 16.4 & 17.4 & \textbf{18} & \textbf{18.8} \\
 & recharging & \textbf{9} & \textbf{9.6} & 5.6 & 5.6 & 5 & 5.6 & 5 & 4.6 & 5.2 & 5.8 & 5.4 & 5 & 5.8 & 5.8 & 6.4 & 6.6 & 6.2 & 6.6 & 7 & 8.2 & 8.6 \\
 & ricochet & 4 & 6.8 & 1.4 & 2.2 & 2.4 & 2.6 & 2.6 & 2 & 1.4 & 9.4 & 8.6 & \textbf{9.6} & \textbf{9.6} & 9.4 & \textbf{9.8} & 9.4 & 9 & 2.6 & 7.6 & 9.4 & \textbf{9.6} \\
 & rubiks & \textbf{20} & \textbf{20} & 5.2 & 5.2 & 5 & 5 & 5 & 5 & 5 & 19.8 & \textbf{20} & \textbf{20} & \textbf{20} & \textbf{20} & \textbf{20} & \textbf{20} & \textbf{20} & 16 & \textbf{20} & \textbf{20} & \textbf{20} \\
 & \textit{ipc23} & 51 & 55.4 & 23.8 & 24.6 & 27.2 & 27.2 & 25.8 & 24.8 & 25.6 & 58.2 & 56.8 & 57.8 & 58 & 56.2 & 55 & 55 & 58.4 & 45.4 & 56.8 & \textbf{60} & \textbf{58.6} \\
 & \textit{$\ff$ total} & 111 & 135.6 & 103.6 & 106.6 & 107 & 110 & 103.6 & 106.2 & 105.2 & 145.6 & 144 & 146.4 & \textbf{148.4} & 146.8 & 146.4 & 142 & \textbf{148.6} & 131.2 & 140.2 & 142.2 & 131.4 \\
\midrule & \textit{$\ce+\cg+\ff$} & 252 & 316.8 & 247 & 255.2 & 257.2 & 258.4 & 253.8 & 253.4 & 247 & 328 & 332.4 & 341.6 & \textbf{347.2} & 343.8 & 338.2 & 334.2 & \textbf{347} & 307 & 326.8 & 317.6 & 302.6 \\
\midrule\multirow{57}{*}{\rotatebox{90}{Agile IPC Score}} & agricola & 0.4 & 0.5 & 0.4 & 0.4 & 0.6 & 0.5 & 0.5 & 0.6 & 0.4 & 1.2 & \textbf{1.5} & 1.4 & 1.4 & 1.5 & 1.4 & 1.0 & 1.4 & \textbf{1.6} & 0.8 & 0.6 & 0.8 \\
 & caldera & 0.4 & 1.0 & 1.8 & 1.7 & 1.9 & 1.7 & 1.8 & \textbf{2.3} & 1.7 & 1.7 & 2.0 & 1.8 & 1.5 & 1.8 & 1.6 & 1.8 & 1.6 & 1.6 & \textbf{2.1} & 1.8 & 1.1 \\
 & data-net & 2.2 & 3.6 & 4.4 & 4.0 & 4.1 & 4.4 & \textbf{4.4} & 4.1 & \textbf{4.7} & 2.4 & 2.6 & 2.4 & 2.6 & 2.3 & 2.2 & 2.9 & 2.6 & 2.0 & 3.2 & 2.9 & 2.1 \\
 & flashfill & 4.8 & 4.3 & 5.1 & 5.5 & 5.5 & 5.6 & 5.4 & 5.3 & 4.7 & \textbf{5.6} & 5.3 & 5.5 & 5.3 & 5.3 & \textbf{5.7} & 5.1 & 5.4 & 5.1 & 4.8 & 4.6 & 4.7 \\
 & nurikabe & \textbf{6.4} & 6.1 & 5.9 & 5.9 & 5.5 & 6.0 & 6.0 & 6.2 & \textbf{6.4} & 4.9 & 5.1 & 5.4 & 4.6 & 5.2 & 6.1 & 6.2 & 5.3 & 5.4 & 4.7 & 5.9 & 6.4 \\
 & org-syn & 1.6 & 1.7 & 1.7 & 1.6 & 1.8 & 1.7 & 1.9 & 1.9 & 1.7 & \textbf{2.3} & 2.0 & 1.8 & \textbf{2.3} & 1.9 & 1.8 & 1.9 & 2.2 & 2.2 & 1.6 & 1.9 & 1.7 \\
 & settlers & 1.4 & \textbf{2.5} & 1.8 & 1.9 & 2.2 & 1.7 & 2.0 & 1.7 & 2.5 & 2.2 & 1.9 & \textbf{2.5} & 2.2 & 2.1 & 1.4 & 1.4 & 2.4 & 0.7 & 2.2 & 1.9 & 1.3 \\
 & snake & 0.0 & 0.0 & 1.7 & 1.9 & 2.1 & 1.1 & 0.9 & 1.0 & 0.7 & 0.9 & 1.8 & \textbf{2.3} & 2.1 & 1.9 & 1.5 & 1.1 & 1.8 & \textbf{2.4} & 0.9 & 0.3 & 0.0 \\
 & spider & 0.0 & 0.2 & 0.4 & 0.1 & 0.4 & 0.3 & 0.2 & 0.5 & 0.2 & 0.0 & 0.0 & 0.5 & \textbf{0.7} & \textbf{0.7} & 0.3 & 0.4 & 0.5 & 0.5 & 0.3 & 0.0 & 0.2 \\
 & termes & 0.5 & 0.7 & 0.7 & 0.5 & 0.6 & 0.7 & 0.5 & 0.5 & 0.4 & 1.1 & 1.2 & 1.4 & 0.8 & 1.4 & 1.4 & \textbf{1.6} & 1.2 & \textbf{1.7} & 1.5 & 0.8 & 0.5 \\
 & \textit{ipc18} & 17.7 & 20.6 & 23.8 & 23.4 & \textbf{24.5} & 23.8 & 23.7 & 24.1 & 23.4 & 22.4 & 23.4 & \textbf{25.2} & 23.6 & 23.9 & 23.3 & 23.4 & 24.4 & 23.1 & 22.1 & 20.8 & 18.8 \\
 & folding & 0.0 & 0.0 & 0.0 & 0.0 & 0.0 & 0.0 & 0.0 & 0.0 & 0.0 & 0.0 & 0.0 & \textbf{0.2} & 0.1 & 0.0 & 0.0 & 0.0 & \textbf{0.2} & 0.0 & 0.0 & 0.0 & 0.0 \\
 & quantum & 14.5 & 14.8 & 14.8 & 14.9 & 14.9 & 15.0 & 13.6 & 13.7 & 13.4 & 15.3 & 15.4 & 15.3 & 15.3 & 14.9 & 15.1 & 15.1 & 15.1 & \textbf{15.9} & \textbf{15.7} & 15.3 & 14.7 \\
 & recharging & \textbf{4.3} & 3.5 & 2.4 & 2.4 & 3.0 & 2.7 & 3.0 & 3.2 & 3.2 & 3.4 & 3.3 & 3.5 & 3.6 & 3.1 & 3.8 & 3.7 & 3.3 & 3.9 & 4.3 & 4.2 & \textbf{4.3} \\
 & ricochet & 0.3 & 1.3 & 0.4 & 0.4 & 0.6 & 0.8 & 0.4 & 0.6 & 0.4 & 1.1 & 1.5 & 1.3 & \textbf{2.4} & 1.8 & \textbf{2.2} & 2.2 & 1.6 & 1.3 & 2.0 & 1.2 & 0.7 \\
 & rubiks & 4.0 & \textbf{4.5} & 3.8 & 3.8 & 3.8 & 3.8 & 3.8 & 3.8 & 3.8 & 3.9 & 4.1 & 4.1 & 3.9 & 4.2 & 4.1 & 4.1 & 4.0 & 4.0 & \textbf{4.4} & 4.0 & 4.0 \\
 & \textit{ipc23} & 23.1 & 24.1 & 21.5 & 21.5 & 22.3 & 22.3 & 20.9 & 21.3 & 20.8 & 23.8 & 24.3 & 24.3 & \textbf{25.3} & 23.9 & 25.2 & 25.0 & 24.2 & 25.0 & \textbf{26.3} & 24.8 & 23.7 \\
 & \textit{$\ce$ total} & 40.8 & 44.7 & 45.3 & 44.9 & 46.9 & 46.1 & 44.5 & 45.3 & 44.2 & 46.2 & 47.7 & \textbf{49.5} & \textbf{48.9} & 47.8 & 48.6 & 48.4 & 48.6 & 48.1 & 48.4 & 45.6 & 42.5 \\
\midrule & agricola & 0.0 & \textbf{2.9} & 1.9 & 2.3 & \textbf{2.5} & 2.3 & 2.2 & 1.9 & 1.1 & 1.8 & 1.9 & 2.2 & 2.3 & 2.1 & 1.8 & 1.9 & 1.6 & 1.7 & 1.8 & 1.2 & 1.4 \\
 & caldera & 2.5 & \textbf{2.7} & 1.8 & 1.8 & 1.9 & 1.9 & 1.7 & 1.8 & 1.7 & 1.7 & 2.0 & 1.8 & 1.7 & 1.8 & 1.8 & 1.7 & 1.7 & 1.9 & 1.8 & 1.8 & \textbf{2.6} \\
 & data-net & 0.5 & \textbf{3.5} & 1.1 & 0.9 & 0.8 & 1.5 & 1.3 & 1.1 & 1.1 & 1.2 & 1.5 & 1.1 & 1.5 & 0.8 & 0.7 & 0.9 & 1.3 & 0.5 & 1.2 & \textbf{1.9} & 1.1 \\
 & flashfill & 0.0 & \textbf{1.3} & 0.0 & 0.0 & 0.0 & 0.0 & 0.0 & 0.0 & 0.0 & 0.5 & 0.5 & 0.5 & 0.6 & 0.5 & 0.5 & 0.5 & 0.4 & 0.2 & 0.1 & 0.3 & \textbf{1.0} \\
 & nurikabe & 5.3 & 6.2 & 6.9 & 6.1 & 6.8 & 7.0 & \textbf{7.5} & 7.3 & \textbf{7.6} & 6.2 & 6.9 & 7.1 & 6.7 & 6.9 & 6.8 & 6.8 & 6.7 & 6.3 & 5.6 & 5.5 & 5.5 \\
 & org-syn & 1.5 & 1.6 & 1.7 & 1.6 & 1.6 & 1.8 & 1.6 & \textbf{1.9} & 1.7 & 1.6 & 1.4 & 1.5 & 1.7 & 1.6 & \textbf{1.8} & 1.6 & 1.4 & 1.5 & 1.5 & 1.6 & 1.5 \\
 & settlers & 2.2 & \textbf{3.9} & 1.1 & 1.0 & 1.3 & 1.1 & 1.5 & 0.8 & 1.4 & 2.3 & 1.9 & 2.4 & 2.3 & 1.0 & 1.1 & 1.3 & 1.8 & 1.0 & \textbf{2.8} & 2.2 & 2.4 \\
 & snake & 1.5 & 0.9 & 3.9 & 4.2 & 4.1 & 3.9 & 3.7 & 3.3 & 2.7 & 4.2 & 5.5 & 6.0 & 6.9 & \textbf{6.9} & 6.7 & 6.4 & 6.5 & \textbf{7.3} & 4.4 & 2.2 & 1.4 \\
 & spider & \textbf{2.1} & 1.1 & 0.9 & 1.0 & 1.3 & 1.3 & 1.3 & 1.5 & 1.3 & \textbf{2.0} & 1.5 & 1.1 & 1.6 & 1.6 & 1.7 & 1.6 & 1.9 & 1.6 & 1.8 & 1.7 & 1.7 \\
 & termes & 2.9 & 3.2 & 0.7 & 0.7 & 0.7 & 0.8 & 0.7 & 0.8 & 1.0 & 4.4 & 4.5 & \textbf{4.9} & 4.7 & \textbf{5.5} & 4.9 & 4.2 & 4.6 & 2.0 & 4.0 & 4.7 & 3.2 \\
 & \textit{ipc18} & 18.4 & 27.2 & 19.9 & 19.7 & 21.1 & 21.6 & 21.5 & 20.4 & 19.6 & 26.1 & 27.6 & 28.6 & \textbf{30.0} & \textbf{28.7} & 27.7 & 26.8 & 28.0 & 23.9 & 25.1 & 23.3 & 21.8 \\
 & folding & 1.2 & 1.6 & 0.9 & 1.0 & 1.0 & 1.2 & 1.2 & 0.8 & 1.2 & \textbf{3.4} & 3.1 & 3.1 & 3.3 & 3.2 & 3.1 & 2.8 & 3.3 & 2.2 & 3.3 & \textbf{3.4} & 2.8 \\
 & labyrinth & 0.0 & 0.0 & 0.0 & 0.0 & 0.0 & 0.0 & 0.0 & 0.0 & 0.0 & \textbf{0.0} & 0.0 & 0.0 & 0.0 & 0.0 & 0.0 & \textbf{0.0} & 0.0 & 0.0 & 0.0 & 0.0 & 0.0 \\
 & quantum & \textbf{11.9} & 11.2 & 7.2 & 7.2 & 9.0 & 8.0 & 8.1 & 8.2 & 9.6 & 11.5 & 10.1 & 10.4 & 11.2 & 10.3 & 9.7 & 9.4 & 10.5 & 10.5 & 10.8 & 10.1 & \textbf{11.7} \\
 & recharging & \textbf{6.1} & \textbf{6.4} & 4.9 & 5.3 & 5.2 & 5.4 & 5.3 & 5.2 & 4.8 & 5.7 & 5.6 & 5.6 & 5.7 & 5.6 & 5.4 & 5.0 & 5.6 & 5.7 & 5.7 & 5.8 & 5.9 \\
 & ricochet & 0.4 & 1.7 & 0.5 & 0.4 & 0.4 & 0.7 & 0.4 & 0.5 & 0.4 & 1.7 & 1.6 & 2.8 & 2.5 & 2.6 & \textbf{3.6} & 2.1 & \textbf{2.9} & 0.5 & 1.7 & 1.7 & 0.9 \\
 & rubiks & \textbf{4.0} & 4.0 & 3.0 & 3.0 & 3.0 & 3.0 & 3.0 & 3.0 & 3.0 & 3.1 & 3.3 & 3.0 & 3.2 & 3.9 & 3.4 & 3.5 & \textbf{4.2} & 3.1 & 3.0 & 3.0 & 3.4 \\
 & \textit{ipc23} & 23.6 & 24.8 & 16.4 & 16.9 & 18.6 & 18.4 & 18.0 & 17.8 & 19.0 & 25.5 & 23.7 & 24.9 & \textbf{25.9} & 25.7 & 25.3 & 22.9 & \textbf{26.5} & 22.0 & 24.4 & 24.0 & 24.6 \\
 & \textit{$\cg$ total} & 42.0 & 52.1 & 36.3 & 36.6 & 39.7 & 40.0 & 39.5 & 38.1 & 38.5 & 51.6 & 51.3 & 53.5 & \textbf{55.8} & 54.4 & 52.9 & 49.7 & \textbf{54.6} & 45.9 & 49.5 & 47.3 & 46.5 \\
\midrule & agricola & 0.9 & 1.4 & 1.7 & 1.7 & 1.9 & 1.8 & 1.9 & 1.7 & 1.4 & \textbf{3.2} & 2.8 & 2.9 & \textbf{2.9} & 2.5 & 2.5 & 1.8 & 2.8 & 2.6 & 1.7 & 1.4 & 0.8 \\
 & caldera & 2.9 & 4.6 & 5.4 & 5.6 & 5.8 & 5.6 & 5.7 & 5.4 & 5.2 & \textbf{6.0} & 5.7 & 5.5 & 5.2 & 5.6 & 5.6 & 5.8 & 5.7 & 5.4 & \textbf{5.9} & 5.8 & 4.9 \\
 & data-net & 3.4 & 5.1 & 5.9 & \textbf{6.0} & 6.0 & 5.5 & \textbf{6.1} & 5.3 & 5.6 & 3.3 & 3.6 & 3.1 & 3.2 & 2.9 & 2.8 & 2.7 & 3.5 & 2.4 & 3.4 & 4.2 & 4.4 \\
 & flashfill & 4.4 & 4.4 & 3.4 & 4.0 & 3.9 & 3.9 & 3.8 & 3.8 & 3.6 & 3.8 & 3.5 & 3.7 & 3.5 & 3.7 & 3.4 & 3.5 & 3.5 & \textbf{5.8} & \textbf{5.5} & 4.1 & 4.1 \\
 & nurikabe & 5.4 & 5.1 & \textbf{5.8} & 4.8 & 5.5 & 5.4 & 5.4 & 5.3 & 5.3 & 5.5 & 5.4 & 5.4 & 5.2 & 5.7 & 5.6 & 5.7 & \textbf{5.7} & 5.1 & 5.4 & 5.4 & 5.4 \\
 & org-syn & 3.7 & 3.5 & 3.6 & \textbf{4.3} & 3.7 & 3.8 & 3.6 & 3.6 & 3.4 & 3.3 & 3.4 & 3.3 & 3.5 & 3.6 & 3.7 & 3.7 & 3.5 & 3.4 & \textbf{3.9} & 3.6 & 3.6 \\
 & settlers & 0.0 & 2.2 & 1.9 & 1.8 & 1.5 & 1.6 & 1.6 & 1.7 & 1.8 & 1.5 & 1.6 & 2.2 & 2.5 & \textbf{2.8} & \textbf{3.1} & 2.6 & 1.1 & 1.2 & 0.6 & 0.4 & 0.0 \\
 & snake & 2.0 & 2.5 & 6.8 & 8.0 & 7.6 & 7.4 & 6.9 & 6.7 & 6.1 & 6.2 & 7.4 & 7.5 & \textbf{8.0} & 7.9 & 7.6 & 6.5 & \textbf{8.4} & 7.8 & 6.2 & 5.3 & 3.8 \\
 & spider & 1.5 & 2.9 & 3.3 & 3.0 & 3.1 & 3.3 & 3.0 & 3.2 & \textbf{3.5} & 2.4 & 2.9 & 3.1 & 2.8 & 2.7 & 2.8 & 3.0 & 3.0 & \textbf{3.6} & 2.9 & 2.3 & 2.4 \\
 & termes & 6.2 & 6.0 & 2.4 & 2.3 & 2.7 & 2.8 & 2.1 & 3.0 & 2.0 & 7.8 & 8.0 & 8.0 & 7.9 & 8.3 & \textbf{9.3} & \textbf{8.6} & 8.2 & 7.1 & 8.2 & 7.1 & 6.2 \\
 & \textit{ipc18} & 30.4 & 37.7 & 40.4 & 41.4 & 41.7 & 41.1 & 40.1 & 39.8 & 38.0 & 43.0 & 44.2 & 44.6 & 44.9 & \textbf{45.7} & \textbf{46.4} & 43.9 & 45.5 & 44.3 & 43.8 & 39.5 & 35.6 \\
 & folding & 0.1 & 0.1 & 0.0 & 0.1 & 0.1 & 0.1 & 0.2 & 0.2 & 0.2 & \textbf{1.8} & 1.6 & 1.6 & 1.4 & 1.1 & 1.0 & 0.5 & 1.5 & 1.4 & \textbf{1.9} & 1.1 & 0.3 \\
 & labyrinth & 0.0 & 0.0 & \textbf{0.0} & 0.0 & 0.0 & 0.0 & 0.0 & 0.0 & 0.0 & \textbf{0.0} & 0.0 & 0.0 & 0.0 & 0.0 & 0.0 & 0.0 & 0.0 & 0.0 & 0.0 & 0.0 & 0.0 \\
 & quantum & 13.8 & \textbf{15.4} & 10.9 & 10.8 & 12.5 & 11.9 & 11.0 & 11.5 & 12.1 & 14.8 & 14.2 & 14.7 & 14.1 & 13.4 & 13.0 & 13.1 & 14.2 & 14.7 & 14.5 & \textbf{15.2} & 14.6 \\
 & recharging & \textbf{3.6} & \textbf{4.0} & 2.4 & 2.3 & 2.6 & 2.3 & 2.4 & 2.3 & 2.5 & 2.4 & 2.4 & 2.7 & 2.7 & 2.6 & 2.7 & 3.0 & 2.7 & 2.3 & 2.9 & 2.9 & 3.4 \\
 & ricochet & 1.0 & 2.7 & 0.6 & 1.0 & 1.3 & 1.3 & 1.3 & 1.1 & 0.8 & 3.7 & 4.4 & 4.5 & \textbf{5.1} & \textbf{5.8} & 4.9 & 4.4 & 4.5 & 1.1 & 4.3 & 3.9 & 3.5 \\
 & rubiks & \textbf{12.5} & \textbf{13.3} & 5.0 & 5.0 & 5.0 & 5.0 & 5.0 & 5.0 & 5.0 & 10.2 & 10.2 & 10.5 & 11.2 & 11.5 & 10.6 & 11.0 & 10.9 & 7.9 & 10.3 & 9.6 & 11.7 \\
 & \textit{ipc23} & 31.0 & \textbf{35.5} & 19.0 & 19.1 & 21.5 & 20.6 & 19.9 & 20.0 & 20.6 & 32.9 & 32.7 & 33.8 & \textbf{34.4} & 34.4 & 32.2 & 32.0 & 33.7 & 27.3 & 34.0 & 32.7 & 33.4 \\
 & \textit{$\ff$ total} & 61.4 & 73.2 & 59.4 & 60.6 & 63.2 & 61.7 & 60.0 & 59.9 & 58.5 & 75.9 & 76.9 & 78.4 & \textbf{79.4} & \textbf{80.1} & 78.6 & 75.9 & 79.2 & 71.6 & 77.7 & 72.2 & 69.1 \\
\midrule & \textit{$\ce+\cg+\ff$} & 144.2 & 169.9 & 141.0 & 142.1 & 149.7 & 147.8 & 144.0 & 143.3 & 141.2 & 173.8 & 176.0 & 181.4 & \textbf{184.1} & 182.3 & 180.1 & 174.0 & \textbf{182.4} & 165.6 & 175.7 & 165.2 & 158.1 \\
\bottomrule
\end{tabular}
\end{adjustbox}
 \caption{
 Evaluating various MCTS enhancements with eager evaluation.
 }
 \label{tbl:base-table-domainwise}
\end{table*}

\begin{figure*}[p]
 \centering
 \includegraphics[width=.3\linewidth]{img+static+scatter-eval_sec+791760ca3a107e4fd6a7bc13d38007ac.pdf}
 \includegraphics[width=.3\linewidth]{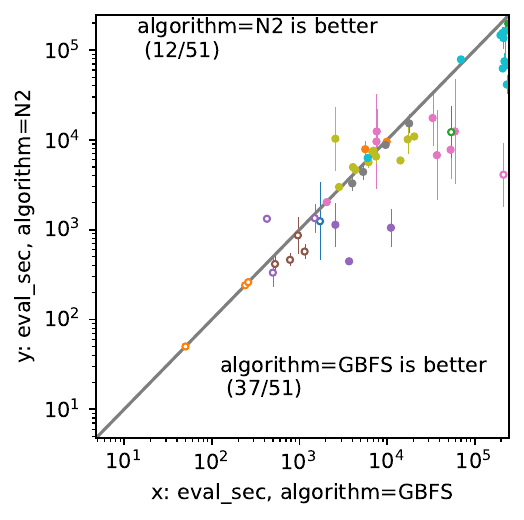}
 \includegraphics[width=.3\linewidth]{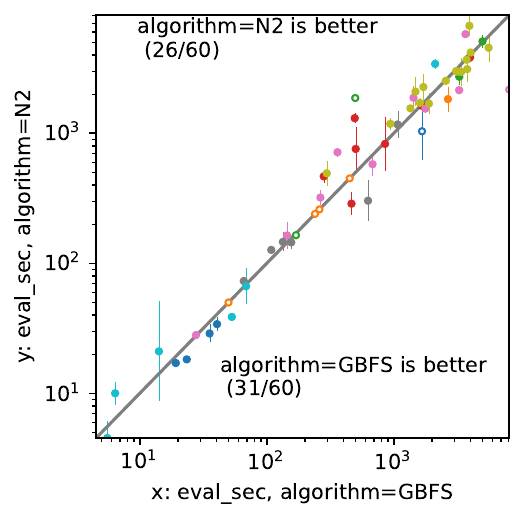}
 \caption{
 ($\ff,\cg,\ce$ results in order)
 Comparing the number of node evaluations per second on IPC instances solved by both GBFS ($x$-axis) vs. GUCTN2 ($y$-axis) within the limit.
 The points below the diagonal indicate that the latter has a significantly slower node evaluations.
 }
 \label{fig:evalsec-bilevel-more}
\end{figure*}

\begin{figure*}[p]
 \centering
 \begin{tabular}{rcc}
  &GUCTN2 &\green{Bilevel GUCTN2} \\
  \raisebox{0.11\linewidth}{$\ff$}
  &\includegraphics[height=0.24\linewidth]{img+static+depth_evalsec+5a7342ee1b712f528c7db74a3c9017a1.pdf}
      &\includegraphics[height=0.24\linewidth]{img+static+depth_evalsec+72efaa6d5cafb9d416f3cbe304f6873e.pdf} \\
  \raisebox{0.11\linewidth}{$\cg$}
  &\includegraphics[height=0.24\linewidth]{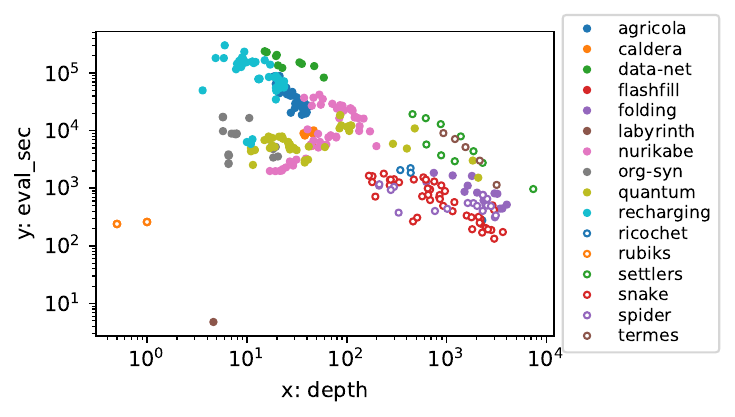}
      &\includegraphics[height=0.24\linewidth]{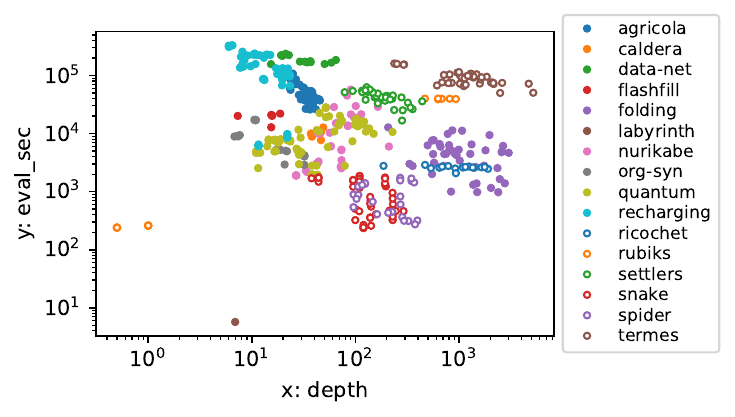} \\
  \raisebox{0.11\linewidth}{$\ce$}
  &\includegraphics[height=0.24\linewidth]{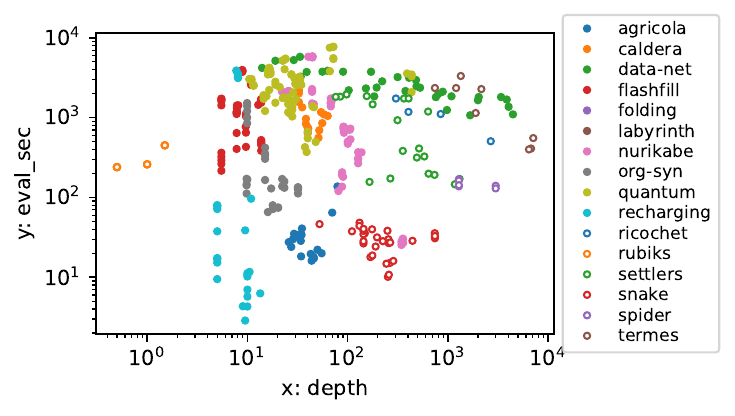}
      &\includegraphics[height=0.24\linewidth]{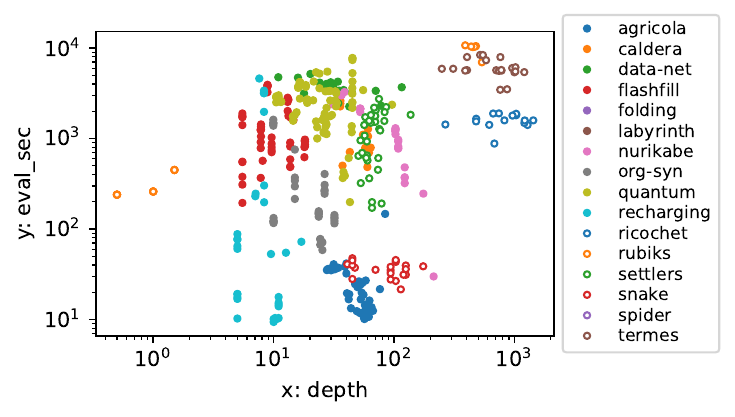} \\
 \end{tabular}
 \caption{
 ($\ff,\cg,\ce$ results from top to bottom, GUCTN2 and Bilevel GUCTN2 from left to right.)
 Log-log plots comparing the number of node evaluations per second ($y$-axis)
 versus the average depth of the nodes evaluated during the search ($x$-axis).
 }
 \label{fig:evalsec-correlation-more}
\end{figure*}

\begin{figure*}[p]
 \centering
 \includegraphics[width=0.24\linewidth]{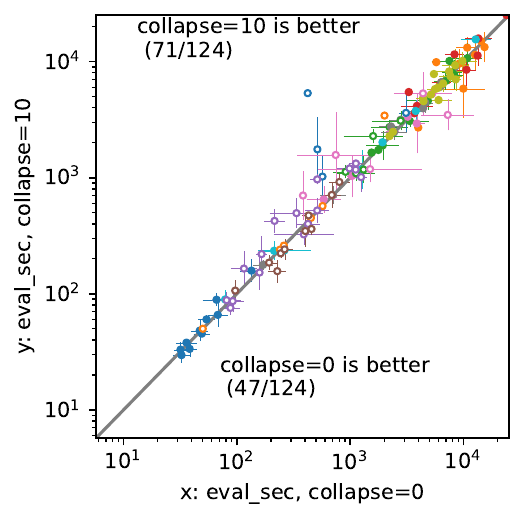}
 \includegraphics[width=0.24\linewidth]{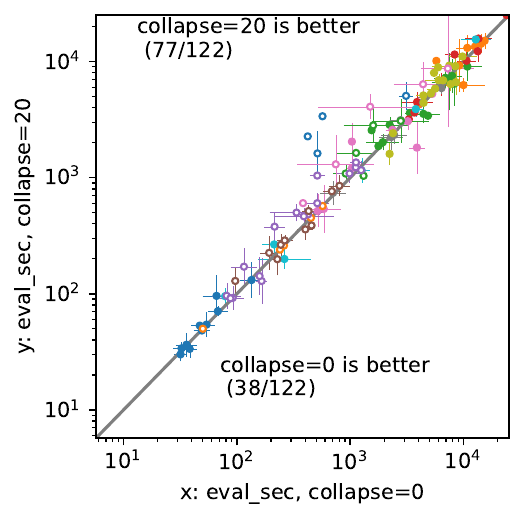}
 \includegraphics[width=0.24\linewidth]{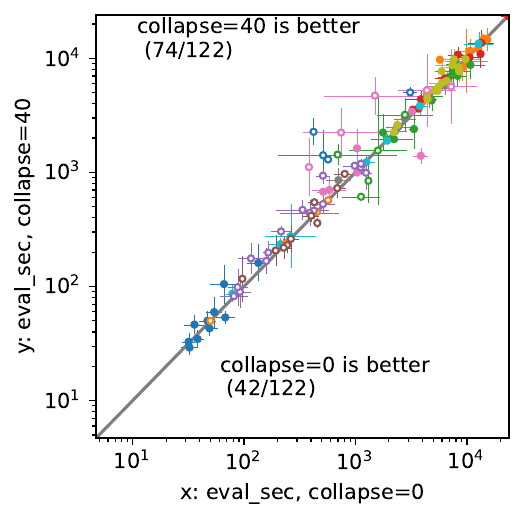}
 \includegraphics[width=0.24\linewidth]{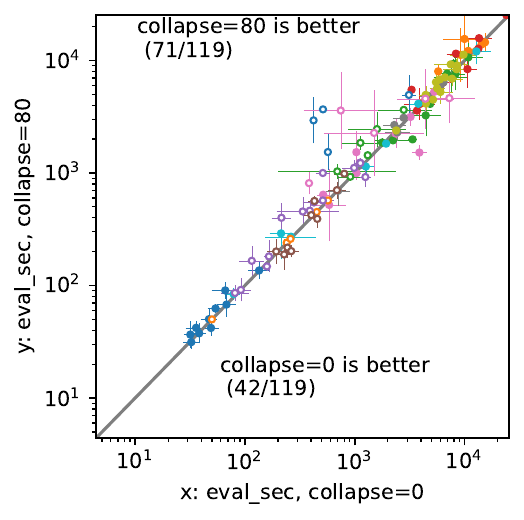}
 \\
 \includegraphics[width=0.24\linewidth]{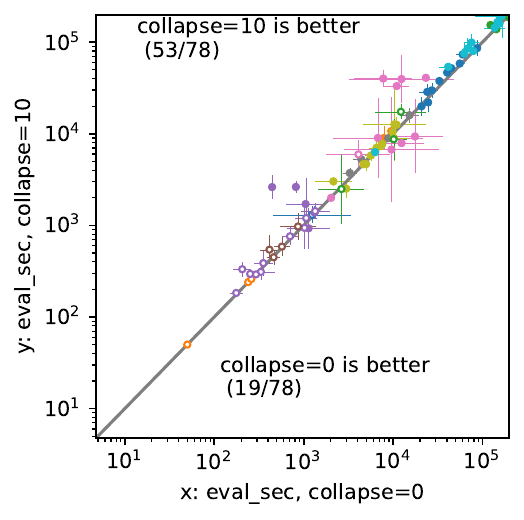}
 \includegraphics[width=0.24\linewidth]{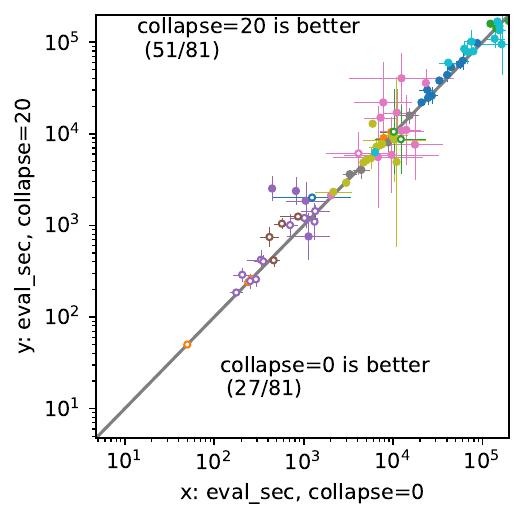}
 \includegraphics[width=0.24\linewidth]{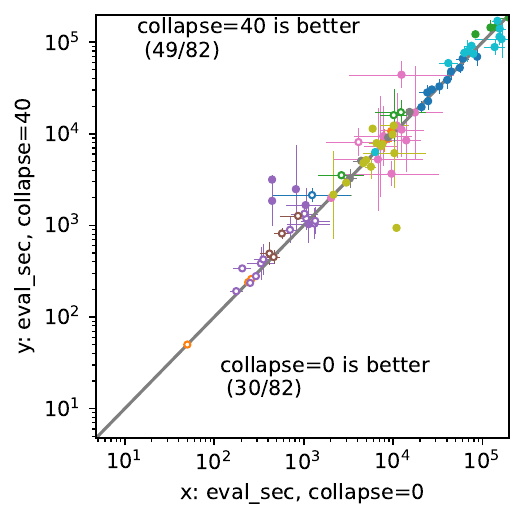}
 \includegraphics[width=0.24\linewidth]{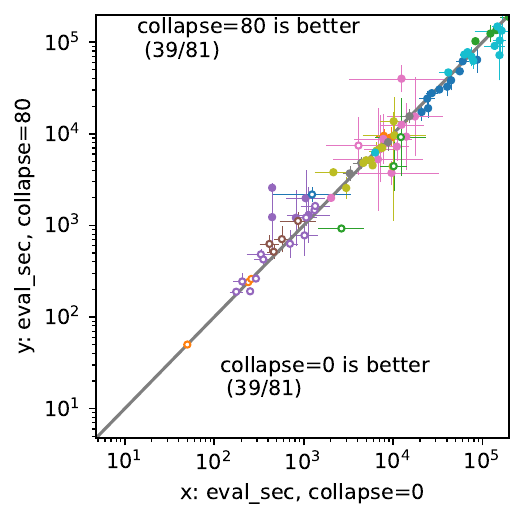}
 \\
 \includegraphics[width=0.24\linewidth]{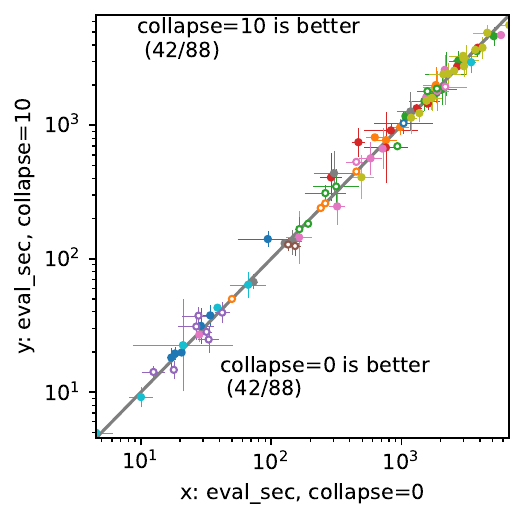}
 \includegraphics[width=0.24\linewidth]{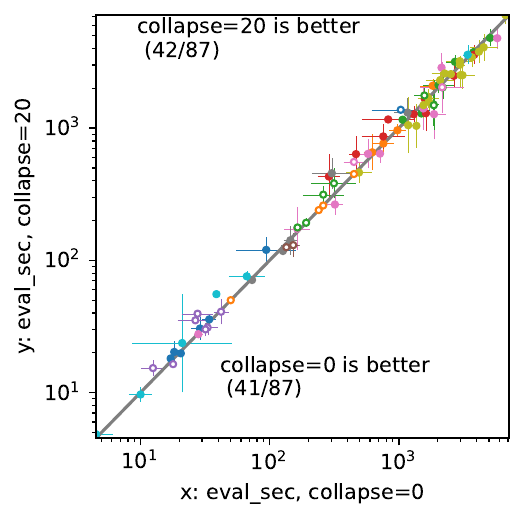}
 \includegraphics[width=0.24\linewidth]{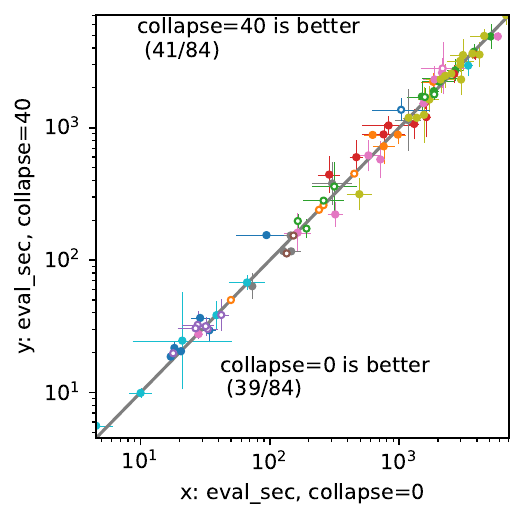}
 \includegraphics[width=0.24\linewidth]{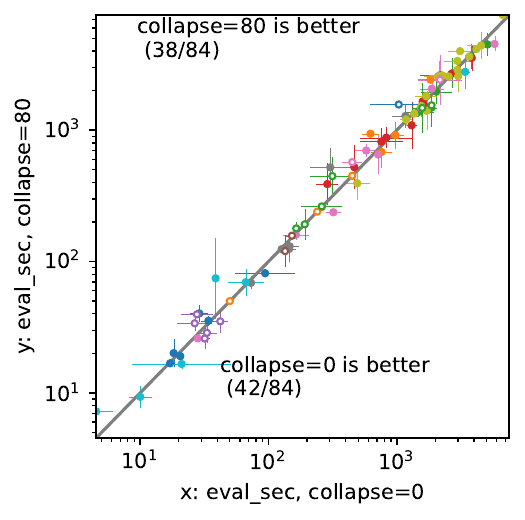}
 \\
 \includegraphics[width=0.3\linewidth]{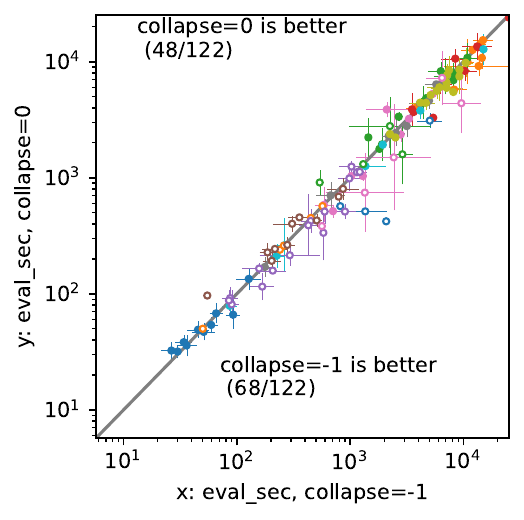}
 \includegraphics[width=0.3\linewidth]{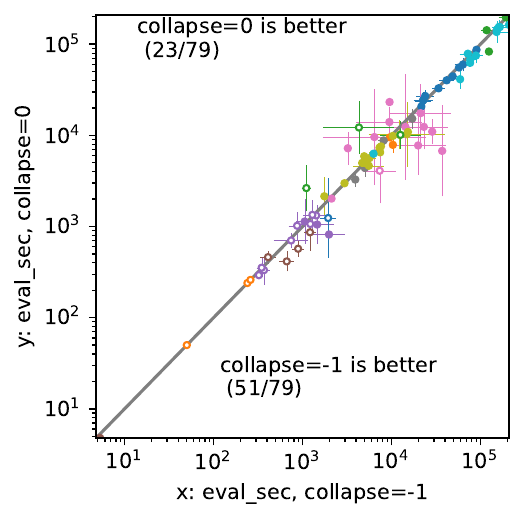}
 \includegraphics[width=0.3\linewidth]{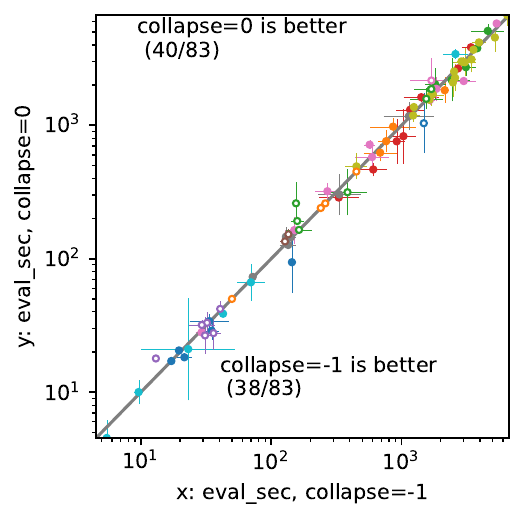}
 \caption{
 \textbf{(Top three rows)}
 Comparing the node/sec with \cyan{Tree Collapsing with $\theta\in \braces{10,20,40,80}$} (left to right) vs. $\theta=0$ (disabled) using $\ff$, $\cg$, $\ce$ (top to bottom).
 The effect is larger in the lighter heuristic ($\cg$) than in $\ff$,
 while the effect is smaller in the heavier heuristic ($\ce$).
 \textbf{(Bottom row)}
 Comparing the node/sec with \orange{Dynamic Tree Collapsing} (denoted as ``collapse=-1'' in the figure)
 vs $\theta=0$ (disabled) using $\ff$, $\cg$, $\ce$.
 The effect is larger in the lighter heuristic ($\cg$) than in $\ff$.
 The effect was insignificantly negative in the heavier heuristic ($\ce$).
 }
 \label{fig:collapse-node-sec}
\end{figure*}

\begin{table*}[p]
\centering
\begin{adjustbox}{max width=\linewidth}
\begin{tabular}{rr|rr|rr|rr|rr|}
\toprule
 &  & \multicolumn{2}{c|}{$\vf_4$}& \multicolumn{2}{c|}{$\vf_5$}& \multicolumn{2}{c|}{$\vf_6$}& \multicolumn{2}{c|}{NOLAN} \\
 &  & $w_{\max}=2$ & $V=100$ & $w_{\max}=2$ & $V=100$ & $w_{\max}=2$ & $V=100$ & $w_{\max}=2$ & $V=100$ \\
\midrule\multirow{19}{*}{\rotatebox{90}{\# solved}} & agricola & \textbf{9} & 3 & \textbf{12} & 1 & \textbf{10} & 5 & \textbf{7} & 5 \\
 & caldera & \textbf{12} & 3 & \textbf{6} & 0.0 & \textbf{12} & 2 & \textbf{11} & 4.0 \\
 & data-net & \textbf{11} & 1 & \textbf{8} & 0.0 & \textbf{10} & 1 & \textbf{8} & \textbf{8.0} \\
 & flashfill & \textbf{12} & 7 & \textbf{0.0} & \textbf{0.0} & \textbf{11} & 5 & \textbf{11.0} & 8.0 \\
 & nurikabe & \textbf{12} & 7 & \textbf{14} & 3 & \textbf{12} & 7 & \textbf{13} & 12 \\
 & org-syn & \textbf{12} & 3 & \textbf{4} & 1 & \textbf{11} & 2 & \textbf{9} & 6 \\
 & settlers & \textbf{5} & 4 & \textbf{5} & 1 & \textbf{4} & \textbf{4} & \textbf{6} & 5 \\
 & snake & \textbf{9} & 0.0 & \textbf{20} & 0.0 & \textbf{11} & 0 & \textbf{10} & 4.0 \\
 & spider & \textbf{17} & 14 & \textbf{13} & 2 & \textbf{16} & 13 & \textbf{17} & 13 \\
 & termes & 10 & \textbf{11} & \textbf{10} & 9 & \textbf{10} & \textbf{10} & \textbf{13} & \textbf{13} \\
 & \textit{ipc18} & \textbf{109} & 53 & \textbf{92} & 17 & \textbf{107} & 49 & \textbf{105} & 78 \\
\addlinespace & folding & 7 & \textbf{8} & \textbf{8} & 0.0 & \textbf{6} & \textbf{6} & \textbf{9} & \textbf{9.0} \\
 & labyrinth & \textbf{0.0} & \textbf{0.0} & \textbf{4} & 1 & \textbf{0.0} & \textbf{0} & 0.0 & \textbf{1} \\
 & quantum & \textbf{20} & 1 & \textbf{20} & 1 & \textbf{20} & 1 & \textbf{20} & 10 \\
 & recharging & 9 & \textbf{10} & \textbf{15} & \textbf{15} & \textbf{9} & \textbf{9} & \textbf{10} & \textbf{10} \\
 & ricochet & \textbf{13} & 0.0 & \textbf{20} & 0.0 & \textbf{16} & 0 & \textbf{14} & 8.0 \\
 & rubiks & \textbf{6} & 4 & \textbf{4} & \textbf{4} & \textbf{5} & 4 & 19 & \textbf{20} \\
 & \textit{ipc23} & \textbf{55} & 23 & \textbf{71} & 21 & \textbf{56} & 20 & \textbf{72} & 58 \\
\addlinespace & \textit{total} & \textbf{164} & 76 & \textbf{163} & 38 & \textbf{163} & 69 & \textbf{177} & 136 \\
\midrule\multirow{19}{*}{\rotatebox{90}{Agile IPC Score}} & agricola & \textbf{1.4} & 0.6 & \textbf{3.5} & 0.4 & \textbf{1.8} & 1.5 & 0.7 & \textbf{1.4} \\
 & caldera & \textbf{7.3} & 1.9 & \textbf{5.1} & 0.0 & \textbf{8.0} & 1.5 & \textbf{6.1} & 2.5 \\
 & data-net & \textbf{6.2} & 0.8 & \textbf{4.8} & 0.0 & \textbf{6.8} & 1.0 & 4.6 & \textbf{5.8} \\
 & flashfill & \textbf{5.8} & 3.7 & \textbf{0.0} & \textbf{0.0} & \textbf{5.3} & 2.7 & \textbf{5.8} & 4.9 \\
 & nurikabe & \textbf{7.4} & 3.6 & \textbf{8.9} & 1.5 & \textbf{7.5} & 3.6 & \textbf{8.1} & 7.0 \\
 & org-syn & \textbf{4.3} & 1.2 & \textbf{1.8} & 0.4 & \textbf{4.0} & 0.6 & \textbf{3.3} & 2.5 \\
 & settlers & 1.2 & \textbf{1.3} & \textbf{1.1} & 0.5 & \textbf{0.9} & 0.2 & 1.7 & \textbf{1.8} \\
 & snake & \textbf{5.1} & 0.0 & \textbf{8.8} & 0.0 & \textbf{5.1} & 0.0 & \textbf{3.2} & 1.7 \\
 & spider & \textbf{7.4} & 5.3 & \textbf{5.8} & 1.0 & \textbf{6.7} & 5.1 & \textbf{5.7} & 5.2 \\
 & termes & \textbf{7.8} & 6.1 & \textbf{5.4} & 4.0 & \textbf{5.3} & 5.3 & \textbf{8.3} & 8.1 \\
 & \textit{ipc18} & \textbf{53.9} & 24.4 & \textbf{45.3} & 7.8 & \textbf{51.4} & 21.6 & \textbf{47.5} & 40.9 \\
\addlinespace & folding & 2.2 & \textbf{2.5} & \textbf{3.5} & 0.0 & 1.8 & \textbf{1.9} & \textbf{3.4} & 3.2 \\
 & labyrinth & \textbf{0.0} & \textbf{0.0} & \textbf{0.2} & 0.1 & \textbf{0.0} & \textbf{0.0} & 0.0 & \textbf{0.1} \\
 & quantum & \textbf{18.1} & 1.0 & \textbf{19.1} & 1.0 & \textbf{18.7} & 0.8 & \textbf{18.0} & 8.3 \\
 & recharging & 3.0 & \textbf{3.7} & \textbf{7.4} & 6.7 & 4.1 & \textbf{4.2} & \textbf{4.9} & 4.5 \\
 & ricochet & \textbf{7.6} & 0.0 & \textbf{9.7} & 0.0 & \textbf{5.8} & 0.0 & \textbf{9.5} & 3.3 \\
 & rubiks & \textbf{5.5} & 3.7 & \textbf{4.0} & 3.8 & \textbf{4.3} & 3.9 & 9.4 & \textbf{12.2} \\
 & \textit{ipc23} & \textbf{36.4} & 11.0 & \textbf{43.9} & 11.6 & \textbf{34.7} & 10.9 & \textbf{45.2} & 31.7 \\
\addlinespace & \textit{total} & \textbf{90.2} & 35.4 & \textbf{89.2} & 19.5 & \textbf{86.1} & 32.4 & \textbf{92.8} & 72.5 \\
\bottomrule
\end{tabular}
\end{adjustbox}
\caption{
Results of $\vf_{4..6}$ and NOLAN comparing $w_{\max}=2$ and dynamic $w_{\max}$ based on the threshold $V=100$.
The dynamic max-width approach did not help the performance in the domains evaluated here.
}
\label{tbl:nolan-dynamic-width}
\end{table*}

\begin{table*}[p]
\centering
\begin{adjustbox}{max width=\linewidth}
\begin{tabular}{rrrrrrrrr}
\toprule
 &   & \multicolumn{7}{c}{\cyan{Tree Collapsing $\theta=$} } \\
 &   & \cyan{10} & \cyan{20} & \cyan{40} & \cyan{80} & \cyan{160} & \cyan{320} & \orange{DTC} \\
\midrule\multirow{19}{*}{\rotatebox{90}{\# solved}} & agricola & 9 & 9.8 & 10.2 & \textbf{10.4} & 9 & \textbf{10.4} & \textbf{11.4} \\
 & caldera & \textbf{11} & 10 & 10 & 10 & 10.4 & \textbf{10.6} & 10.4 \\
 & data-net & \textbf{17} & 16.8 & 15.6 & 16.6 & \textbf{17.2} & 16.8 & \textbf{17} \\
 & flashfill & \textbf{12} & 11.4 & 11.4 & \textbf{11.6} & 11.4 & \textbf{11.6} & \textbf{11.6} \\
 & nurikabe & 9.8 & \textbf{10} & \textbf{10} & 9.8 & \textbf{10.2} & \textbf{10} & 9.8 \\
 & org-syn & 8.4 & 8.4 & \textbf{9.4} & 8.6 & 8.4 & 8.6 & \textbf{9.2} \\
 & settlers & 14.4 & 15 & 15.2 & \textbf{15.6} & 15.4 & \textbf{15.6} & 14.6 \\
 & snake & 10 & 10.2 & \textbf{10.6} & 10.2 & \textbf{11.2} & 9.6 & \textbf{10.6} \\
 & spider & \textbf{13.8} & 13.4 & \textbf{13.8} & \textbf{14} & 12.8 & \textbf{13.8} & \textbf{13.8} \\
 & termes & 12.8 & \textbf{13.6} & \textbf{13.6} & 13.4 & 12.6 & 12.2 & \textbf{13.6} \\
 & \textit{ipc18} & 118.2 & 118.6 & 119.8 & \textbf{120.2} & 118.6 & 119.2 & \textbf{122} \\
\addlinespace & folding & \textbf{7.8} & 7 & \textbf{7.4} & \textbf{7.4} & 6.2 & 7 & \textbf{7.4} \\
 & labyrinth & \textbf{0.4} & \textbf{0.4} & 0.0 & 0.2 & 0.2 & 0.2 & \textbf{0.4} \\
 & quantum & \textbf{20} & \textbf{20} & \textbf{20} & \textbf{20} & \textbf{20} & \textbf{20} & \textbf{20} \\
 & recharging & 10 & \textbf{10.4} & \textbf{10.2} & 10 & 9.4 & \textbf{10.2} & 9.8 \\
 & ricochet & 16.4 & 15.2 & \textbf{18.2} & 13.6 & 16.6 & 14 & \textbf{16.8} \\
 & rubiks & \textbf{17.4} & \textbf{17.4} & 14.6 & 11.6 & 14.8 & 13.8 & 15.8 \\
 & \textit{ipc23} & \textbf{72} & \textbf{70.4} & \textbf{70.4} & 62.8 & 67.2 & 65.2 & 70.2 \\
\addlinespace & \textit{total} & \textbf{190.2} & 189 & \textbf{190.2} & 183 & 185.8 & 184.4 & \textbf{192.2} \\
\midrule\multirow{19}{*}{\rotatebox{90}{Agile IPC Score}} & agricola & 1.8 & 2.0 & 2.0 & 2.1 & 2.1 & \textbf{2.1} & \textbf{2.4} \\
 & caldera & 7.5 & \textbf{7.5} & 7.0 & 7.2 & 7.1 & 7.5 & \textbf{7.8} \\
 & data-net & 8.4 & 8.4 & 8.1 & 8.7 & 8.6 & \textbf{9.2} & \textbf{10.1} \\
 & flashfill & \textbf{6.2} & 6.2 & 5.9 & 5.9 & 6.1 & 6.1 & \textbf{6.3} \\
 & nurikabe & 5.9 & 6.2 & \textbf{6.3} & 6.2 & \textbf{6.4} & 6.0 & 6.0 \\
 & org-syn & 3.0 & 3.0 & \textbf{3.4} & 3.4 & 3.2 & 3.3 & \textbf{3.7} \\
 & settlers & 5.7 & 6.5 & \textbf{6.7} & 6.7 & 6.4 & \textbf{6.7} & 6.4 \\
 & snake & 4.6 & \textbf{4.7} & \textbf{4.7} & 4.2 & 4.5 & 4.4 & 4.5 \\
 & spider & 4.6 & 4.1 & 4.6 & \textbf{4.8} & 4.2 & \textbf{4.6} & 4.5 \\
 & termes & 7.5 & 8.1 & 7.6 & \textbf{8.2} & 7.3 & 7.3 & \textbf{8.4} \\
 & \textit{ipc18} & 55.4 & 56.8 & 56.2 & 57.3 & 55.8 & \textbf{57.3} & \textbf{60.0} \\
\addlinespace & folding & 2.4 & 2.2 & 2.7 & \textbf{2.7} & 2.3 & 2.2 & \textbf{2.7} \\
 & labyrinth & \textbf{0.2} & 0.0 & 0.0 & 0.0 & 0.0 & 0.0 & \textbf{0.0} \\
 & quantum & 18.3 & 18.3 & 18.0 & \textbf{18.3} & 18.2 & 18.2 & \textbf{18.4} \\
 & recharging & 3.9 & 3.8 & \textbf{4.1} & 3.9 & 3.6 & \textbf{4.2} & 3.7 \\
 & ricochet & \textbf{6.2} & 5.5 & 5.8 & 4.9 & 5.0 & 4.4 & \textbf{6.5} \\
 & rubiks & \textbf{8.1} & \textbf{8.1} & 6.8 & 7.2 & 7.3 & 6.7 & 8.0 \\
 & \textit{ipc23} & \textbf{39.0} & 37.8 & 37.4 & 37.1 & 36.4 & 35.8 & \textbf{39.4} \\
\addlinespace & \textit{total} & 94.4 & \textbf{94.6} & 93.5 & 94.4 & 92.1 & 93.1 & \textbf{99.4} \\
\bottomrule
\end{tabular}
\end{adjustbox}
\caption{
Results of \coolname with \cyan{$\theta\in\braces{10,20,40,80,160,320}$},
as well as \orange{Dynamic Tree Collapsing}.
Top-2 (including ties) are highlighted in \textbf{bold}.
}
\label{tbl:lamanormal2-collapsing}
\end{table*}

\begin{table*}[p]
\centering
\begin{adjustbox}{max width=\linewidth}
\begin{tabular}{rc|rr|rr||rrrr|rrrrrc}
\toprule
 & IPC   & \multicolumn{2}{c|}{Lapkt-BFWS} & \multicolumn{2}{c||}{FD-BFWS} & \multicolumn{2}{c}{LAMA} & Dec  & NO & LAMAe  & \green{N2}\orange{DTC} & \green{N2}\orange{DTC} & \multirow{2}{*}{\coolname}\\
 & year  & Apx$^{\text{fd}}$ & $\vf_5$$^{\text{fd}}$ & $\vf_4$ & $\vf_5$ &  & +SM & Star & LAN                                      & +BFWS & +BFWS           & +LAMAe            & \\
\midrule\multirow{19}{*}{\rotatebox{90}{\# solved}} & agricola & \textbf{12} & 10 & 9 & \textbf{12} & 10 & 11 & 11 & 7 & 11 & \textbf{12.6} & 11.4 & 11.4 \\
 & caldera & 10 & 2 & \textbf{12} & 6 & 6 & 6.2 & 6 & 11 & \textbf{12} & 10.2 & 5.8 & 10.4 \\
 & data-net & 15 & 11 & 11 & 8 & 11 & \textbf{16} & 9 & 8 & 9 & \textbf{16} & 8.2 & \textbf{17} \\
 & flashfill & 13 & 12 & 12 & 0.0 & \textbf{14} & 12.6 & \textbf{14} & 11 & 9 & 5.8 & 9.2 & 11.6 \\
 & nurikabe & 0.0 & 0.0 & 12 & \textbf{14} & 9 & 9 & 8 & \textbf{13} & 12 & 10.8 & 8.4 & 9.8 \\
 & org-syn & 7 & 5 & \textbf{12} & 4 & 10 & 10 & 9 & 9 & \textbf{11} & 8.6 & 10.6 & 9.2 \\
 & settlers & 7 & 6 & 5 & 5 & \textbf{17} & \textbf{17} & 12 & 6 & 4 & 0.6 & 6.2 & 14.6 \\
 & snake & \textbf{20} & 18 & 9 & \textbf{20} & 4 & 10.2 & 5 & 10 & 8 & 18.6 & 9.6 & 10.6 \\
 & spider & 12 & 11 & \textbf{17} & 13 & 16 & 15.2 & 13 & \textbf{17} & 15 & 12.2 & 11.6 & 13.8 \\
 & termes & 3 & 8 & 10 & 10 & 14 & 13.2 & 12 & 13 & 9 & 11.6 & \textbf{13.6} & \textbf{13.6} \\
\multicolumn{1}{l}{} & \textit{ipc18} & 99 & 83 & 109 & 92 & 111 & \textbf{120.4} & 99 & 105 & 100 & 107 & 94.6 & \textbf{122} \\
\addlinespace & folding & 4 & 7 & 7 & 8 & \textbf{11} & \textbf{11} & 7 & 9.0 & 8.0 & 5.6 & 10.2 & 7.4 \\
 & labyrinth & \textbf{14} & \textbf{15} & 0 & 4 & 0.0 & 0.2 & 0.0 & 0 & 0 & 0.4 & 0.2 & 0.4 \\
 & quantum & \textbf{20} & \textbf{20} & \textbf{20} & \textbf{20} & 19 & 19 & 19 & \textbf{20} & \textbf{20} & 19.6 & 19.2 & \textbf{20} \\
 & recharging & \textbf{15} & \textbf{16} & 9.0 & \textbf{15} & 10 & 10 & 10 & 10 & 12 & 7.8 & 6 & 9.8 \\
 & ricochet & 7 & 1 & 13 & \textbf{20} & 6 & 5.8 & 7 & 14 & 14 & 14 & 9.2 & \textbf{16.8} \\
 & rubiks & 4 & 5 & 6 & 4 & \textbf{20} & \textbf{20} & \textbf{20} & 19.0 & 18.0 & 13.0 & 16.8 & 15.8 \\
 & slitherlink & \textbf{3} & \textbf{3} & 0 & 0.0 & 0.0 & 0.0 & 0.0 & 0 & 0 & 0 & 0 & 0 \\
 & \textit{ipc23} & 67 & 67 & 55 & 71 & 66 & 66 & 63 & \textbf{72} & \textbf{72} & 60.4 & 61.6 & 70.2 \\
\addlinespace & \textit{total} & 166 & 150 & 164 & 163 & 177 & \textbf{186.4} & 162 & 177 & 172 & 167.4 & 156.2 & \textbf{192.2} \\
\midrule\multirow{19}{*}{\rotatebox{90}{Agile IPC Score}} & agricola & \textbf{2.9} & 2.4 & 1.4 & \textbf{3.5} & 2.0 & 2.1 & 2.5 & 0.7 & 1.5 & 2.7 & 2.0 & 2.4 \\
 & caldera & 6.4 & 1.8 & 7.3 & 5.1 & 3.8 & 4.2 & 4.1 & 6.1 & \textbf{7.8} & \textbf{7.9} & 5.1 & 7.8 \\
 & data-net & 7.1 & 5.8 & 6.2 & 4.8 & 8.3 & \textbf{10.4} & 6.1 & 4.6 & 5.4 & 8.5 & 4.7 & \textbf{10.1} \\
 & flashfill & \textbf{9.0} & \textbf{8.6} & 5.8 & 0.0 & 7.8 & 7.3 & 8.3 & 5.8 & 5.0 & 2.4 & 4.8 & 6.3 \\
 & nurikabe & 0.0 & 0.0 & 7.4 & \textbf{8.9} & 6.3 & 6.3 & 5.7 & \textbf{8.1} & 7.3 & 6.8 & 5.6 & 6.0 \\
 & org-syn & \textbf{3.9} & 3.8 & \textbf{4.3} & 1.8 & 3.3 & 3.5 & 3.5 & 3.3 & 3.7 & 3.4 & 3.6 & 3.7 \\
 & settlers & 2.7 & 2.1 & 1.2 & 1.1 & \textbf{8.5} & \textbf{8.8} & 6.8 & 1.7 & 1.5 & 0.2 & 1.7 & 6.4 \\
 & snake & \textbf{13.1} & \textbf{9.1} & 5.1 & 8.8 & 2.0 & 4.0 & 2.3 & 3.2 & 3.5 & 8.4 & 4.3 & 4.5 \\
 & spider & 4.6 & 4.5 & \textbf{7.4} & \textbf{5.8} & 5.4 & 5.0 & 5.2 & 5.7 & 5.6 & 4.0 & 3.7 & 4.5 \\
 & termes & 1.5 & 3.4 & 7.8 & 5.4 & \textbf{9.6} & 8.6 & 6.5 & 8.3 & 4.6 & 7.7 & \textbf{8.9} & 8.4 \\
 & \textit{ipc18} & 51.3 & 41.4 & 53.9 & 45.3 & 57.1 & \textbf{60.3} & 51.1 & 47.5 & 45.8 & 52.1 & 44.3 & \textbf{60.0} \\
\addlinespace & folding & 1.7 & 2.1 & 2.2 & \textbf{3.5} & 3.4 & 3.4 & 2.4 & \textbf{3.4} & 2.5 & 1.5 & 3.1 & 2.7 \\
 & labyrinth & \textbf{3.9} & \textbf{3.8} & 0.0 & 0.2 & 0.0 & 0.0 & 0.0 & 0.0 & 0.0 & 0.0 & 0.0 & 0.0 \\
 & quantum & \textbf{19.9} & \textbf{19.9} & 18.1 & 19.1 & 16.9 & 17.2 & 17.8 & 18.0 & 17.4 & 16.7 & 17.4 & 18.4 \\
 & recharging & \textbf{7.5} & \textbf{7.7} & 3.0 & 7.4 & 3.6 & 4.0 & 4.0 & 4.9 & 4.9 & 3.3 & 2.3 & 3.7 \\
 & ricochet & 0.6 & 0.4 & 7.6 & \textbf{9.7} & 1.6 & 1.9 & 3.4 & \textbf{9.5} & 5.6 & 5.5 & 3.8 & 6.5 \\
 & rubiks & 4.0 & 4.1 & 5.5 & 4.0 & \textbf{12.7} & 12.6 & \textbf{15.0} & 9.4 & 9.0 & 7.7 & 8.2 & 8.0 \\
 & slitherlink & \textbf{1.5} & \textbf{1.9} & 0.0 & 0.0 & 0.0 & 0.0 & 0.0 & 0.0 & 0.0 & 0.0 & 0.0 & 0.0 \\
 & \textit{ipc23} & 39.0 & 39.9 & 36.4 & \textbf{43.9} & 38.3 & 39.0 & 42.6 & \textbf{45.2} & 39.3 & 34.7 & 34.8 & 39.4 \\
\addlinespace & \textit{total} & 90.3 & 81.3 & 90.2 & 89.2 & 95.4 & \textbf{99.3} & 93.7 & 92.8 & 85.2 & 86.8 & 79.1 & \textbf{99.4} \\
\bottomrule
\end{tabular}
\end{adjustbox}
\caption{
300 seconds agile run comparing \lsota planners.
Top-2 (including ties) are highlighted in \textbf{bold}.
}
\label{tbl:domain-wise}
\end{table*}

\begin{table*}[p]
\begin{adjustbox}{max width=\linewidth}
\begin{tabular}{rc|rr|rr||rrrr|rrrrrc}
\toprule
 & IPC   & \multicolumn{2}{c|}{Lapkt-BFWS} & \multicolumn{2}{c||}{FD-BFWS} & \multicolumn{2}{c}{LAMA} & Dec  & NO & LAMAe  & \green{N2}\orange{DTC} & \green{N2}\orange{DTC} & \multirow{2}{*}{\coolname}\\
 & year  & Apx$^{\text{fd}}$ & $\vf_5$$^{\text{fd}}$ & $\vf_4$ & $\vf_5$ &  & +SM & Star & LAN                                      & +BFWS & +BFWS           & +LAMAe            & \\
\midrule\multirow{19}{*}{\rotatebox{90}{\# solved}} & agricola & \textbf{18} & 10 & 14 & 12 & 12 & 12 & 12 & 11 & \textbf{15} & 14.4 & 12.8 & 13.8 \\
 & caldera & 12 & 2 & \textbf{13} & 6 & 7 & 7.2 & 7 & 12 & \textbf{18} & 12 & 5.8 & 11.8 \\
 & data-net & \textbf{19} & 11 & 12 & 11 & 13 & 17 & 10 & 12 & 10 & 18.2 & 11 & \textbf{19.8} \\
 & flashfill & 13 & 12 & 13 & 0.0 & \textbf{14} & \textbf{14} & \textbf{14} & 13 & 13 & 8.8 & 9.6 & 13.8 \\
 & nurikabe & 0.0 & 0.0 & 12 & \textbf{15} & 9 & 9.2 & 10 & \textbf{14} & \textbf{14} & 11.8 & 8.8 & 11 \\
 & org-syn & 7 & 5 & \textbf{13} & 5 & \textbf{14} & 12.4 & 11 & \textbf{13} & \textbf{13} & 10.8 & 12.2 & 12.6 \\
 & settlers & 11 & 10 & 12 & 7 & \textbf{17} & \textbf{17} & 13 & 11 & 13 & 0.8 & 7.6 & \textbf{17.4} \\
 & snake & \textbf{20} & 18 & 12 & \textbf{20} & 5 & 13 & 6 & 12 & 11 & 19.4 & 13.6 & 14 \\
 & spider & 14 & 13 & 17 & 15 & 16 & 16.6 & 13 & \textbf{19} & \textbf{19} & 16.2 & 14.4 & 16 \\
 & termes & 4 & 9 & 14 & 11 & \textbf{16} & 15 & 14 & \textbf{16} & 14 & 14.4 & 15.8 & 14.8 \\
 & \textit{ipc18} & 118 & 90 & 132 & 102 & 123 & 133.4 & 110 & 133 & \textbf{140} & 126.8 & 111.6 & \textbf{145} \\
\addlinespace & folding & 5 & 8 & 9 & 9 & \textbf{12} & \textbf{12} & 8 & 10.0 & 9.0 & 7 & \textbf{12.2} & 10.4 \\
 & labyrinth & \textbf{15} & \textbf{15} & 3 & 4 & 1.0 & 1 & 0.0 & 3 & 3 & 4 & 2 & 3.8 \\
 & quantum & \textbf{20} & \textbf{20} & \textbf{20} & \textbf{20} & \textbf{20} & 19 & \textbf{20} & \textbf{20} & \textbf{20} & \textbf{20} & 19.2 & \textbf{20} \\
 & recharging & \textbf{16} & \textbf{16} & 13.0 & 15 & 14 & 12.8 & 12 & 14 & 14 & 11 & 7 & 11.4 \\
 & ricochet & 19 & 1 & 20 & 20 & 14 & 15.2 & 10 & 19 & \textbf{20} & 19.8 & 13.2 & \textbf{20} \\
 & rubiks & 4 & 5 & 11 & 4 & \textbf{20} & \textbf{20} & \textbf{20} & \textbf{20.0} & \textbf{20.0} & 18.6 & 19.8 & \textbf{20.0} \\
 & slitherlink & \textbf{4} & \textbf{4} & 0 & 0.0 & 0.0 & 0.0 & 0.0 & 0 & 0 & 0 & 0 & 0 \\
 & \textit{ipc23} & 83 & 69 & 76 & 72 & 81 & 80 & 70 & \textbf{86} & \textbf{86} & 80.4 & 73.4 & 85.6 \\
\addlinespace & \textit{total} & 201 & 159 & 208 & 174 & 204 & 213.4 & 180 & 219 & \textbf{226} & 207.2 & 185 & \textbf{230.6} \\
\midrule\multirow{19}{*}{\rotatebox{90}{Extended IPC Score}} & agricola & \textbf{6.0} & 4.2 & 4.0 & \textbf{5.6} & 4.3 & 4.4 & 4.6 & 3.0 & 4.5 & 5.3 & 4.4 & 5.0 \\
 & caldera & 7.6 & 1.8 & 8.5 & 5.3 & 4.4 & 4.8 & 4.7 & 7.3 & \textbf{9.3} & \textbf{8.7} & 5.2 & 8.5 \\
 & data-net & 9.6 & 7.1 & 7.4 & 6.2 & 9.3 & \textbf{11.9} & 7.0 & 6.0 & 6.4 & 10.6 & 5.9 & \textbf{12.0} \\
 & flashfill & \textbf{9.9} & 9.4 & 7.4 & 0.0 & 9.3 & 8.7 & \textbf{9.7} & 7.4 & 6.5 & 3.6 & 5.9 & 7.9 \\
 & nurikabe & 0.0 & 0.0 & 8.5 & \textbf{10.3} & 7.0 & 7.0 & 6.4 & \textbf{9.3} & 8.7 & 7.9 & 6.3 & 7.0 \\
 & org-syn & 4.7 & 4.1 & \textbf{6.2} & 2.5 & 5.1 & 5.3 & 5.0 & 5.2 & \textbf{5.6} & 4.9 & 5.5 & 5.5 \\
 & settlers & 4.4 & 3.7 & 2.8 & 2.4 & \textbf{10.6} & \textbf{10.7} & 8.1 & 3.4 & 3.3 & 0.3 & 2.9 & 8.8 \\
 & snake & \textbf{14.8} & 11.2 & 6.5 & \textbf{11.5} & 2.6 & 5.9 & 3.0 & 5.0 & 4.8 & 11.0 & 6.0 & 6.4 \\
 & spider & 6.4 & 6.3 & \textbf{9.7} & 7.7 & 8.0 & 7.7 & 7.1 & \textbf{8.8} & 8.3 & 6.5 & 5.8 & 7.0 \\
 & termes & 2.1 & 4.7 & 8.8 & 6.7 & \textbf{10.7} & 9.9 & 8.0 & 9.7 & 6.1 & 8.9 & \textbf{10.3} & 9.9 \\
 & \textit{ipc18} & 65.4 & 52.5 & 69.8 & 58.1 & 71.0 & \textbf{76.3} & 63.6 & 65.0 & 63.5 & 67.8 & 58.4 & \textbf{77.9} \\
\addlinespace & folding & 2.4 & 3.5 & 3.6 & 4.7 & \textbf{5.3} & \textbf{5.3} & 3.6 & 5.0 & 4.0 & 2.7 & 5.0 & 4.2 \\
 & labyrinth & \textbf{6.5} & \textbf{6.5} & 0.4 & 1.1 & 0.1 & 0.1 & 0.0 & 0.4 & 0.3 & 0.6 & 0.2 & 0.5 \\
 & quantum & \textbf{19.9} & \textbf{19.9} & 18.5 & 19.3 & 17.5 & 17.6 & 18.3 & 18.4 & 18.0 & 17.4 & 17.8 & 18.8 \\
 & recharging & \textbf{9.5} & \textbf{9.7} & 4.9 & 9.2 & 5.5 & 5.8 & 5.7 & 6.5 & 6.7 & 4.8 & 3.3 & 5.2 \\
 & ricochet & 4.2 & 0.5 & 10.2 & \textbf{12.2} & 3.8 & 4.1 & 4.6 & \textbf{11.5} & 8.6 & 8.6 & 5.6 & 9.5 \\
 & rubiks & 4.0 & 4.3 & 6.1 & 4.0 & \textbf{14.5} & 14.3 & \textbf{16.2} & 11.9 & 11.6 & 9.8 & 10.8 & 10.6 \\
 & slitherlink & \textbf{2.0} & \textbf{2.3} & 0.0 & 0.0 & 0.0 & 0.0 & 0.0 & 0.0 & 0.0 & 0.0 & 0.0 & 0.0 \\
 & \textit{ipc23} & 48.5 & 46.8 & 43.8 & \textbf{50.4} & 46.6 & 47.2 & 48.4 & \textbf{53.7} & 49.2 & 44.0 & 42.7 & 48.9 \\
\addlinespace & \textit{total} & 113.9 & 99.2 & 113.7 & 108.6 & 117.6 & \textbf{123.4} & 112.0 & 118.7 & 112.8 & 111.8 & 101.0 & \textbf{126.8} \\
\bottomrule
\end{tabular}
\end{adjustbox}
\caption{
1800 seconds extended runs. Extended IPC score is measured by $\sum_i \max(1, 1-\nicefrac{\log t_i}{\log 1800})$.
Top-2 (including ties) are highlighted in \textbf{bold}.
}
\label{tbl:domain-wise-extended}
\end{table*}

\begin{table*}[p]
\centering
\begin{adjustbox}{max width=\linewidth}
\begin{tabular}{rc|rr|rr||rrrr|rrrrrc}
\toprule
 & IPC   & \multicolumn{2}{c|}{Lapkt-BFWS} & \multicolumn{2}{c||}{FD-BFWS} & \multicolumn{2}{c}{LAMA} & Dec  & NO & LAMAe  & \green{N2}\orange{DTC} & \green{N2}\orange{DTC} & \multirow{2}{*}{\coolname}\\
 & year  & Apx$^{\text{fd}}$ & $\vf_5$$^{\text{fd}}$ & $\vf_4$ & $\vf_5$ &  & +SM & Star & LAN                                      & +BFWS & +BFWS           & +LAMAe            & \\
\midrule\multirow{19}{*}{\rotatebox{90}{\# solved}} & agricola & 11.4 & 0.2 & 8.7 & \textbf{11.5} & 9.8 & 10.8 & 10.7 & 6.6 & 10.5 & \textbf{11.7} & 10.6 & 11.0 \\
 & caldera & 7.9 & 1.1 & \textbf{10.6} & 4.4 & 4.6 & 4.7 & 4.6 & 10.4 & \textbf{10.7} & 8.3 & 4.2 & 8.3 \\
 & data-net & \textbf{11.9} & 0.8 & 9.5 & 4.9 & 7.3 & 8.3 & 5.2 & 7.5 & 7.9 & 7.0 & 5.6 & \textbf{12.1} \\
 & flashfill & 11.8 & 0.6 & 11.2 & 0.0 & \textbf{13.7} & 12.4 & \textbf{12.9} & 10.6 & 8.1 & 3.3 & 7.0 & 10.8 \\
 & nurikabe & 0.0 & 0.0 & 9.8 & \textbf{12.2} & 8.3 & 8.3 & 6.0 & \textbf{11.8} & 9.8 & 7.6 & 7.8 & 8.5 \\
 & org-syn & 7.0 & 0.8 & \textbf{11.7} & 4.0 & 10.0 & 10.0 & 9.0 & 9.0 & \textbf{11.0} & 8.6 & 10.4 & 9.1 \\
 & settlers & 5.7 & 0.7 & 4.0 & 4.1 & \textbf{14.3} & \textbf{13.0} & 9.7 & 5.5 & 3.3 & 0.3 & 3.1 & 11.7 \\
 & snake & \textbf{17.0} & \textbf{17.8} & 5.4 & 15.8 & 2.7 & 5.4 & 2.6 & 7.5 & 4.9 & 6.1 & 4.5 & 5.8 \\
 & spider & 2.4 & \textbf{11.0} & 6.0 & 3.1 & 5.2 & 4.5 & 3.2 & \textbf{6.1} & 4.5 & 2.3 & 2.9 & 4.2 \\
 & termes & 2.5 & 6.9 & 4.7 & 6.2 & \textbf{9.3} & 8.6 & 5.8 & \textbf{8.8} & 4.3 & 1.4 & 3.4 & 7.4 \\
 & \textit{ipc18} & 77.6 & 39.9 & 81.6 & 66.3 & 85.3 & \textbf{86.0} & 69.6 & 83.7 & 75.0 & 56.6 & 59.7 & \textbf{88.9} \\
\addlinespace & folding & 0.2 & \textbf{7.0} & 3.2 & 1.1 & \textbf{3.8} & \textbf{3.8} & 1.3 & 1.4 & 2.1 & 0.5 & 2.8 & 1.4 \\
 & labyrinth & \textbf{6.4} & \textbf{15.0} & 0.0 & 2.0 & 0.0 & 0.0 & 0.0 & 0.0 & 0.0 & 0.0 & 0.1 & 0.1 \\
 & quantum & 16.0 & 15.8 & 17.0 & 16.2 & 16.3 & 16.3 & 15.6 & 16.6 & \textbf{17.2} & 11.7 & 16.1 & \textbf{17.1} \\
 & recharging & 2.9 & \textbf{4.0} & 2.8 & \textbf{4.0} & 3.0 & 2.9 & 3.0 & 3.8 & 3.8 & 2.8 & 1.9 & 3.0 \\
 & ricochet & 5.8 & 1.0 & 6.8 & \textbf{16.7} & 3.3 & 2.4 & 3.9 & \textbf{10.7} & 8.9 & 0.2 & 1.3 & 3.7 \\
 & rubiks & 4.0 & 5.0 & 5.4 & 4.0 & \textbf{13.5} & \textbf{13.6} & 11.8 & 12.0 & 13.0 & 7.5 & 7.6 & 9.2 \\
 & slitherlink & \textbf{3.0} & \textbf{3.0} & 0.0 & 0.0 & 0.0 & 0.0 & 0.0 & 0.0 & 0.0 & 0.0 & 0.0 & 0.0 \\
 & \textit{ipc23} & 38.3 & \textbf{50.8} & 35.1 & 44.1 & 39.9 & 39.1 & 35.5 & 44.6 & \textbf{45.0} & 22.7 & 29.7 & 34.4 \\
\addlinespace & \textit{total} & 115.9 & 90.6 & 116.7 & 110.4 & \textbf{125.2} & 125.1 & 105.1 & \textbf{128.3} & 120.1 & 79.4 & 89.4 & 123.4 \\
\bottomrule
\end{tabular}
\end{adjustbox}
\caption{
IPC Satisficing scores on 300-second agile run comparing \lsota planners.
Top-2 (including ties) are highlighted in \textbf{bold}.
}
\label{tbl:ipc-sat-300}
\end{table*}

\begin{table*}[p]
\centering
\begin{adjustbox}{max width=\linewidth}
\begin{tabular}{rc|rr|rr||rrrr|rrrrrc}
\toprule
 & IPC   & \multicolumn{2}{c|}{Lapkt-BFWS} & \multicolumn{2}{c||}{FD-BFWS} & \multicolumn{2}{c}{LAMA} & Dec  & NO & LAMAe  & \green{N2}\orange{DTC} & \green{N2}\orange{DTC} & \multirow{2}{*}{\coolname}\\
 & year  & Apx$^{\text{fd}}$ & $\vf_5$$^{\text{fd}}$ & $\vf_4$ & $\vf_5$ &  & +SM & Star & LAN                                      & +BFWS & +BFWS           & +LAMAe            & \\
\midrule\multirow{19}{*}{\rotatebox{90}{\# solved}} & agricola & 11.4 & 0.2 & 8.7 & \textbf{11.5} & 9.8 & 10.8 & 10.7 & 6.6 & 10.5 & \textbf{11.7} & 10.6 & 11.0 \\
 & caldera & 7.9 & 1.1 & \textbf{10.6} & 4.4 & 4.6 & 4.7 & 4.6 & 10.4 & \textbf{10.7} & 8.3 & 4.2 & 8.3 \\
 & data-net & \textbf{11.9} & 0.8 & 9.5 & 4.9 & 7.3 & 8.3 & 5.2 & 7.5 & 7.9 & 7.0 & 5.6 & \textbf{12.1} \\
 & flashfill & 11.8 & 0.6 & 11.2 & 0.0 & \textbf{13.7} & 12.4 & \textbf{12.9} & 10.6 & 8.1 & 3.3 & 7.0 & 10.8 \\
 & nurikabe & 0.0 & 0.0 & 9.8 & \textbf{12.2} & 8.3 & 8.3 & 6.0 & \textbf{11.8} & 9.8 & 7.6 & 7.8 & 8.5 \\
 & org-syn & 7.0 & 0.8 & \textbf{11.7} & 4.0 & 10.0 & 10.0 & 9.0 & 9.0 & \textbf{11.0} & 8.6 & 10.4 & 9.1 \\
 & settlers & 5.7 & 0.7 & 4.0 & 4.1 & \textbf{14.3} & \textbf{13.0} & 9.7 & 5.5 & 3.3 & 0.3 & 3.1 & 11.7 \\
 & snake & \textbf{17.0} & \textbf{17.8} & 5.4 & 15.8 & 2.7 & 5.4 & 2.6 & 7.5 & 4.9 & 6.1 & 4.5 & 5.8 \\
 & spider & 2.4 & \textbf{11.0} & 6.0 & 3.1 & 5.2 & 4.5 & 3.2 & \textbf{6.1} & 4.5 & 2.3 & 2.9 & 4.2 \\
 & termes & 2.5 & 6.9 & 4.7 & 6.2 & \textbf{9.3} & 8.6 & 5.8 & \textbf{8.8} & 4.3 & 1.4 & 3.4 & 7.4 \\
 & \textit{ipc18} & 77.6 & 39.9 & 81.6 & 66.3 & 85.3 & \textbf{86.0} & 69.6 & 83.7 & 75.0 & 56.6 & 59.7 & \textbf{88.9} \\
\addlinespace & folding & 0.2 & \textbf{7.0} & 3.2 & 1.1 & \textbf{3.8} & \textbf{3.8} & 1.3 & 1.4 & 2.1 & 0.5 & 2.8 & 1.4 \\
 & labyrinth & \textbf{6.4} & \textbf{15.0} & 0.0 & 2.0 & 0.0 & 0.0 & 0.0 & 0.0 & 0.0 & 0.0 & 0.1 & 0.1 \\
 & quantum & 16.0 & 15.8 & 17.0 & 16.2 & 16.3 & 16.3 & 15.6 & 16.6 & \textbf{17.2} & 11.7 & 16.1 & \textbf{17.1} \\
 & recharging & 2.9 & \textbf{4.0} & 2.8 & \textbf{4.0} & 3.0 & 2.9 & 3.0 & 3.8 & 3.8 & 2.8 & 1.9 & 3.0 \\
 & ricochet & 5.8 & 1.0 & 6.8 & \textbf{16.7} & 3.3 & 2.4 & 3.9 & \textbf{10.7} & 8.9 & 0.2 & 1.3 & 3.7 \\
 & rubiks & 4.0 & 5.0 & 5.4 & 4.0 & \textbf{13.5} & \textbf{13.6} & 11.8 & 12.0 & 13.0 & 7.5 & 7.6 & 9.2 \\
 & slitherlink & \textbf{3.0} & \textbf{3.0} & 0.0 & 0.0 & 0.0 & 0.0 & 0.0 & 0.0 & 0.0 & 0.0 & 0.0 & 0.0 \\
 & \textit{ipc23} & 38.3 & \textbf{50.8} & 35.1 & 44.1 & 39.9 & 39.1 & 35.5 & 44.6 & \textbf{45.0} & 22.7 & 29.7 & 34.4 \\
\addlinespace & \textit{total} & 115.9 & 90.6 & 116.7 & 110.4 & \textbf{125.2} & 125.1 & 105.1 & \textbf{128.3} & 120.1 & 79.4 & 89.4 & 123.4 \\
\bottomrule
\end{tabular}
\end{adjustbox}
\caption{
IPC Satisficing scores on 1800-second extended run comparing \lsota planners.
Top-2 (including ties) are highlighted in \textbf{bold}.
}
\label{tbl:ipc-sat-1800}
\end{table*}

\begin{figure*}[p]
 \centering
 \includegraphics[height=0.2\paperheight]{img+static+hist-elapsed+342c55ffbca24f01dd5e5fefeaaa23ef.pdf}
 \includegraphics[height=0.2\paperheight]{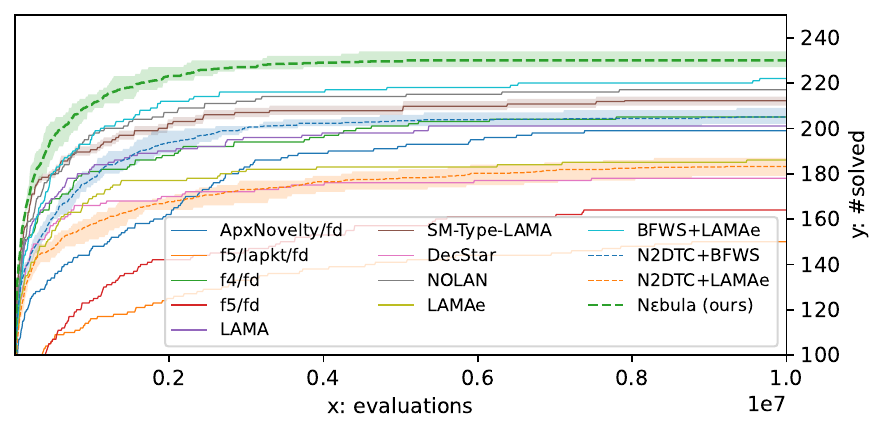}
 \includegraphics[height=0.2\paperheight]{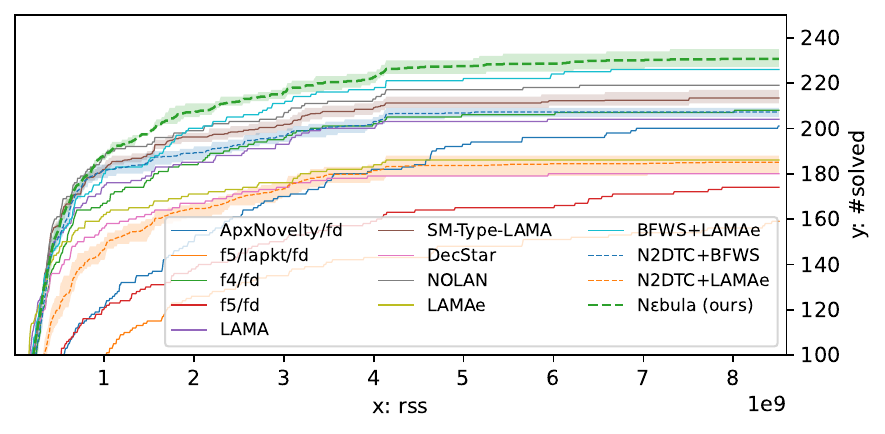}
 \caption{
 A histogram plot of the number of IPC18+IPC23 instances solved within 30 minutes.
 The $y$-axis indicates the number of instances solved within $x$ seconds (top), $x$ evaluations (middle), $x$ bytes of memory usage (bottom).
 The lines indicate the average of 5 seeds, while the band indicate the maximum and the minimum among the seeds.}
 \label{fig:histograms}
\end{figure*}


\begin{thebibliography}{37}
\providecommand{\natexlab}[1]{#1}

\bibitem[{Auer, Cesa-Bianchi, and Fischer(2002)}]{auer2002finite}
Auer, P.; Cesa-Bianchi, N.; and Fischer, P. 2002.
\newblock {Finite-Time Analysis of the Multiarmed Bandit Problem}.
\newblock \emph{Machine Learning}, 47(2-3): 235--256.

\bibitem[{Bonet and Geffner(2001)}]{bonet2001planning}
Bonet, B.; and Geffner, H. 2001.
\newblock {Planning as Heuristic Search}.
\newblock \emph{{Artificial Intelligence}}, 129(1): 5--33.

\bibitem[{B{\"u}chner et~al.(2023)B{\"u}chner, Keller, Eriksson, and
  Helmert}]{buchner2023landmark}
B{\"u}chner, C.; Keller, T.; Eriksson, S.; and Helmert, M. 2023.
\newblock Landmark progression in heuristic search.
\newblock In \emph{{Proc. of the International Conference on Automated Planning
  and Scheduling (ICAPS)}}, volume~33, 70--79.

\bibitem[{Burns et~al.(2012)Burns, Hatem, Leighton, and
  Ruml}]{burns2012implementing}
Burns, E.~A.; Hatem, M.; Leighton, M.~J.; and Ruml, W. 2012.
\newblock {Implementing Fast Heuristic Search Code}.
\newblock In \emph{{Proc. of Annual Symposium on Combinatorial Search}}.

\bibitem[{Bush and Mosteller(1953)}]{bush1953stochastic}
Bush, R.~R.; and Mosteller, F. 1953.
\newblock {A Stochastic Model with Applications to Learning}.
\newblock \emph{The Annals of Mathematical Statistics}, 559--585.

\bibitem[{Corr{\'e}a and Seipp(2025)}]{correa2025nolan}
Corr{\'e}a, A.~B.; and Seipp, J. 2025.
\newblock {Alternation-Based Novelty Search}.
\newblock In \emph{{Proc. of the International Conference on Automated Planning
  and Scheduling (ICAPS)}}.

\bibitem[{Dial(1969)}]{dial1969shortest}
Dial, R. 1969.
\newblock {Shortest Path Forest with Topological Ordering}.
\newblock \emph{CACM}, 12(11): 632--633.

\bibitem[{Dijkstra(1959)}]{dijkstra1959note}
Dijkstra, E.~W. 1959.
\newblock {A Note on Two Problems in Connexion with Graphs}.
\newblock \emph{Numerische mathematik}, 1(1): 269--271.

\bibitem[{Doran and Michie(1966)}]{doran1966experiments}
Doran, J.~E.; and Michie, D. 1966.
\newblock {Experiments with the Graph Traverser Program}.
\newblock \emph{Proceedings of the Royal Society of London. Series A.
  Mathematical and Physical Sciences}, 294(1437): 235--259.

\bibitem[{Fikes and Nilsson(1972)}]{fikes1972strips}
Fikes, R.~E.; and Nilsson, N.~J. 1972.
\newblock {STRIPS: A New Approach to the Application of Theorem Proving to
  Problem Solving}.
\newblock \emph{{Artificial Intelligence}}, 2(3): 189--208.

\bibitem[{Gelly and Silver(2011)}]{gelly2011monte}
Gelly, S.; and Silver, D. 2011.
\newblock {Monte-Carlo Tree Search and Rapid Action Value Estimation in
  Computer Go}.
\newblock \emph{{Artificial Intelligence}}, 175(11): 1856--1875.

\bibitem[{Gnad, Shleyfman, and Hoffmann(2018)}]{gnad2018decstar}
Gnad, D.; Shleyfman, A.; and Hoffmann, J. 2018.
\newblock DecStar--STAR-topology DECoupled Search at its best.
\newblock In \emph{{Proc. of the International Planning Competition}}, 42--46.

\bibitem[{Hart, Nilsson, and Raphael(1968)}]{hart1968formal}
Hart, P.~E.; Nilsson, N.~J.; and Raphael, B. 1968.
\newblock {A Formal Basis for the Heuristic Determination of Minimum Cost
  Paths}.
\newblock \emph{Systems Science and Cybernetics, IEEE Transactions on}, 4(2):
  100--107.

\bibitem[{Helmert(2004)}]{Helmert04}
Helmert, M. 2004.
\newblock {A Planning Heuristic Based on Causal Graph Analysis}.
\newblock In \emph{{Proc. of the International Conference on Automated Planning
  and Scheduling (ICAPS)}}, 161--170.

\bibitem[{Helmert(2006)}]{Helmert2006}
Helmert, M. 2006.
\newblock {The Fast Downward Planning System}.
\newblock \emph{{J. Artif. Intell. Res.(JAIR)}}, 26: 191--246.

\bibitem[{Helmert and Geffner(2008)}]{helmert2008unifying}
Helmert, M.; and Geffner, H. 2008.
\newblock {Unifying the Causal Graph and Additive Heuristics}.
\newblock In \emph{{Proc. of the International Conference on Automated Planning
  and Scheduling (ICAPS)}}, 140--147.

\bibitem[{Hoffmann and Nebel(2001)}]{hoffmann01}
Hoffmann, J.; and Nebel, B. 2001.
\newblock {The FF Planning System: Fast Plan Generation through Heuristic
  Search}.
\newblock \emph{{J. Artif. Intell. Res.(JAIR)}}, 14: 253--302.

\bibitem[{Keller and Helmert(2013)}]{keller2013trial}
Keller, T.; and Helmert, M. 2013.
\newblock {Trial-Based Heuristic Tree Search for Finite Horizon MDPs.}
\newblock In \emph{{Proc. of the International Conference on Automated Planning
  and Scheduling (ICAPS)}}.

\bibitem[{Kishimoto et~al.(2012)Kishimoto, Winands, M{\"u}ller, and
  Saito}]{kishimoto2012game}
Kishimoto, A.; Winands, M.~H.; M{\"u}ller, M.; and Saito, J.-T. 2012.
\newblock {Game-Tree Search using Proof Numbers: The First Twenty Years}.
\newblock \emph{ICGA journal}, 35(3): 131--156.

\bibitem[{Kocsis and Szepesv{\'a}ri(2006)}]{kocsis2006bandit}
Kocsis, L.; and Szepesv{\'a}ri, C. 2006.
\newblock {Bandit Based Monte-Carlo Planning}.
\newblock In \emph{{Proc. of the European Conference on Machine Learning and
  Principles and Practice of Knowledge Discovery in Databases}}, 282--293.
  Springer.

\bibitem[{Korf(1985)}]{korf1985depth}
Korf, R.~E. 1985.
\newblock {Depth-First Iterative-Deepening: An Optimal Admissible Tree Search}.
\newblock \emph{{Artificial Intelligence}}, 27(1): 97--109.

\bibitem[{Korf(1993)}]{korf1993linear}
Korf, R.~E. 1993.
\newblock {Linear-Space Best-First Search}.
\newblock \emph{{Artificial Intelligence}}, 62(1): 41--78.

\bibitem[{Kuroiwa and Beck(2022)}]{kuroiwa2022biased}
Kuroiwa, R.; and Beck, J.~C. 2022.
\newblock {Biased Exploration for Satisficing Heuristic Search}.
\newblock In \emph{{Proc. of the International Conference on Automated Planning
  and Scheduling (ICAPS)}}.

\bibitem[{Lipovetzky and Geffner(2017)}]{lipovetzky2017bfws}
Lipovetzky, N.; and Geffner, H. 2017.
\newblock {Best-First Width Search: Exploration and Exploitation in Classical
  Planning }.
\newblock In \emph{{Proc. of AAAI Conference on Artificial Intelligence}}.

\bibitem[{Nakhost and M{\"u}ller(2009)}]{nakhost2009monte}
Nakhost, H.; and M{\"u}ller, M. 2009.
\newblock {Monte-Carlo Exploration for Deterministic Planning}.
\newblock In \emph{{Proc. of International Joint Conference on Artificial
  Intelligence (IJCAI)}}.

\bibitem[{Richter and Helmert(2009)}]{RichterH2009}
Richter, S.; and Helmert, M. 2009.
\newblock {Preferred Operators and Deferred Evaluation in Satisficing
  Planning}.
\newblock In \emph{{Proc. of the International Conference on Automated Planning
  and Scheduling (ICAPS)}}, 273--280.

\bibitem[{Richter, Westphal, and Helmert(2011)}]{richter2011lama}
Richter, S.; Westphal, M.; and Helmert, M. 2011.
\newblock {LAMA 2008 and 2011}.
\newblock In \emph{{Proc. of the International Planning Competition}},
  117--124.

\bibitem[{Robbins(1952)}]{robbins1952some}
Robbins, H. 1952.
\newblock {Some Aspects of the Sequential Design of Experiments}.
\newblock \emph{Bulletin of the American Mathematical Society}, 58(5):
  527--535.

\bibitem[{R{\"o}ger and Helmert(2010)}]{RogerH10}
R{\"o}ger, G.; and Helmert, M. 2010.
\newblock {The More, the Merrier: Combining Heuristic Estimators for
  Satisficing Planning}.
\newblock In \emph{{Proc. of the International Conference on Automated Planning
  and Scheduling (ICAPS)}}.

\bibitem[{Schulte and Keller(2014)}]{schulte2014balancing}
Schulte, T.; and Keller, T. 2014.
\newblock {Balancing Exploration and Exploitation in Classical Planning}.
\newblock In \emph{{Proc. of Annual Symposium on Combinatorial Search}}.

\bibitem[{Silver et~al.(2016)Silver, Huang, Maddison, Guez, Sifre, {Van Den
  Driessche}, Schrittwieser, Antonoglou, Panneershelvam, Lanctot
  et~al.}]{alphago}
Silver, D.; Huang, A.; Maddison, C.~J.; Guez, A.; Sifre, L.; {Van Den
  Driessche}, G.; Schrittwieser, J.; Antonoglou, I.; Panneershelvam, V.;
  Lanctot, M.; et~al. 2016.
\newblock {Mastering the Game of Go with Deep Neural Networks and Tree Search}.
\newblock \emph{Nature}, 529(7587): 484--489.

\bibitem[{Singh et~al.(2021)Singh, Lipovetzky, Ram{\'\i}rez, and
  Segovia-Aguas}]{singh2021approximate}
Singh, A.; Lipovetzky, N.; Ram{\'\i}rez, M.; and Segovia-Aguas, J. 2021.
\newblock {Approximate Novelty Search}.
\newblock In \emph{{Proc. of the International Conference on Automated Planning
  and Scheduling (ICAPS)}}, volume~31, 349--357.

\bibitem[{Tesauro, Rajan, and Segal(2010)}]{tesauro2010bayesian}
Tesauro, G.; Rajan, V.; and Segal, R. 2010.
\newblock {Bayesian Inference in Monte-Carlo Tree Search}.
\newblock In \emph{{Proc. of the International Conference on Uncertainty in
  Artificial Intelligence (UAI)}}, 580--588.

\bibitem[{Thompson(1933)}]{thompson1933likelihood}
Thompson, W.~R. 1933.
\newblock {On the Likelihood that One Unknown Probability Exceeds Another in
  View of the Evidence of Two Samples}.
\newblock \emph{Biometrika}, 25(3-4): 285--294.

\bibitem[{Wissow and Asai(2024)}]{wissow2023scale}
Wissow, S.; and Asai, M. 2024.
\newblock {Scale-Adaptive Balancing of Exploration and Exploitation in
  Classical Planning}.
\newblock In \emph{{Proc. of European Conference on Artificial Intelligence}}.

\bibitem[{Xie et~al.(2014)Xie, M{\"u}ller, Holte, and Imai}]{xie14type}
Xie, F.; M{\"u}ller, M.; Holte, R.~C.; and Imai, T. 2014.
\newblock {Type-Based Exploration with Multiple Search Queues for Satisficing
  Planning}.
\newblock In \emph{{Proc. of AAAI Conference on Artificial Intelligence}}.

\bibitem[{Yoshizoe et~al.(2011)Yoshizoe, Kishimoto, Kaneko, Yoshimoto, and
  Ishikawa}]{yoshizoe2011scalable}
Yoshizoe, K.; Kishimoto, A.; Kaneko, T.; Yoshimoto, H.; and Ishikawa, Y. 2011.
\newblock {Scalable Distributed Monte-Carlo Tree Search}.
\newblock In \emph{{Proc. of Annual Symposium on Combinatorial Search}},
  volume~2, 180--187.

\end{thebibliography}
\end{document}